%% file: main_ats.tex
\documentclass[twoside]{article}
\usepackage[accepted]{aistats2024}
\input{packages}

\begin{document}

\runningtitle{Accelerating Approximate Thompson Sampling with Underdamped Langevin Monte Carlo}

\twocolumn[
\aistatstitle{Accelerating Approximate Thompson Sampling \\ with Underdamped Langevin Monte Carlo}

\aistatsauthor{ Haoyang Zheng \And Wei Deng \And  Christian Moya \And Guang Lin }

\aistatsaddress{ Purdue University \And Morgan Stanley \And Purdue University \And Purdue University } ]

\begin{abstract}

Approximate Thompson sampling with Langevin Monte Carlo broadens its reach from Gaussian posterior sampling to encompass more general smooth posteriors. However, it still encounters scalability issues in high-dimensional problems when demanding high accuracy. To address this, we propose an approximate Thompson sampling strategy, utilizing underdamped Langevin Monte Carlo, where the latter is the go-to workhorse for simulations of high-dimensional posteriors. Based on the standard smoothness and log-concavity conditions, we study the accelerated posterior concentration and sampling using a specific potential function. This design improves the sample complexity for realizing logarithmic regrets from $\mathcal{\tilde O}(d)$ to $\mathcal{\tilde O}(\sqrt{d})$. The scalability and robustness of our algorithm are also empirically validated through synthetic experiments in high-dimensional bandit problems. 

\end{abstract}

\input{documents}

\bibliographystyle{apalike}
\bibliography{refs}

\clearpage

\appendix
\onecolumn
\aistatstitle{
Supplementary Materials}
\vspace{-2 in}

\input{appendix}

\end{document}

%% file: packages.tex
\usepackage{amsfonts}       
\usepackage{amsmath}
\usepackage{amssymb}
\usepackage{amsthm}
\usepackage{algorithm}
\usepackage{algpseudocode}
\usepackage{appendix}
\usepackage{bm, bbm}
\usepackage{booktabs}
\usepackage{color}
\usepackage{xcolor}
\definecolor{darkblue}{rgb}{0.0, 0.0, 0.6}
\definecolor{darkred}{rgb}{0.7, 0.0, 0.0}
\usepackage{comment}
\usepackage{etoc}
\usepackage[T1]{fontenc}    
\usepackage[breaklinks]{hyperref}
\usepackage[utf8]{inputenc} 
\usepackage{nicefrac}       
\usepackage{mathtools}
\usepackage{microtype}      
\usepackage{multirow}
\usepackage{natbib}
\usepackage{tabularx}
\usepackage{txfonts}
\usepackage{graphicx}
\usepackage{pifont}
\usepackage{subfigure}
\usepackage{url}            
\usepackage{xcolor}         

\newtheorem{theorem}{Theorem}
\newtheorem{proposition}{Proposition}
\newtheorem{lemma}{Lemma}
\newtheorem{remark}{Remark}

\newcounter{assumption}%
\renewcommand{\theassumption}{\arabic{assumption}}

\newtheorem{innercustomgeneric}{\customgenericname}
\providecommand{\customgenericname}{}
\newcommand{\newcustomtheorem}[2]{%
\newenvironment{#1}[1]
{%
 \renewcommand\customgenericname{#2}%
 \renewcommand\theinnercustomgeneric{##1}%
 \innercustomgeneric
}
{\endinnercustomgeneric}
}
\newcustomtheorem{assump}{Assumption}

\hypersetup{
  pdffitwindow=true,
  pdfstartview={FitH},
  pdfnewwindow=true,
  colorlinks,
  linktocpage=true,
  linkcolor=darkblue,
  urlcolor=darkblue,
  citecolor=darkblue
}

%% file: documents.tex
\section{INTRODUCTION}

Multi-armed bandit (MAB) problem is a classic decision-making dilemma to find the best option among multiple arms, which has been applied in various domains such as recommendation systems \citep{may2012optimistic,chapelle2011empirical}, advertising \citep{graepel2010web-scale}, etc. 
A fundamental challenge in the MAB problem is to balance exploration and exploitation, and Thompson sampling (TS) has proven to be an effective mechanism for addressing this, particularly appealing due to its simplicity and strong empirical performance. In TS, each arm is affiliated with a posterior distribution that approximates the model parameters. Arms with higher uncertainty in their approximate posterior may be selected to facilitate exploration. Post-selection, the observed rewards serve to iteratively update the corresponding posterior distribution, thereby improving the estimation and reducing variance \citep{russo2014learning}. Thanks to recent extensive advancements in both empirical experimentation \citep{granmo2010solving} and theoretical formulations \citep{agrawal2012analysis, russo2016information}, TS gained widespread recognition. 

TS often models the posteriors with Laplace approximation for computational efficiency. However, such a simplification is computationally intensive and may inaccurately capture non-Gaussian posteriors \citep{huix2023tight}. Given the limitations of the exclusive use of distributions with Laplace approximations, researchers are increasingly turning to Langevin Monte Carlo (LMC) as a robust method for exploring and exploiting the posterior distributions.

Overdamped Langevin Monte Carlo (LMC) in machine learning originated from stochastic gradient Langevin dynamics (SGLD) \citep{Welling11}, which injects noise into stochastic gradient descent, evolving into a sampling algorithm as the learning rate decays. The theoretical guarantees on global optimization \citep{Maxim17, xu2018global, Yuchen17} and uncertainty estimation \citep{Teh16} were further established, which provides specific guidance on the tuning of temperatures and learning rates. Despite the elegance, it is inevitable to suffer from scalability issues w.r.t. the dimension, accuracy target, and condition number. To tackle this issue, the go-to framework is the underdamped Langevin Monte Carlo (ULMC) \citep{Chen14, cheng2018underdamped}, which can be viewed as a variant of Hamiltonian Monte Carlo (HMC) \citep{Neal12, no_U_turn_sampler, Mangoubi18_leapfrog}. The incorporation of momentum (or velocity) in the sampling process facilitates the exploration and results in more effective samples.

\textbf{Our Contribution}

In this work, we focus on enhancing the practicality and scalability of TS. The introduction of LMC to TS offers a potential route to avoid the Laplace approximation of the posterior distributions but falls short in high-dimensional settings when sample complexity is in high demand. Our primary contribution lies in the novel integration of ULMC into the TS framework to address this issue. 

We first analyze the posterior as the limiting distribution of stochastic differential equation (SDE) trajectories. Through a specific potential function, we delve into the analysis of SDE trajectories and provide a renewed perspective on posterior concentration rates. This analytical perspective further enables a more efficient posterior approximation and results in an enhancement of the sample complexity to realize logarithmic regrets from $\tilde O(d)$ to $\tilde O(\sqrt{d})$, where $d$ represents the dimension of model parameters. Synthetic experiments across varied scenarios provide empirical results to further validate our theoretical statements and the robustness of our proposed algorithms.

\section{RELATED WORKS}

\textbf{Thompson Sampling} began with a primary focus on experimental studies \citep{russo2018tutorial} due to the theoretical challenges we introduced earlier. In recent MAB works, an in-depth analysis of TS from a theoretical perspective has become more prevalent across various reward and prior assumptions \citep{russo2016information, riou2020bandit, baudry2021optimal}. Recent works further explored TS with problem-independent bounds \citep{agrawal2017near}, which was further improved to achieve minimax optimal performance \citep{jin2021mots, jin2022finite, jin2023thompson}. Its mechanism is straightforward when using closed-form posteriors, but in more general settings, the posterior sampling becomes challenging and requires further discussions. One challenge is the dependence of TS on prior quality, as improper priors may lead to insufficient explorations \citep{liu2016prior}. Another challenge is from approximate errors, specifically, small constant approximate errors may induce linear regrets. Drawing upon insights from \citep{phan2019thompson}, approximate TS with LMC achieves logarithmic regrets, with approximation errors diminishing over rounds \citep{mazumdar2020approximate}.

Beyond its well-known application in MABs, TS has broader implications in the domain of reinforcement learning. Specifically, it has been further validated in various contexts such as in combinatorial bandits \citep{wang2018thompson, perrault2020statistical}, contextual bandits \citep{agrawal2013thompson, xu2022langevin, chakraborty2023thompson}, linear Markov Decision Processes \citep{ishfaq2023provable}, and linear quadratic controls \citep{kargin2022thompson}, etc.

\textbf{Langevin Monte Carlo} becomes notable for convergence analysis in overdamped forms under strong log-concavity \citep{Alain17}, which achieves \(O(d/\epsilon^2)\) complexity for \(\epsilon\) error in 2-Wasserstein (W2) distance. This result was further extended to the stochastic gradient version \citep{dalalyan2019user}. A similar convergence result, but based on the Total Variation metric, was achieved later \citep{dk17}. Furthermore, LMC was shown to offer theoretical guarantees for handling multi-modal distributions in \citet{Maxim17}.

To expedite convergence, \citet{cheng2018underdamped} delved into the non-asymptotic convergence of the ULMC, revealing a computational complexity of $O(\sqrt{d}/\epsilon)$ to obtain $\epsilon$ error in W2. For the related Hamiltonian Monte Carlo, \citet{mangoubi2017rapid} demonstrated that the convergence rate depends quadratically on the condition number. This result was significantly improved by \citet{Chen_Vempala} using an elegant geometric approach, achieving a linear convergence rate. Further insights into the numerical analysis of the leapfrog scheme were provided by \citet{chen2020fast} and \citet{Mangoubi18_leapfrog}. Some works have concentrated on improving sampling efficiency and accuracy by utilizing control variates \citep{zou2018stochastic, zou2018subsampled, zou2019stochastic}. For simulations of multi-modal distributions, the replica exchange Langevin Monte Carlo \citep{chen2018accelerating, deng2020} ran multiple chains at different temperatures to balance exploration and exploitation; adaptive importance sampling algorithms \citep{awsgld, deng2022interacting} were studied to simulate from an adaptively modified landscape to address the local trapping issue and further employed the importance weights to correct the bias.

\section{PROBLEM SETTING}

MABs are classical frameworks for modeling the trade-off between exploration and exploitation in sequential decision-making problems. Formally, MABs consist of a set of arms $\mathcal A=\left\{1, 2, \cdots, K\right\}$, each associated with an unknown reward distribution. At each round $n$ ($n = 1, 2, \cdots, N$), the agent selects an arm $A_n\in \mathcal A$ and receives a random reward ${\mathcal R}_{a, n}\in \mathbb R$ sampled from the corresponding unknown reward distribution. Multiple rewards received from playing the same arm repeatedly are identically and independently distributed (i.i.d.) and are statistically independent of the rewards received from playing other arms.

Within $N$ rounds, the objective of the TS algorithm is to identify the optimal arm that most likely incurs the lowest regrets ${\mathfrak R}(N)$:
\begin{equation*}
    {\mathfrak R}(N) = N \bar{\mathcal R}_1 - \sum_{n=1}^{N} \mathbb{E}[\mathcal R_{A_n, n}],
\end{equation*}
where $\bar{\mathcal R}_1$ denotes the expected reward of the first arm, and without loss of generality, we name the first arm as the optimal one. In the classical setting, the true posteriors of the arms are unknown and must be estimated through sampling. With each arm play, we sample a reward from its true posterior and then update the approximate posterior (denoted as $\mu_a$) based on the given reward. The challenge lies in deciding which arm to play at each round, balancing the need to gather information about potentially suboptimal arms (exploration) against the desire to play the arm that currently seems best (exploitation). A general framework of TS for MABs is given in Algorithm \ref{algorithm_thompson_main}:
\begin{algorithm}[!htbp]
\caption{A general framework of TS for MABs.}\label{algorithm_thompson_main}
{\hspace*{\algorithmicindent} \textbf{Input} Bandit feature vectors $\alpha_{a}$ for $\forall a \in \mathcal{A}$.}
\begin{algorithmic}[1]
\For{{$n$=1 to $N$}}
\State {Sample $\left(x_{a, n}, v_{a, n}\right)\sim \mu_a\left[\rho_a\right]$ for $\forall a \in \mathcal{A}$.}
\State {Choose arm $A_n = \text{argmax}_{a\in\mathcal A} \left\langle \alpha_a, x_{a, n}\right\rangle$.}
\State {Play arm $A_n$ and receive reward ${\mathcal R}_{n}$.}
\State {Update posterior distribution of arm $A_n$: $\mu_{A_n}\left[\rho_{A_n}\right]$.}
\State {Calculate expected regret ${{\mathfrak R}}_n$.}
\EndFor
\end{algorithmic}
\hspace*{\algorithmicindent} {\textbf{Output}  Expected total regret $\sum_{n=1}^N {{\mathfrak R}}_n$.}
\end{algorithm}

\section{POSTERIOR ANALYSIS}\label{section_posterior_analysis}

The unique feature that distinguishes TS from frequentist methods \citep{sutton2018reinforcement} is its use of posterior distributions to guide decision-making. This probabilistic sampling inherently incorporates both the mean and the uncertainty (including but not limited to variance) of our current beliefs about each arm. Therefore, understanding how these posteriors evolve over time becomes fundamental not only for explaining the algorithm's empirical efficiency but also for establishing its theoretical properties.

For the purposes of this study, we make the assumption that each arm's reward distribution is parameterized by a set of fixed parameters $x_{a,*}\in \mathbb R^d$, and the reward distribution $\mathtt P_a({\mathcal R})$ is now parameterized: $\mathtt P_a({\mathcal R}) = \mathtt P_a({\mathcal R}, x_{a,*})$. To simplify our notation, we drop the arm-specific parameter $a$ when discussing settings uniformly across arms in this section. To facilitate our analysis, we introduce a series of standard assumptions on both the reward distributions and distributions related to arms:

\begin{assump}{1}[Local assumptions on the likelihood]
\label{assumption_likelihood}
For the function $\log \mathtt P\left({\mathcal R}|x\right)$, its gradient w.r.t. $x$ is $L$-Lipschitz and the function is $m$-strongly convex for constants $L \geq 0$ and $m > 0$. This means for all ${\mathcal R}\in \mathbb R$ and $x, x_* \in \mathbb R^d$, the following inequalities hold:
\begin{equation*}
\begin{array}{cc}
&\left\langle\nabla_x \log \mathtt P\left({\mathcal R}|x_*\right),x-x_*\right\rangle + \frac{m}{2}\left\|x-x_*\right\|_2^2 \vspace{0.05 in} \\
& \leq \log \mathtt P\left({\mathcal R}|x\right) - \log \mathtt P\left({\mathcal R}|x_*\right) \leq \vspace{0.05 in} \\
& \left\langle\nabla_x \log \mathtt P\left({\mathcal R}|x_*\right),x-x_*\right\rangle + \frac{L}{2}\left\|x-x_*\right\|_2^2.
\end{array}
\end{equation*}
\end{assump}

\begin{assump}{2}[Assumptions on the reward distribution]
\label{assumption_reward}
For $\log \mathtt P\left({\mathcal R}|x_*\right)$, its gradient concerning $\mathcal R$ is $L$-Lipschitz and the function is $\nu$-strongly convex, given constants $L \geq 0$ and $\nu > 0$. Consequently, For every ${\mathcal R}, {\mathcal R'}\in \mathbb R$ and $x_* \in \mathbb R^d$, we have the relationship:
\begin{equation*}
\begin{array}{cc}
&\left\langle\nabla_{\mathcal R} \log \mathtt P\left({\mathcal R'}|x_*\right),\mathcal R-\mathcal R'\right\rangle + \frac{\nu}{2}\left\|\mathcal R-\mathcal R'\right\|_2^2 \vspace{0.05 in} \\
& \leq \log \mathtt P\left({\mathcal R}|x_*\right) - \log \mathtt P\left({\mathcal R'}|x_*\right) \leq \vspace{0.05 in} \\
& \left\langle\nabla_{\mathcal R} \log \mathtt P\left({\mathcal R'}|x_*\right),\mathcal R-\mathcal R'\right\rangle + \frac{L}{2}\left\|\mathcal R-\mathcal R'\right\|_2^2.
\end{array}
\end{equation*}
\end{assump}
\begin{assump}{3}[Assumption on the prior]\label{assumption_prior}
For the gradient of $\pi\left({\mathcal R}|x\right)$, there exists a constant $L$ such that for $x, x'\in \mathbb R^d$:
\begin{equation*}
\left\|\nabla \pi\left(x\right)-\nabla \pi\left(x'\right)\right\|_2 \leq L \left\|x-x'\right\|_2.
\end{equation*}
\end{assump}

Ignoring the difference across arms, this section presents theoretical investigations of the general convergence properties of the TS posterior distributions. We will focus on understanding the dynamics of the posterior distributions, and how ULMC samples from the posterior and approximates it. The main contribution regarding the sample complexity is also discussed subsequently.

\subsection{Continuous-Time Diffusion Analysis}\label{section_posterior_concentration_main}

For the derivation of posterior concentration rates for parameters in terms of the accumulative number of likelihoods, we conduct an innovative analysis of the moments of a potential function along the trajectories of SDEs, where we consider the posterior as the limiting distribution given $n$ rewards:
\begin{equation*}
\begin{split}
&\mu[\rho] = \lim_{t\rightarrow \infty}\mathtt P\left((x_t, v_t)|\mathcal R_1, \mathcal R_2,\ldots, \mathcal R_n\right) \\
\propto & \exp\left(-\rho\left(nf_{n}(x)+n\|v\|_2^2/2u+\log\pi(x)\right)\right)
\end{split}
\end{equation*}
almost surely. Here $x_t,v_t\in \mathbb R^d$ are the position and velocity terms over time $t>0$, $f_{n}(x)=\frac{1}{n}\sum_{j=1}^{\mathcal L(n)}\log\mathtt P\left({\mathcal R_j}|x\right)$ represents the average log-likelihood function and the limiting distribution is a scaled posterior distribution $\mu\left[\rho\right]$ with the scale parameter $\rho$. We denote here $\mathcal L(n)$ is the total number of rewards we have up to round $n$. By studying the evolution of the limiting joint distribution of $x$ and $v$ according to the following SDEs:
\begin{equation}\label{eq_sde2}
  \begin{split}
      d v_t & =-\gamma v_t d t - u \nabla f_n\left(x_t\right)dt - \frac{u}{n} \nabla\log\pi(x_t) d t + 2\sqrt{\frac{\gamma  u}{n\rho}} d B_t, \\
      d x_t & =v_t d t
  \end{split}
\end{equation}
as $t$ goes to infinity, we are able to deduce the convergence rate of the scaled posterior distribution. Here $\gamma$ is the friction coefficient, and $u$ is the noise amplitude.

\textbf{Posterior Concentration Analysis}

Prior to presenting our findings on posterior concentration, we initially outline several fundamental points for achieving our results. Our derivation of the scaled posterior concentration $\mu[\rho]$ draws inspiration from \citet{cheng2018underdamped}, wherein a specific potential function is structured to infer the rates of posterior concentration. However, our focus is primarily on how the posterior concentrates around $x_*$, where $x_*$ itself does not conform to SDEs. Therefore, we achieve the concentration rates by an extension of Gr{\"o}nwall's inequality in \citet{dragomir2003some} rather than applying it directly. Moreover, the coefficients employed in constructing the potential function render $\gamma=2$ and $u=\frac{1}{L}$ as the optimal selections. Subsequently, we clarify another additional essential property, indicating that, with high probability, $\left\|\nabla_x f_n\left(x_*\right)\right\|_2$ and the gradient of $\left\|v_*\right\|_2^2$ concentrating around zero, given the data $\mathcal R_1, \ldots,\mathcal R_n$. Specifically, our findings demonstrate that $\left\|\nabla_x f_n\left(x_*\right)\right\|_2$ exhibits sub-Gaussian tails, and we designate $\nabla_v\left\|v_*\right\|_2^2$ as zero to maintain simplicity.

Building upon the above analysis, we construct the following potential function given $\alpha >0$:
$$V\left( x_t, v_t, t\right) = \frac{1}{2}e^{\alpha t} \left\|(x_t, x_t+v_t) - (x_*, x_* + v_*)\right\|^2_2,$$ 
which develops along the trajectories of SDEs \eqref{eq_sde2}. By delimiting the supremum of the given potential function, we obtain an upper bound for higher moments of $\left\|x - x_*\right\|$, where $x\sim \mu\left[\rho\right]$. These established moment bounds directly correspond to the posterior concentration rates of $x$ around $x_*$, as illustrated in the following theorem:
\begin{theorem}\label{theorem_posterior_concentration}
Suppose that Assumptions \ref{assumption_likelihood} and \ref{assumption_prior} hold, and suppose $x\in \mathbb R^d$ follows SDEs \eqref{eq_sde2}, then for $x_*\in \mathbb R^d$ and $\delta \in \left(0, e^{-0.5}\right)$, the posterior satisfies:
\begin{equation*}
  \mathbb{P}_{{x} \sim {\mu\left[\rho\right]}}\left(\left\|x - x_*\right\|_2\geq \sqrt{\frac{2e}{mn}\left(D+2\Omega\log \frac{1}{\delta}\right)}\right)\leq\delta,
\end{equation*}
with $D=\frac{8d}{\rho }+{2}\log B$, $\Omega=16\kappa^2d+\frac{256}{\rho}$, $\kappa=\frac{L}{m}$ is the condition number, and $B=\max_x \frac{\pi(x)}{\pi(x_*)}$ represents prior quality.
\end{theorem}

We proceed to offer further insight from the theorem. The first term $\frac{8d}{\rho }$ originates from squaring the $dB_t$ term in SDEs, which has a negligible influence on the results. The second term $2\log B$ is associated with the prior $\pi(x)$ and vanishes when $x_* = \text{argmax}\left[\pi(x)\right]$. Section \ref{experiments} also provides extensive explorations into how the choice of prior affects TS performance. Term $16\kappa^2d$ originates from the likelihood in SDEs, and the last term is associated with the noise present in SDEs, which substantially influences the contraction rate. Selecting a suitable value for $\rho$ enables control of the posterior concentration scale. A detailed proof can be found in Appendix \ref{section_posterior_concentration}, Theorem \ref{theorem_posterior_concentration2}.

\subsection{Discrete-Time Dynamics Analysis}

We continue to delineate the methodology for employing ULMC to approximate posterior distributions by sampling positions and velocities. In essence, they provide a more efficient framework to draw samples sequentially based on current steps. Algorithm \ref{algorithm_langevin_mcmc} outlines the execution of ULMC, detailing the sequential generation of samples for arm $a$. It initiates from the sample gathered at the last step of rounds $n-1$  and advances to the sample created for rounds $n$.

\begin{algorithm}[!htbp]
\caption{(Stochastic Gradient) Underdamped Langevin Monte Carlo at round $n$.}\label{algorithm_langevin_mcmc}
{\hspace*{\algorithmicindent} \textbf{Input} Data $\{{\mathcal R}_{1}, {\mathcal R}_{2}, \cdots, {\mathcal R}_{\mathcal L(n-1)}\}$;}\\
{\hspace*{\algorithmicindent} \textbf{Input} Sample $(x_{Ih^{(n-1)}}, v_{Ih^{(n-1)}})$ from last round;}
\begin{algorithmic}[1]
\State {Initialize $x_{0}=x_{Ih^{(n-1)}}$ and $v_{0}=v_{Ih^{(n-1)}}$.}
\For{$i=0,1,\cdots, I-1$}
\State {Subsample data set $ \{{\mathcal R}_{1}, \cdots, {\mathcal R}_{|\mathcal S|}\}\subseteq \mathcal{S}$. }
\State {Compute gradient estimate $\nabla{\hat U}(x_i)$.}
\State {Sample $\left(x_{i+1}, v_{i+1}\right)$ from $\left(x_{i}, v_{i}\right)$.}
\EndFor
\State $x_{I h^{(n)}} \sim \mathcal{N}\left( x_{I} , \frac{1}{nL\rho}\mathbf{I}_{d\times d}\right)$ and $v_{I h^{(n)}}=v_{I}$
\end{algorithmic}
\hspace*{\algorithmicindent} \textbf{Output} Sample $(x_{Ih^{(n)}}, v_{Ih^{(n)}})$ from current round.
\end{algorithm}
Within the ULMC framework, one can utilize either full gradients or stochastic gradients for the gradient estimation. The former is computed by summing all the gradients of likelihoods up to round $n$ combined with the prior:
\begin{equation*}
    \nabla U(x_{i}) = - \sum_{j=1}^{\mathcal L(n)} \nabla \log \mathtt P \left({{\mathcal R}_{j}|x_{i}}\right)  - \nabla \log \pi(x_{i}),
\end{equation*}
which provide high-precision information for updates but might be computationally intensive. On the other hand, stochastic gradients, approximated from a subset $\mathcal{S}$ of data:
$$\nabla \hat U(x_{i}) = - \frac{\mathcal L(n)}{|\mathcal{S}|} \sum_{\mathcal R_{k}\in\mathcal{S} } \nabla \log \mathtt P \left({{\mathcal R}_{k}|x_{i}}\right) - \nabla \log \pi(x_{i}),$$
offer computational efficiency at the cost of precision. The choice between full or stochastic gradients depends on the specific requirements and constraints of the problem at hand, balancing between computational feasibility and approximation accuracy.

To transit consecutive samples from the index $i$ to $i+1$, we followed by integrating the discrete underdamped Langevin dynamics up to $h$, which are given as follows:
\begin{equation}\label{eq_sample_ula2}
\begin{split}
	\begin{bmatrix}
			x_{i+1}\\
			v_{i+1}\\
	\end{bmatrix} \sim \mathcal{N}\left(
	\begin{bmatrix}
		\mathbb E\left[x_{i+1}\right] \\ 
		\mathbb E\left[v_{i+1}\right]
	\end{bmatrix}, \begin{bmatrix}
		\mathbb V_x 	&  \mathbb K_{x,v}\\ 
		\mathbb K_{v,x}	&  \mathbb V_v \\
	\end{bmatrix}\right).
\end{split}
\end{equation}
Here the expectation of the position $x_{i+1}$ and velocity $v_{i+1}$ at step $i+1$ are given by the following equations:
\begin{equation*}
	\begin{split}
	\mathbb E\left[{v_{i+1}}\right] &= v_{i} e^{-\gamma {h}} - \frac{u}{\gamma}(1-e^{-\gamma {h}}) {\nabla} U(x_{i}),\\
	\mathbb E\left[{x_{i+1}}\right] &= x_{i} + \frac{1}{\gamma}(1-e^{-\gamma {h}})v_{i} - \frac{u}{\gamma} \left( {h} - \frac{1}{\gamma}\left(1-e^{-\gamma {h}}\right) \right) {\nabla}  U(x_{i}),
	\end{split}   
\end{equation*}
where $h<1$ is the step size, and the choices of $\gamma=2$ and $u=\frac{1}{L}$ maintain consistency. Moreover, the variance around the expected position and velocity at the step $i+1$, as well as the covariance between them, are encapsulated by the following equations:
\begin{equation*}
	\begin{split}
	\mathbb V_x & = \frac{2u}{\gamma} \left[{h}-\frac{1}{2\gamma}e^{-2\gamma{h}}-\frac{3}{2\gamma}+\frac{2}{\gamma	}e^{-\gamma{h}}\right] \cdot \mathbf{I}_{d\times d}\\
	\mathbb V_v & = u(1-e^{-2\gamma  {h}})\cdot \mathbf{I}_{d\times d}\\
	\mathbb K_{x,v}=\mathbb K_{v,x} & = \frac{u}{\gamma} \left[1+e^{-2\gamma{h}}-2e^{-\gamma{h}}\right] \cdot \mathbf{I}_{d \times d}.
	\end{split}   
\end{equation*}
Following $I$th iterations of the sampling process, we yield $v_I$ directly. For $x_I$, we adopt a resampling strategy, adhering to a normal distribution with variance $\frac{1}{nL_a\rho_a}\mathbf{I}_{d\times d}$. A proper resampling step accelerates the mixing rate and facilitates our theoretical analysis. 

\textbf{Posterior Convergence Analysis}

Sampling from an approximate posterior in TS demands careful control of approximation errors, as an improper approximation may lead to linear regrets \citep{phan2019thompson}. In pursuit of sublinear regrets, we employ the strategy from \citet{mazumdar2020approximate}, which stipulates a $\tilde O\left(\frac{1}{n}\right)$ approximation error rate across rounds $n=1,2,\ldots, N$. To elaborate, W2 between the sample-generated measure $\hat\mu^{(n)}$ using ULMC and the posterior measure $\mu^{(n)}$ at each respective round should maintain $\tilde O\left(\frac{1}{n}\right)$. To attain such a rate, it is required to further assume that the log-likelihood is both globally strongly convex and Lipschitz smooth. 

\begin{assump}{4}[Global assumptions on the likelihood]
\label{assumption_lipschitz_global_main}
$\log \mathtt P\left({\mathcal R}|x\right)$ is characterized by its $L$-Lipschitz continuous gradient and its $m$-strong convexity, where $L \geq 0$ and $m > 0$. This means for all $x, y \in \mathbb R^d$, the following inequality holds:
\begin{equation*}
\begin{array}{cc}
&\left\langle\nabla_x \log \mathtt P\left({\mathcal R}|y\right),x-y\right\rangle + \frac{m}{2}\left\|x-y\right\|_2^2 \vspace{0.05 in} \\
& \leq \log \mathtt P\left({\mathcal R}|x\right) - \log \mathtt P\left({\mathcal R}|y\right) \leq \vspace{0.05 in} \\
& \left\langle\nabla_x \log \mathtt P\left({\mathcal R}|y\right),x-y\right\rangle + \frac{L}{2}\left\|x-y\right\|_2^2.
\end{array}
\end{equation*}
\end{assump}
These prerequisites are imperative for the efficacious application of ULMC. In scenarios where the algorithm employs stochastic gradients, an additional prerequisite is the joint Lipschitz smoothness of the likelihood function concerning both rewards and parameters. 
\begin{assump}{5}[Further assumption on the likelihood]
\label{assumption_sgld_lipschitz_main}
For the gradient of $\log \mathtt P_a\left({\mathcal R}|x\right)$, there exists a constant $L'\geq 0$\footnote{For simplicity, we let the Lipschitz constants $L=L'$ to be consistent in our paper.} such that for all $x, y \in \mathbb R^d$, ${\mathcal R}, {\mathcal R}'\in \mathbb R$, a stronger version of Lipschitz smooth condition holds:
\begin{equation*}
\left\|\nabla_x \log \mathtt P\left({\mathcal R}|x\right)-\nabla_x \log \mathtt P\left({\mathcal R}'|x\right)\right\|_2
\leq L\left\|x-y\right\|_2 + L'\left\|{\mathcal R}-{\mathcal R}'\right\|.
\end{equation*}
\end{assump}
With this assumption and by wisely selecting the batch size for the number of stochastic gradient estimates, the algorithm can accomplish logarithmic regrets with a sample complexity of $\tilde O(\sqrt{d})$, indicative of the algorithm's efficacy in solving high-dimensional problems.
\begin{theorem}[Convergence of Underdamped Langevin Monte Carlo]
\label{theorem_ula_stochastic_main}
Assume that the likelihoods, rewards, and the prior satisfy Assumptions \ref{assumption_reward}-\ref{assumption_sgld_lipschitz_main}. If we take step size $h = \tilde O\left(\frac{1}{\sqrt{d}}\right)$, number of steps $I=\tilde O\left(\sqrt{d}\right)$, and batch size $k=\tilde O\left(\kappa^2\right) $ in Algorithm \ref{algorithm_langevin_mcmc}, we have convergence of ULMC in W2 to the posterior $\mu^{(n)}$: 
$W_2\left(\hat \mu^{(n)}, \mu^{(n)}\right) \leq \frac{2}{\sqrt{n}} \hat{D}$, where $\hat{D}\geq 8\sqrt{\frac{d}{m}}$.
\end{theorem}
We focus on the convergence result for the stochastic gradient version, considering the full gradient approach as its more general counterpart. For deeper insights, readers can refer to Theorem \ref{theorem_ula_stochastic} in Appendix \ref{section_posterior_convergence}. Notably, the sample complexity $\tilde O(\sqrt{d})$ of our proposed algorithms outperforms the previous works with overdamped variations of $\tilde O(d)$, which is improved in terms of efficiency and reduced sample demands.

\textbf{Posterior Concentration Analysis}

The subsequent result provides guarantees for the convergence of samples generated from Algorithm \ref{algorithm_langevin_mcmc}. This theorem not only validates that the generated samples can approximate $\mu[\rho]$ well but also ensures the generated samples towards $x_*$. 

\begin{theorem}\label{theorem_empirical_concentrate}
Assume that the likelihood, rewards, and the prior satisfy Assumptions \ref{assumption_reward}-\ref{assumption_sgld_lipschitz_main}, and that arm $a$ has been chosen $\mathcal L_a(n)$ times up to rounds $n$. We further state that the step size $h^{(n)} = \tilde O\left(\frac{1}{\sqrt{d}}\right)$, number of steps $I = \tilde O\left(\sqrt{d}\right)$, and batch size $k=\tilde O\left(\kappa^2\right) $ in Algorithm \ref{algorithm_langevin_mcmc}. Then for $x\sim\hat\mu[\hat\rho]$ follows \eqref{eq_sample_ula2}, $x_*\in\mathbb R^d$, and $\delta \in \left(0, e^{-0.5}\right)$, the following concentration inequality holds:
\begin{equation*}
	\mathbb{P}_{x \sim \hat \mu^{(n)} [\hat \rho]} \left(\|x - x_*\|_2 \geq \sqrt{\frac{36e}{m n} \left( D + 4\hat\Omega\log{\frac{1}{\delta}} \right)} \right)\leq\delta,
\end{equation*}
where $D = 8d + 2\log B$, $\hat\Omega =16\kappa^2d+256+\frac{d}{36\kappa \hat\rho}$, and $\hat \mu^{(n)} [\hat \rho]$ serves as the scaled approximate posterior at round $n$, with the scale parameter $\hat \rho$, to approximate $\mu[\rho]$.
\end{theorem}
In contrast to our results with the posterior concentration rates from Theorem \ref{theorem_posterior_concentration}, two notable distinctions are worthy of further discussion. Firstly, the magnitude increases from $\sqrt{\frac{2e}{mn}}$ to $\sqrt{\frac{36e}{mn}}$, which is an outcome of introducing approximation error into the posterior concentration analysis. Additionally, the term $\frac{\kappa d}{18 \hat\rho}$ within $\hat\Omega$ serves to bound errors arising in the last resampling step of Algorithm \ref{algorithm_langevin_mcmc}. For an in-depth exploration, readers can be directed to \ref{section_approx_concentration} in the Appendix.

\section{REGRET ANALYSIS}

We now turn our attention to analyzing the regrets incurred by TS when the posterior follows SDEs \eqref{eq_sde2} or is approximated by ULMC \eqref{eq_sample_ula2}. For the sake of clarity in analysis, we adopt the convention of treating the first arm as optimal, resulting in zero regrets upon its selection, while pulling other arms would lead to an increase in the expected regrets, without any loss of generality. A further discussion of the unique optimal arm setting can refer to Appendix A of \citet{agrawal2012analysis}.

To facilitate our theoretical examination, we introduce several essential notations. Let $\mathcal L_a(n)$ represent the number of plays of the sub-optimal arm $a$ after round $n$. We then define $\mathcal F_{n}=\sigma\left(A_1, {\mathcal R}_1, A_2, {\mathcal R}_2, A_3, {\mathcal R}_3, \cdots, A_{n}, {\mathcal R}_{n}\right)$ is a $\sigma$-algebra generated by Algorithm \ref{algorithm_thompson_main} after playing arms $n$ times. We also denote an event $\mathcal E_a(n)=\left\{{\mathcal R}_{a, n}\geq \bar{{\mathcal R}}_1-\epsilon\right\}$ to indicate the estimated reward of arm $a$ at round $n$ exceeds the expected reward of optimal arm by at least a positive constant $\epsilon$, and its probability is denoted as $\mathcal G_{a, n} = \mathtt P\left(\mathcal E_a(n)|\mathcal F_{n-1}\right)$. Lastly, the reward difference between the optimal arm and arm $a$ is represented by $\Delta_a = \bar{\mathcal R}_1-\bar{\mathcal R}_a$.

To understand the algorithm's behavior, we first establish an upper bound on the expected number of times the algorithm plays sub-optimal arm $a$ ($a\in \mathcal{A}$, $a\neq 1$). It is clear from our setup that the regrets would increase only if we played suboptimal arms. As indicated by \citet{agrawal2012analysis, lattimore2020bandit}, after playing sub-optimal arms extensively, their posteriors become precisely estimated, and they will no longer be played with high probability. Consequently, we can set upper bounds on their regrets. Based on the above analysis, the expected plays of suboptimal arm $a$ are played after $N$ rounds can be represented as $\mathbb{E}\left[\mathcal L_a(N)\right]$, which can be further separated into two parts:
\begin{equation}\label{equation_sub_play_main}
\begin{split}
\mathbb{E}\left[\mathcal L_a(N)\right] & =\mathbb{E}\left[\sum_{n=1}^N \mathbb{I}\left(A_n=a, \mathcal E_a(n)\right) + \mathbb{I}\left(A_n=a, \mathcal E_a^c(n)\right)\right]\\
& = 1+ \mathbb{E}\left[\sum_{s=0}^{N-1} \mathbb{I}\left(\mathcal G_{a, s}>\frac{1}{N}\right)\right] + \mathbb{E}\left[\sum_{s=0}^{N-1} \left(\frac{1}{\mathcal G_{1, s}} - 1\right)\right].
\end{split}
\end{equation}
For a small positive $\epsilon$, $\mathbb{E}\left[\sum \frac{1}{\mathcal G_{1, s}} - 1\right]$ reflects how closely the sample from the first arm at round $n$ reach to its true posterior, which serves as a strong indicator favoring immediate exploitation of the first arm. However, the magnitude of this term diminishes with large variance, thereby suggesting a possible necessity for further exploration. The expectation of $\sum \mathbb{I}\left(\mathcal G_{a, s}>\frac{1}{N}\right)$, on the other hand, offers a contrasting perspective by evaluating the likelihood that the choice of arm $a$ is nearly as optimal as the first arm, which functions as an effective metric for assessing the similarity between the selected arm and the optimal arm. A low value for this term, especially with low perturbation variance, indicates that the selected arm may not be worth further exploration. Therefore, achieving an optimal equilibrium between these two terms would enable more enlightened decision-making and keep the balance between exploration and exploitation.

We subsequently integrate the concentration findings from Section \ref{section_posterior_analysis} with the exploration and exploitation terms in \eqref{equation_sub_play_main} to extract the relevant concentration guarantees.

\subsection{Regrets for Exact Thompson Sampling}\label{section_regret_exact}

We first analyze the regrets of exact TS, where the posterior distribution is described by limiting distributions governed by SDEs \eqref{eq_sde2}. Within rounds $n = 1, 2, \ldots, N$, we establish a lower boundary for $\mathcal G_{1, n}$, which serves to establish a minimum threshold for the probability of the first arm being optimistic:
\begin{equation*}\label{eq_exact_anti_concentration}
	\mathbb{E}\left[\frac{1}{\mathcal G_{1, n}}\right] \leq C \sqrt{\kappa_1 B_1}.
\end{equation*}
This bound is dependent on the prior quality $B_1$ and the condition number $\kappa_1$ of the optimal arm. While it yields an almost sure concentration of the first arm's unscaled posterior to true posteriors when $\mathcal L_1(n)$ is sufficiently large, it may fail when $\mathcal L_1(n)$ is small. Intuitively, when $\mathcal L_1(n)$ is large, it implies that the obtained reward, $r_{1,n}$, tends to center around $\bar r_1$; however, a small $\mathcal L_a(n)$ can yield a reward falling below $\bar r_1-\epsilon$. In our proof, 
setting $\rho_1$ as $\kappa_1^{-3}\left(8 d\right)^{-1}$ requires the SDEs \eqref{eq_sde2} to concentrate around a scaled posterior. This introduces a requisite variance in rewards to counterbalance any potential underestimation biases, thus yielding the highlighted anti-concentration result. The choice of $\rho_1$ is also constrained by the squaring term of the sub-Gaussian random variable $\left\|\nabla f_n(x_*)\right\|_2^2$. An overly large $\rho_1$ would also fail to derive the concentration bound.

Integrating the established anti-concentration bound with the conclusions in Theorem \ref{theorem_posterior_concentration} enables more problem-dependent bounds of \eqref{equation_sub_play_main}. Specifically, with the choice of $\rho_1=\kappa_1^{-3}\left(8 d\right)^{-1}$, we can further derive the following upper bounds in \eqref{equation_sub_play_main}:
\begin{equation*}\label{eq_exact_regret1_main}
\mathbb{E}\left[\sum_{s=0}^{N-1} \left(\frac{1}{\mathcal G_{1, s}} - 1\right)\right] \leq \left\lceil \frac{C_1 \sqrt{\kappa_1 B_1}}{m_1\Delta_a^2}\left(D_1+\Omega_1\right) \right\rceil +1,
\end{equation*}
\begin{equation*}\label{eq_exact_regret2_main}
\mathbb{E}\left[\sum_{s=0}^{N-1} \mathbb{I}\left(\mathcal G_{a, s}>\frac{1}{N}\right)\right]  \leq\frac{ C_2}{m_a\Delta_a^2}\left(D_a+\Omega_a\right) ,
\end{equation*}
where for all arms $a \in \mathcal{A}$, $D_a=\log B_a + d^2\kappa_a^3$, $\Omega_a=\kappa_a^2d+\kappa_a^3d$, and $C_1, C_2>0$ are constant which is not related to any parameter related to the algorithm. 

Based on the above results to constrain the expected plays for the sub-optimal arms, we can now derive a bound for the total expected regrets of the proposed exact TS algorithm.

\begin{theorem}\label{theorem_exact_regret_main}
  Suppose the likelihoods, rewards, and priors satisfy Assumptions \ref{assumption_likelihood}-\ref{assumption_prior}. With the choice of $\rho_a=\kappa_a^{-3}\left(8d\right)^{-1}$, we have that the total expected regrets after $N>0$ rounds of TS with exact sampling satisfies:
\begin{equation*}
  \begin{split}
      \mathbb{E}[{\mathfrak R}(N)] & \leq \sum_{a>1} \frac{C_{a}}{\Delta_a}\left(\log B_a + d^2 + d\log N\right) \\
      & \ \ \ \ + \frac{C_1\sqrt{B_1}}{\Delta_a}\left(\log B_1 + d^2\right) + 2\Delta_a ,
  \end{split}
\end{equation*}
where $C_1, C_a>0$ are universal constants and are unaffected by parameters specific to the problem.
\end{theorem}
According to the results, the first term evaluates the probability of arm $a$ being nearly optimal among all choices. Its logarithmic growth with increasing $N$ highlights there is no need for further exploration for large $N$. The second term, $\sqrt{B_1}\left(\log B_1 + d^2\right)/\Delta_a$, estimates how closely the sample from the first arm in the $n$th round aligns with its true posterior, consistent across rounds. With a sufficient number of rounds, the proposed algorithm is inclined to select the optimal arm and no further incur regrets. The proposed algorithms exhibit logarithmic regrets $\tilde O\left(\frac{\log N}{\Delta_a}\right)$, which is equivalent to the overdamped version of TS. Nonetheless, advancements in sample complexity denote an improvement over the previous works. Due to the considerable dependency of outcomes on the prior quality, further experimental analysis is conducted in Section \ref{experiments} for the reliance of results on prior quality. For the in-depth derivation of the proof, one can refer to Appendices \ref
{section_regret_general_ts} and \ref{section_regret_exact_ts}.

\subsection{Regrets for Approximate Thompson Sampling}

Building upon the regret analysis of exact TS and the concentration results for approximate posteriors produced by ULMC, we present the regret findings for approximate TS implemented with ULMC.

We first denote $\hat {\mathcal G}_{a,n}=\mathtt P\left(\mathcal E_a(n)|\mathcal F_{n-1}\right)$ as the distribution of samples from the approximate posterior $\hat \mu_a$ for arm $a$ after $n$ rounds. Consistent with the anti-concentration revelations of ${\mathcal G}_{a,n}$, we also provide concentration guarantees for $\hat{\mathcal G}_{a,n}$:
\begin{equation*}
	\mathbb{E}\left[ \frac{1}{\hat{\mathcal G}_{1,n}} \right] \leq C\sqrt{B_1},
\end{equation*}
which uniquely do not depend on the condition number $\kappa_1$. This characteristic is due to the fact that in $\hat{\mathcal G}_{a,n}$, the generated sample $x_{a, I h}$ conforms to a normal distribution, while in ${\mathcal G}_{a,n}$, a strongly-log concave characteristic is utilized for approximating normal distributions. As a result, the approximation introduces $\kappa_a$. Additionally, the choice of the scaled parameter, denoted as $\hat\rho_a=\left(8\kappa_a\Omega_a\right)^{-1}$, is strategically made to bound the moment generating function of the random variable $\left\|x_{a, I h}-x_{a,*}\right\|_2^2$, where $\left\|x_{a, I h}-x_{a,*}\right\|_2^2$ is denoted as the L2 norm between samples generated for TS and the point $x_{a,*}$. With the above information, we continue to derive a similar upper bound for the terms in \eqref{equation_sub_play_main}:
   \begin{equation*}\label{eq_approx_regret_main}
        \begin{split}
               \sum_{s=0}^{N-1}\mathbb{E}\left[ \frac{1}{\hat {\mathcal G}_{1,s}}-1\right] \leq  \left\lceil \frac{C_1\sqrt{B_1}}{m_1\Delta_a^2}\left(\log B_1+d\kappa_1^2 \log N +d^2\kappa_1^2 \right) \right\rceil +1\\
   \sum_{s=0}^{N-1} \mathbb{E}\left[ \mathbb{I}\left(\hat {\mathcal G}_{a,s}>\frac{1}{N}\right)\right]  \leq \frac{C_2}{m_a\Delta_a^2}\left(d+\log B_a+d^2\kappa_a^2\log N \right).
        \end{split}
   \end{equation*}
With these problem-dependent bounds, we reach the regret bounds for the proposed approximate TS algorithm:
\begin{theorem}
\label{theorem_approximate_ts_regret_main}
Suppose the likelihoods, rewards, and priors satisfy Assumptions \ref{assumption_reward}-\ref{assumption_sgld_lipschitz_main}. Given $\hat\rho_a=\left(8\kappa_a\Omega_a\right)^{-1}$ and sampling schemes specified in Theorem \ref{theorem_ula_stochastic_main}, then the total expected regrets after $N$ rounds of approximate TS conform to:
\begin{equation*}\begin{split}
      \mathbb{E}[\mathfrak R(N)] & \leq \sum_{a>1} \frac{\hat C_a}{ \Delta_a}\left( d+ {\log B_a}+d^2\kappa_a^2 \log N  \right)\\ &\ \ \ \  +\frac{\hat C_1 \sqrt{B_1} }{\Delta_a}\left( \log B_1+d\kappa_1^2\log N+d^2\kappa_1^2 \right) + 4\Delta_a.
\end{split}\end{equation*}    
  where $\hat C_1,\hat C_a>0$ are universal constants that are independent of problem-dependent parameters.
\end{theorem}

\begin{figure*}[hbtp]
\centering
\subfigure[Regrets with full or stochastic gradients.]{
\centering\includegraphics[width=2.15 in]{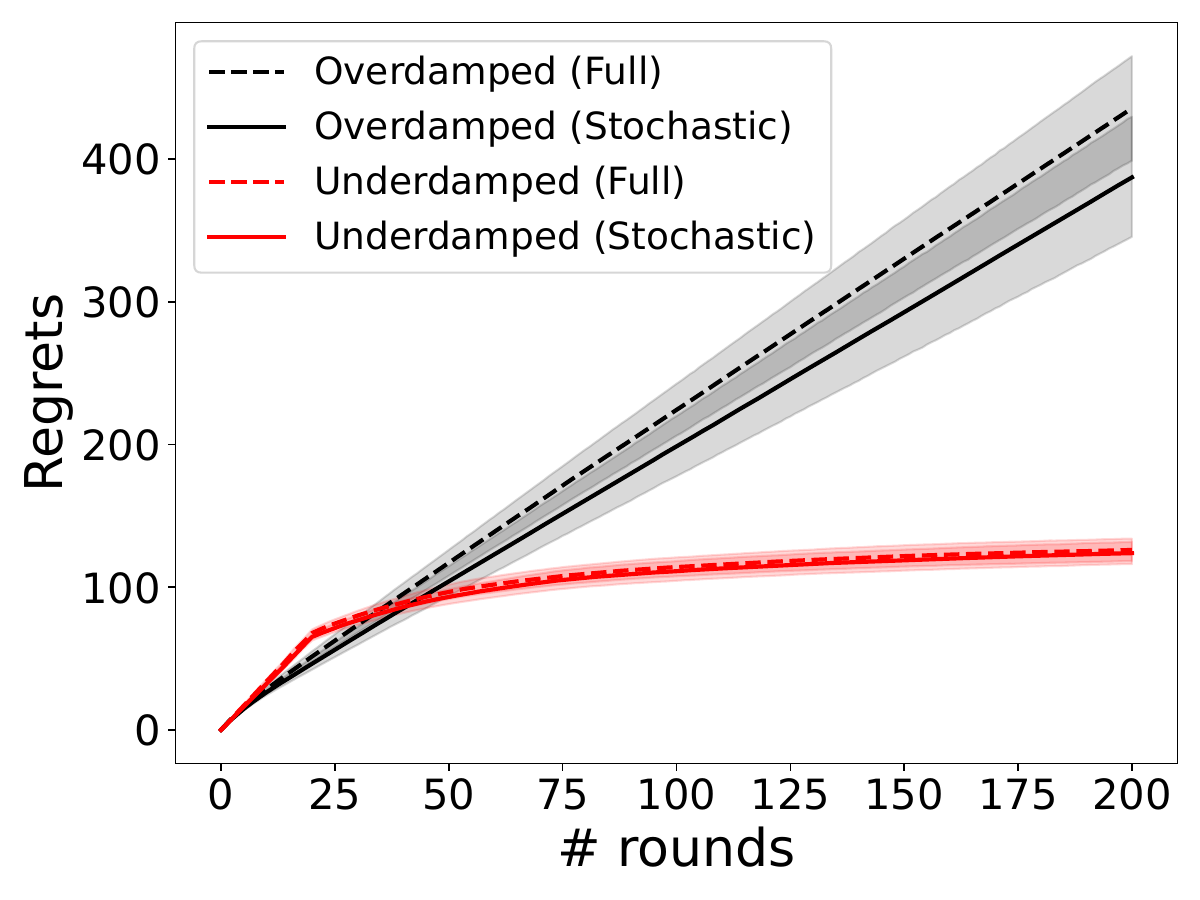}\label{fig_compare1}}
\subfigure[Regrets with $\mathcal{\tilde O}(\sqrt d)$ samples.]{
\centering\includegraphics[width=2.15 in]{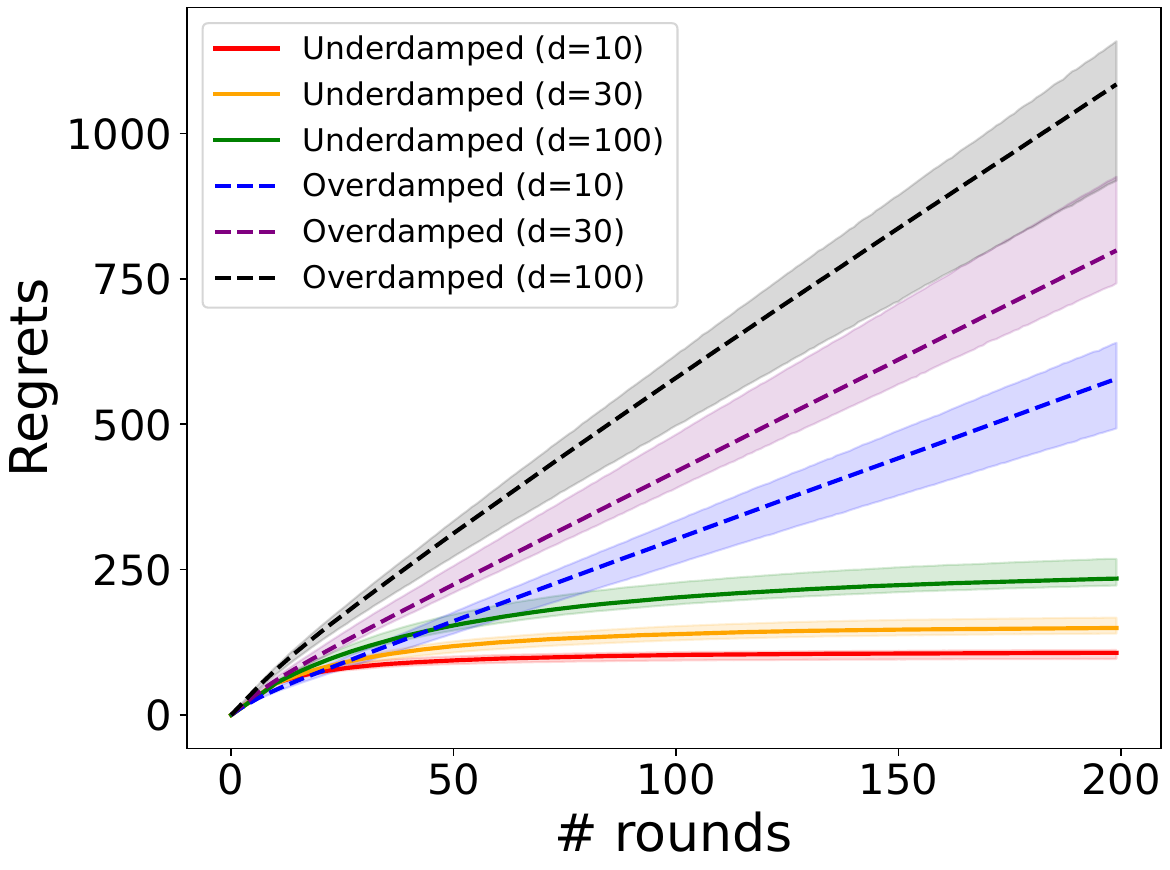}\label{fig_sample1}}
\subfigure[Regrets with $\mathcal{\tilde O}( d)$ samples.]{
\centering\includegraphics[width=2.15 in]{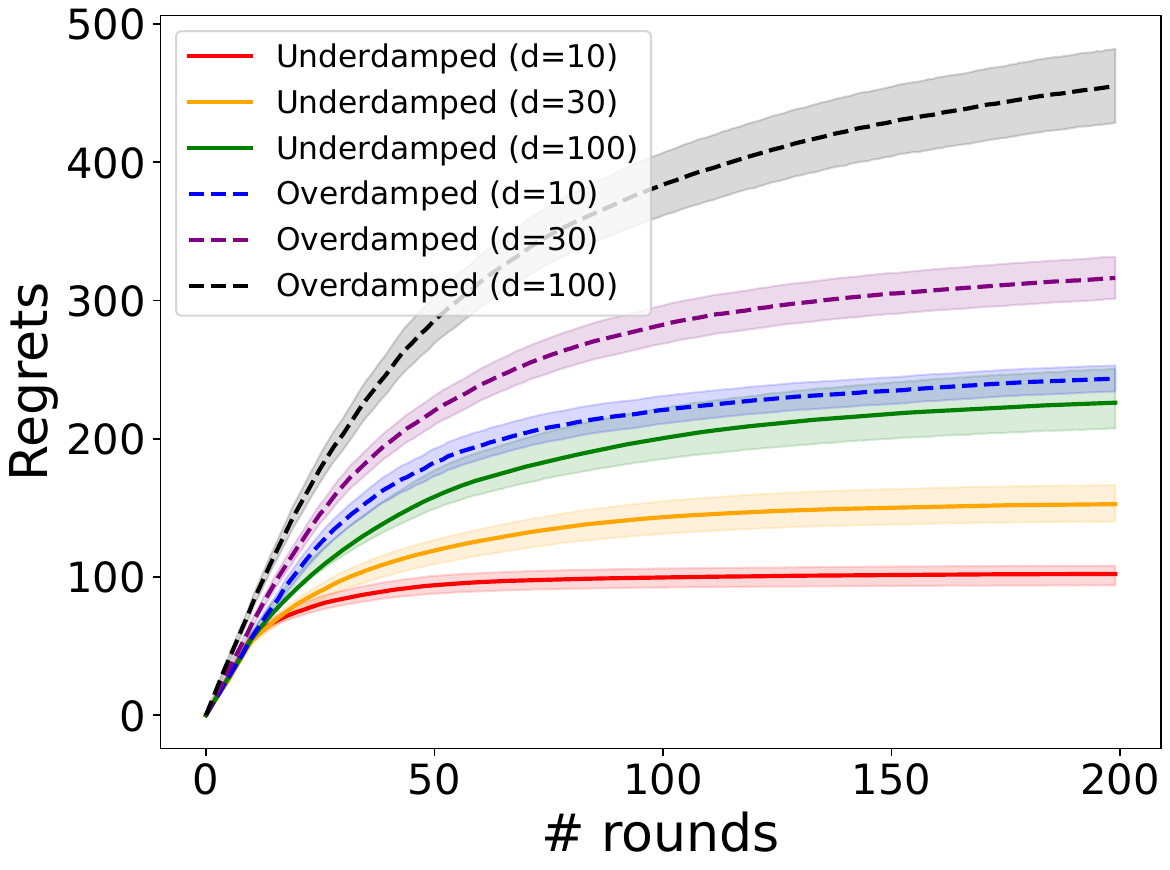}\label{fig_sample2}}

\subfigure[Regrets with different momentums $\gamma$.]{
\centering\includegraphics[width=2.15 in]{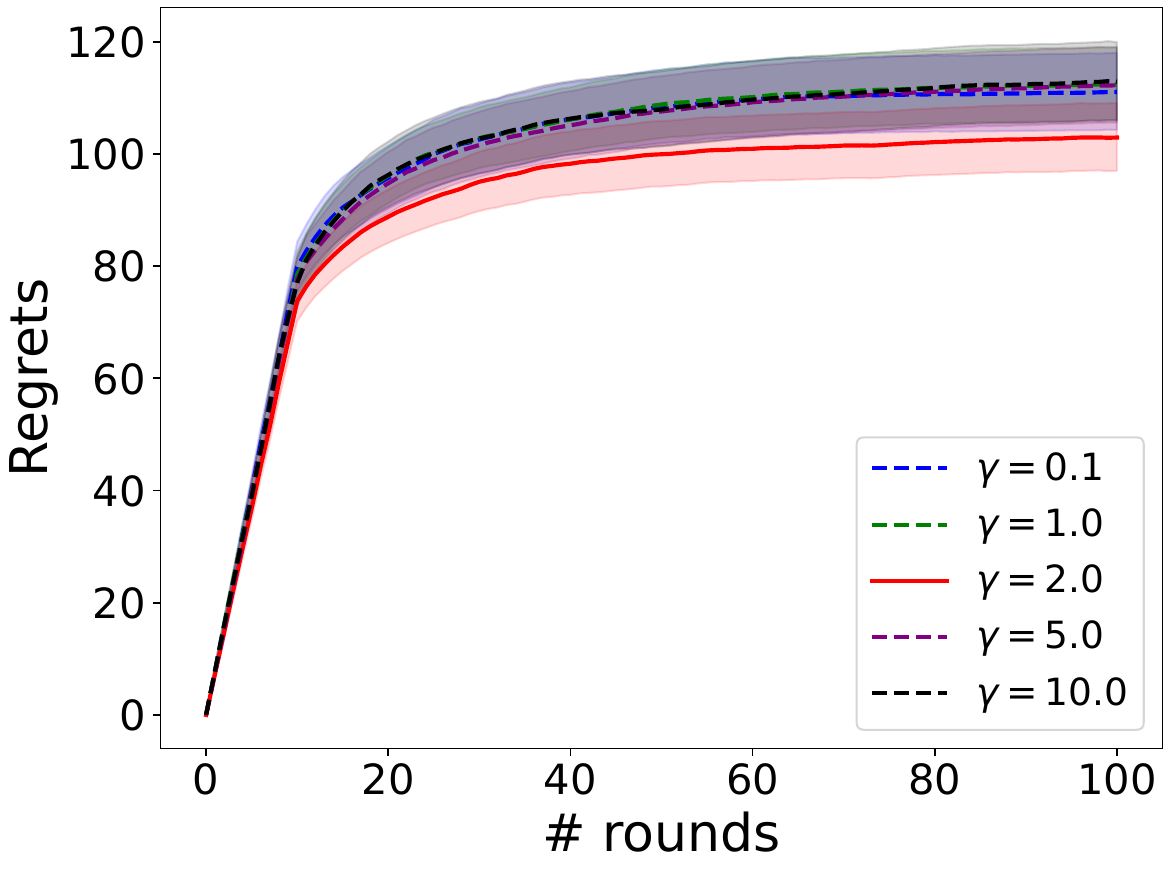}\label{fig_momentum}} 
\subfigure[Regrets with flat priors.]{
\centering\includegraphics[width=2.15 in]{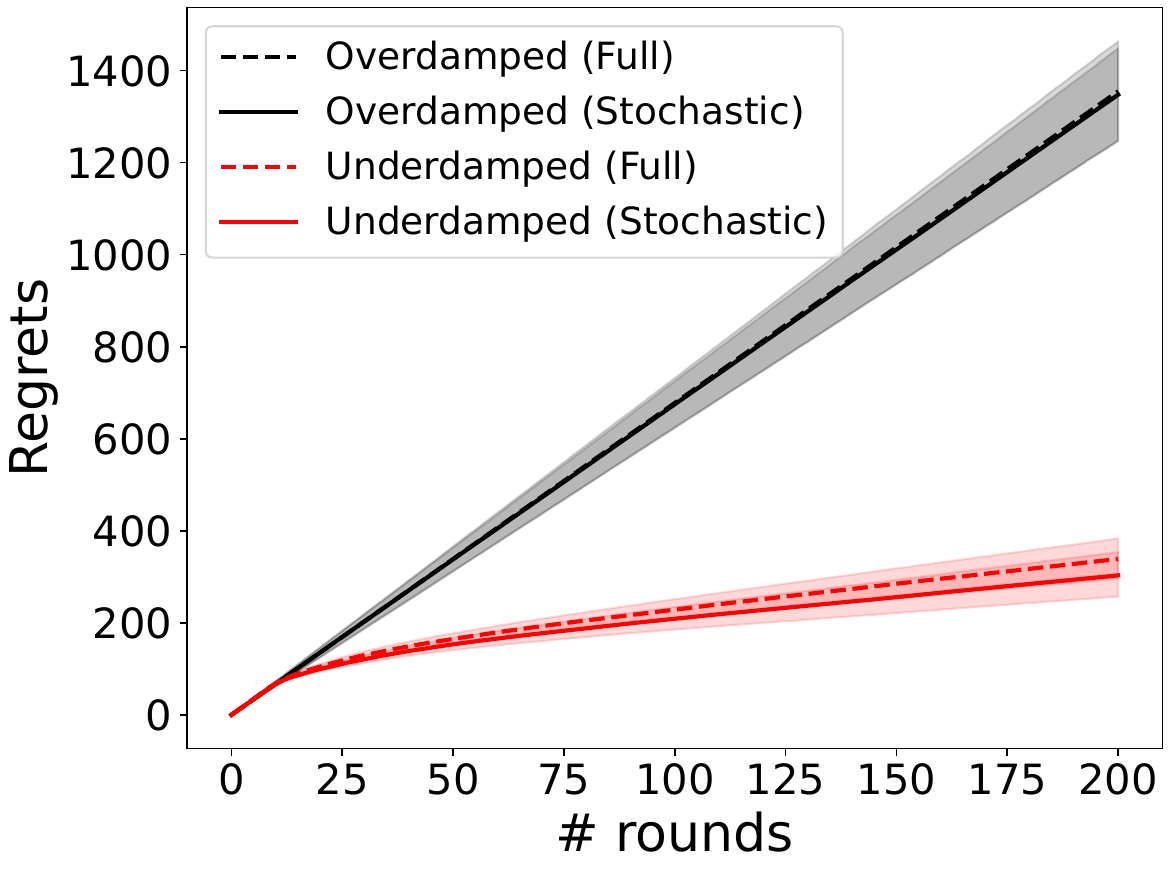}\label{fig_prior}} 
\subfigure[Regrets in high dimensional cases.]{
\centering\includegraphics[width=2.15 in]{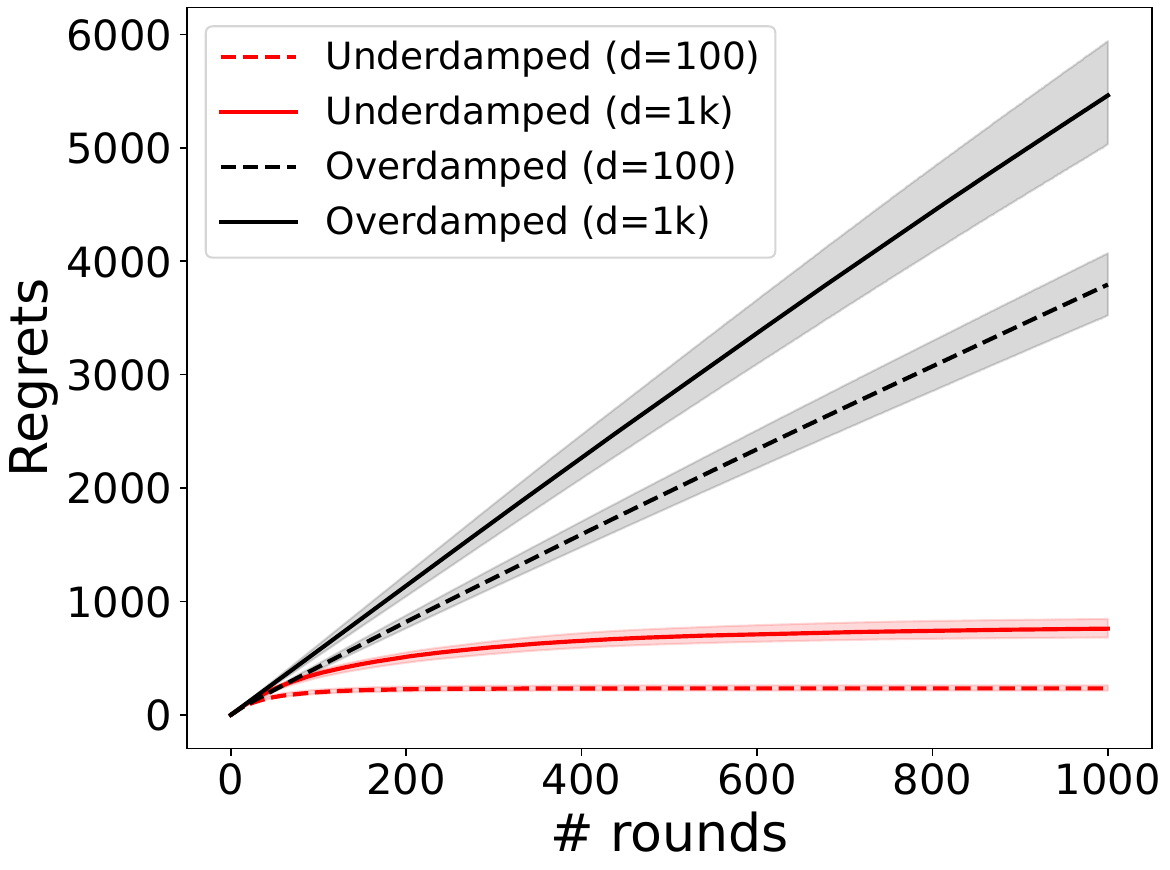}\label{fig_dimension2}}
\caption{Regret comparisons of TS among different settings.}\label{fig_regret}
\end{figure*}

Comparing the regrets provided in both exact and approximate TS, we observe that the first term still demonstrates a logarithmic rate at $\tilde O\left(\frac{\log N}{\Delta_a}\right)$. Interestingly, its dependence on the parameter dimensions increases from linear to quadratic. This change is attributed to the distance between the unscaled and scaled approximate posteriors when establishing the contraction results of the approximate posterior, which yields a dimension-related result. While the second term in exact TS is independent of $N$, the result in approximate TS behaves at a rate of $\tilde O\left(\frac{\sqrt{B_1}\log N}{\Delta_a}\right)$. This underscores the sensitivity of approximate TS to inferior prior quality. 

It should be noted that our proof achieves similar regrets as previous works \citep{mazumdar2020approximate} but under less stringent sample complexity assumptions. With the same \({\mathcal{\tilde O}(\sqrt{d})}\) samples, our method achieves logarithmic regrets, a guarantee absent in previous research. With \(\mathcal{\tilde O}(d)\) samples, our method accelerates posterior approximation and facilitates better convergence around the posterior mode, which indicates better regret bounds with the same \(\mathcal{\tilde O}(d)\) samples. Further discussion will be offered in the next section, and detailed theoretical derivations can be found in \ref{section_regret_approx_ts}.

\section{EXPERIMENTS}\label{experiments}

In this section, we present numerical experiments designed to validate the benefits of the proposed methods on MAB from different perspectives\footnote{Code is available at \href{https://github.com/haoyangzheng1996/ts_ulmc}{the GitHub repository}.}. We regard TS with LMC as a benchmark (both full gradient and stochastic gradient versions) across our tests. We first demonstrate general settings among all experiments without specific illustrations. For each round, the agent observes samples from ten-dimensional posterior distributions, and the agent has to choose action $A_n$ between ten arms and receives rewards $\mathcal R_n$ to update the posterior of arm $A_n$. We generated 500 trajectories to get the expected mean of each result and then applied Bootstrapping to obtain their 95\% confidence intervals. More details are provided in Appendix \ref{section_experiments}.

We first compare the performance of the proposed algorithms and the benchmarks in Figure \ref{fig_regret}. Across all figures, the x-axis is the number of rounds to play arms, and the y-axis is the total expected regrets generated by different algorithms among various settings. Each line represents the average regrets among the generated trajectories, accompanied by a 95\% confidence interval represented by a lighter shade. When we have the same step size and sample complexity $\mathcal{\tilde O}(\sqrt d)$, it is clear to see from Figure \ref{fig_compare1} that at a consistent sample complexity and step size, the proposed algorithm showcased logarithmic regrets while the overdamped algorithm incurred linear regrets. This disparity can be attributed to the overdamped version's inability to provide approximate error guarantees with large step sizes. Also, the figure indicates that a proper choice of batch size enables the proposed method to achieve logarithmic regrets relying only on stochastic gradients. 

We continue to explore the derived regrets under different sample complexity and different dimension settings. Figure \ref{fig_sample1} demonstrates the regrets when we employ $\mathcal{\tilde O}(\sqrt d)$ samples and Figure \ref{fig_sample2} employs $\mathcal{\tilde O}(d)$ samples. With $\mathcal{\tilde O}(\sqrt d)$ samples, it is clear that as the dimensions increase, there is a noticeable trend of slower regret convergence and higher regret magnitudes. Nevertheless, the algorithm persisted in delivering logarithmic regrets across all dimensional variations. On the contrary, the overdamped algorithm can only incur linear regrets. With $\mathcal{\tilde O}( d)$ samples, the overdamped algorithm succeeds in achieving logarithmic regrets, but with the underdamped algorithm, the incurred regrets bounds are much tighter than the overdamped one. 

We then further explore the stability of the proposed algorithm from different perspectives. We first explore the impact of various momentum (denoted by $\gamma$) values in Figure \ref{fig_momentum}. Across different momentum values, the ULMC consistently displayed logarithmic regrets, which also indicates that our algorithms are stable among different momentum settings. It is worth noting that a gamma value of $2.0$ incurs the smallest regrets, which aligns with our theoretical selection. 

As we mentioned earlier in our analysis, the regrets may be heavily dependent on prior quality, and Figure \ref{fig_prior} shows how the proposed method and benchmark behave when we consider flat priors. While the ULMC consistently outperformed the overdamped version, they are likely to incur linear regrets in smooth ascent. However, this milder linear regret growth hints at the algorithm's ability to avoid the worst decisions, even if it does not always identify the optimal ones. 

Moreover, the advantages of our proposed algorithm over the conventional TS with LMC become significant in high-dimensional challenges, as depicted in Figure \ref{fig_dimension2}. Our approach exhibits much smaller regrets when dimensions increase to 100 and 1,000, which emphasizes the need for such advancements.

\section{CONCLUSIONS AND DISCUSSIONS}

To address the intricate task of sampling multiple approximate posteriors and guiding sequential decisions toward the optimal posteriors, we introduced a novel strategy using TS with ULMC to improve the approximation accuracy. The main contribution is the novel posterior analysis in the use of a specific potential function, which offers new insights into posterior concentration rates in TS. Based on this, the proposed algorithm offers a more favorable sample complexity $\tilde O(\sqrt{d})$ relative to the overdamped counterpart with $\tilde O(d)$ -- a claim validated through both theoretical analysis and experimental validation. 

Our algorithm demonstrates a consistent performance by incurring logarithmic regrets with the sample complexity of both $\tilde O(\sqrt{d})$ and $\tilde O({d})$ comparable to that of the overdamped variant, which, by contrast, exhibits worse regret performance. Furthermore, the integration of variance reduction techniques, particularly control variates \citep{zou2018stochastic, zou2019stochastic}, into the approximate TS algorithm with stochastic gradients is worthy of additional discussion here. These techniques have proven to significantly diminish noise variance in stochastic gradients, which contributes to accelerating sampling convergence. Consistent with the prior theoretical framework outlined in \citet{mazumdar2020approximate}, our results also imply that, with variance reduction techniques, the sample complexity of approximate TS can be improved with a better constant, but not necessarily with a better order. 

The robustness of the proposed algorithm was further supported by its consistent performance across various experiments: Our algorithms consistently perform well across a wide range of momentum values, which is consistent with our theory. In the face of challenges from the curse of dimensionality, it tends towards near-optimal arms and maintains logarithmic regret convergence. Its consistent and excellent performance becomes significant when comparing it and TS with LMC in high-dimensional settings. We also provide experiments to compare algorithms considering non-informative priors, but for our theoretical analysis, we highlight that we assume sufficiently good priors to target logarithmic regrets. It should be noted that current literature on prior sensitivity focuses more on simple cases and does not align with our settings. The impact of considering non-informative priors is an important aspect we aim to explore in future work. 

This study has tentatively shown that our algorithms can align with theoretical regret bounds under weaker sample complexity assumptions and empirical evidence suggests that our algorithms could potentially derive tighter regret bounds. However, it is critical to recognize that we fall short of offering a theoretical guarantee for achieving consistently tighter regret bounds with equivalent sample complexities. Additionally, our method does not reach the minimax optimal regret benchmarks established in simpler TS settings \citep{jin2021mots}. Our subsequent research will aim to develop a tight, optimal regret bound specifically to the approximate TS algorithm for MAB problems.

Another interesting topic for approximate TS is the extension to non-convex scenarios, which is an important but non-trivial future direction. With the dissipative or log-Sobolev assumptions, we aim to achieve exponential convergence in continuous time. Achieving this would not only validate the approach theoretically but also enhance the practical deployment of approximate TS across more diverse settings, which marks a substantial leap forward in the domain of TS and MCMC.

\section*{Acknowledgments}
We thank the anonymous reviewers for their helpful comments. G. Lin acknowledges the support of the National Science Foundation (DMS-2053746, DMS-2134209, ECCS-2328241, and OAC-2311848), and U.S. Department of Energy (DOE) Office of Science Advanced Scientific Computing Research program DE-SC0023161, and DOE–Fusion Energy Science, under grant number: DE-SC0024583. 

%% file: appendix.tex
\section{INTRODUCTION TO THOMPSON SAMPLING}

Before we move on to the detailed proof, we here remind the general framework for the Thompson sampling algorithm. Thompson sampling is a probabilistic approach used for solving the multi-armed bandit problem, emphasizing the trade-off between the exploration of less-known arms and the exploitation of the best-known arms. As outlined in Algorithm \ref{algorithm_thompson}, for each round, the algorithm samples from the scaled posterior distributions of each arm and then chooses the arm with the highest probability of getting optimal rewards. Once an arm is played and the reward is observed, the associated posterior distribution is updated. Through continuous updates and sampling, it enables the algorithm to quantify its performance in terms of total expected regrets.

\begin{algorithm*}[!htbp]
\caption{A general framework of Thompson sampling for multi-armed bandits.}\label{algorithm_thompson}
{\hspace*{\algorithmicindent} \textbf{Input} Bandit feature vectors $\alpha_{a}$ for $\forall a \in \mathcal{A}$.}\\
{\hspace*{\algorithmicindent} \textbf{Input} Scaled posterior distribution $\mu_a\left[\rho_a\right]$ for $\forall a \in \mathcal{A}$.}
\begin{algorithmic}[1]
\For{{$n$=1 to $N$}}
\State {Sample $\left(x_{a, n}, v_{a, n}\right)\sim \mu_a\left[\rho_a\right]$ for $\forall a \in \mathcal{A}$.}
\State {Choose arm $A_n = \text{argmax}_{a\in\mathcal A} \left\langle \alpha_a, x_{a, n}\right\rangle$.}
\State {Play arm $A_n$ and receive reward ${\mathcal R}_{n}$.}
\State {Update posterior distribution of arm $A_n$: $\mu_{A_n}\left[\rho_{A_n}\right]$.}
\State {Calculate expected regrets ${{\mathfrak R}}_n$.}
\EndFor
\end{algorithmic}
\hspace*{\algorithmicindent} {\textbf{Output}  Total expected regrets $\sum_{n=1}^N {{\mathfrak R}}_n$.}
\end{algorithm*}

\section{PROOF OUTLINE}
To make the reader easier to follow the proof, we outline the general idea of our proof as follows:

\begin{enumerate}
    \item Decompose regrets into exploration and exploitation terms ({Section \ref{section_regret_general_ts}}).
    \item Derive the posterior concentration rates ({Section \ref{section_posterior_concentration}}).
    \item Apply posterior concentration rates to bound regret terms in the first step ({Section \ref{section_regret_exact_ts}}).
    \item Determine sample size to attain \(\tilde O(1/n)\) approximate error rates ({Section \ref{section_posterior_convergence}}).
    \item With approximate error guarantee, derive the ULMC-generated sample concentration rates ({Section \ref{section_approx_concentration}}).
    \item Apply approximate posterior concentration rates to bound regret terms in the first step ({Section \ref{section_regret_approx_ts}}).
\end{enumerate}

\clearpage
\section{ANALYSIS OF EXACT THOMPSON SAMPLING}

\subsection{Notation}

Table \ref{table_notation} presents the key symbols used throughout our study on the posterior concentration analysis and the regret analysis of exact Thompson sampling, which will repeatedly appear in the following theoretical analysis. The choice of these notations helps clarity and consistency in our analyses and discussions. Readers are encouraged to refer to the table for a comprehensive understanding of the symbols employed.

\begin{table}[!htbp]
\caption{Notation summary in the analysis of exact Thompson sampling.}\label{table_notation}
\centering
\begin{tabular}{|c|c|}
\hline
Symbols & Explanations    \\ \hline  \hline

  $\llbracket A, B \rrbracket$ &   set of all integers from $A$ to $B$ ($A, B\in \mathbb Z$) \\ \hline 
  $d$             &  model parameter dimensions    \\ \hline
  $t$             &   time horizon $t\in [0, T)$    \\ \hline
  $n$			  &   number of rounds to play arm ($n\in \llbracket 1, N \rrbracket$)    \\ \hline
  $a$             &   arm index in the Thompson sampling ($a\in \mathcal A$)    \\ \hline
  $A_n$           &   action (the selected arm) in round $n$     \\ \hline
  ${{\mathcal R}}_{a, n}$      &   reward in round $n$ by pulling arm $a$     \\ \hline
  $\bar {{\mathcal R}}_a$      &   expected reward for arm $a$         \\  \hline 
  ${\mathfrak R}(n)$      	&   total expected regret in round $n$         \\  \hline \hline
  $x_*\ (v_*)$      		&   fixed position (velocity)    \\ \hline
  $x_a\ (v_a)$    &   inherent position (velocity) parameter for arm $a$    \\ \hline
  $x_{a, n}\ (v_{a, n})$  	&   sampled position (velocity) parameter in round $n$ for arm $a$         \\ \hline \hline

  $L_{a}$  		&   Lipschitz smooth constant for arm $a$         \\ \hline
  $m_{a}$  		&   strongly convex constant on the likelihood for arm $a$         \\ \hline
  $\nu_{a}$  		&   strongly convex constant on the reward for arm $a$         \\ \hline
  $\rho_{a}$  	&   parameter to scale posterior for arm $a$         \\ \hline
  $\kappa_a$  	&   condition number for the likelihood of arm $a$: $\kappa_a=\frac{L_a}{m_a}$        \\ \hline
  $B_a$			& 	prior quality of arm $a$: $B_a = \frac{\max_{x} \pi_a\left(x\right)}{\pi_a\left(x_*\right)}$ \\ \hline
  $\gamma$		& 	friction coefficient	\\ \hline
  $u$				& 	noise amplitude \\ \hline
  $\alpha_a$		& bandit inherent parameters for arm $a$ \\  \hline 
  $\omega_a$		& norm of $\alpha$ for arm $a$ \\ \hline
  $\Delta_a$ 	& \begin{tabular}{@{}c@{}}distance between expected rewards \\of optimal arm and arm $a$ \end{tabular}\\ \hline \hline
  $\mu_a^{(n)}$ & 	\begin{tabular}{@{}c@{}}probability measure of posterior distribution \\ of arm $a$ after $n$ rounds\end{tabular}		\\ \hline
  $\mu_a^{(n)}[\rho_a]$ & 	\begin{tabular}{@{}c@{}}probability measure of scaled posterior distribution \\ of arm $a$ after $n$ rounds\end{tabular}		\\ \hline
\end{tabular}
\end{table}

\subsection{Posterior Concentration Analysis}\label{section_posterior_concentration}

Suppose the posterior distribution {$\mu_a^{(n)}\left[\rho_a\right]\propto \exp\left(-\rho_a\left(nf_{n}(x_a)+\log\pi(x_a)+n\|v_a\|_2^2/2u\right)\right)$}, where $x_a, v_a\in \mathbb R^d$ represent the vector for position and velocity. Then the posterior distribution satisfies the following theorem.

\begin{theorem}\label{theorem_posterior_concentration2}
Suppose that the likelihood and the prior follow Assumptions \ref{assumption_likelihood} and \ref{assumption_prior} hold. Given rewards $\mathcal R_{a, 1}, \mathcal R_{a, 2},\ldots, \mathcal R_{a, n}$ then for $x_a, v_a\in \mathbb R^d$ and $\delta_1 \in \left(0, e^{-0.5}\right)$, the posterior distribution satisfies:

\begin{equation}
  \mathbb{P}_{{x_a} \sim {\mu_a\left[\rho^{(n)}_a\right]}}\left(\left\|x_a - x_*\right\|_2\geq \sqrt{\frac{2e}{m_an}\left(D_a+2\Omega_a\log 1 / \delta_1\right)}\right)\leq\delta_1
\end{equation}
with $D_a=\frac{8d}{\rho_a }+{2}\log B_a$, $\Omega_a=\frac{16L_a^2d}{m\nu_a}+\frac{256}{\rho_a}$. Here $x_* \in \mathbb R^{d}$ represents the fixed pair toward which the posterior distribution tends to concentrate.

\end{theorem}

\begin{proof}
The proof incorporates the strategies utilized in establishing Theorem 1 from \citet{mou2019diffusion} and Theorem 5 from \citet{cheng2018underdamped}. Throughout this posterior concentration proof, we generalize the conclusion across all arms. For simplicity and clarity, we omitted the arm-specific subscript in this proof. We first define the following stochastic differential equations (SDEs):
\begin{equation}\label{eq_sde}
  \begin{split}
      d v_t & =-\gamma v_t d t - u\nabla f\left(x_t\right)dt - \frac{u}{n} \nabla\log\pi(x_t) d t + 2\sqrt{\frac{\gamma  u}{n\rho}} d B_t, \\
      d x_t & =v_t d t,
  \end{split}
\end{equation}
where $x_t, v_t \in \mathbb R^d$ are position and velocity at time $t$, $f_{n}(x)\coloneqq\frac{1}{n}\sum_{j=1}^n\log \mathtt P(\mathcal R_{j}|x)$ is the log-likelihood function (we omit the subscript $n$ in the following parts for simplicity), $\log \pi(x)$ the prior, $\gamma$ is the friction coefficient, $u$ is the noise amplitude, and $B_t$ is a standard Wiener process (Brownian motion). Using a chosen value of $\alpha > 0$, we define a potential function $V:\mathbb R^d\times \mathbb R^d \times \mathbb R^+\rightarrow \mathbb R^+$:
\begin{equation}
  \begin{split}
      V\left( x, v, t\right) = \frac{1}{2}e^{\alpha t} \left\|(x, x+v) - (x_*, x_* + v_*)\right\|^2_2.
  \end{split}
\end{equation}
Importantly, as $t$ approaches infinity, the $p$-th moments of $V(x_t, v_t, t)$ can translate to the corresponding moments of $V(x, v, t)$, due to the convergence of $ (x_t, v_t) $ to $ (x, v) \sim \mu[\rho] $. Applying It${\rm{\hat o}}$'s Lemma \citep{pavliotis2014stochastic} to $V\left( x_t, v_t, t\right) $, we can decompose the potential function as follows:

\begin{equation}\label{eq_ito_decompose}
\begin{split}
  V\left( x_s, v_s, t\right) & = \frac{\alpha}{2}\int_0^t e^{\alpha s} \left\| x_s - x_*\right\|^2_2 + e^{\alpha s} \left\| {(x_s + v_s) - (x_* + v_*)}\right\|^2_2 ds \\
  & \ \ \ \ \ \ \ \ + \int_0^t e^{\alpha s} \left[\left\langle{ {\left(1 -\gamma\right)} {v_s} - u\nabla f\left( {{x_s}} \right)} - \frac{u}{n} \nabla\log\pi(x_t), {(x_s + v_s)-(x_* + v_*)}\right\rangle + \left\langle v_s, x_s-x_*\right\rangle\right] ds \\
  & \ \ \ \ \ \ \ \ + \frac{2\gamma  ud}{n\rho} \int_0^t e^{\alpha s} ds + 2\sqrt{\frac{\gamma  u}{n\rho}}\int_0^te^{\alpha s}\left\langle {(x_s + v_s)-(x_* + v_*)}, dB_s\right\rangle\\
  & \stackrel{(i)}{=} \frac{\alpha}{2}\int_0^t e^{\alpha s} \left\| x_s - x_*\right\|^2_2 + e^{\alpha s} \left\| {(x_s + v_s) - (x_* + v_*)}\right\|^2_2 ds \\
  & \ \ \ \ \ \ \ \ + \int_0^t e^{\alpha s} \underbrace{\left[ \left\langle{ \left(1 - \gamma\right)} ({v_s}-v_*) - u\left(\nabla f\left( {{x_s}} \right) - \nabla f\left( {{x_*}} \right)\right), {(x_s + v_s)-(x_* + v_*)} \right\rangle + \left\langle v_s-v_*, x_s-x_*\right\rangle \right]}_{T1} ds \\
  & \ \ \ \ \ \ \ \ +   \underbrace{\frac{2\gamma  ud}{n\rho} \int_0^t e^{\alpha s} ds}_{T2} + \underbrace{2\sqrt{\frac{\gamma  u}{n\rho}}\int_0^te^{\alpha s}\left\langle {(x_s + v_s)-(x_* + v_*)}, dB_s\right\rangle}_{T3} \\
  & \ \ \ \ \ \ \ \ + \underbrace{\int_0^t e^{\alpha s} \left\langle{\left(1 - \gamma\right)} v_* - u\nabla f\left( x_* \right), {(x_s + v_s)-(x_* + v_*)}\right\rangle ds}_{T4} \\
  & \ \ \ \ \ \ \ \ + \frac{u}{n} \underbrace{\int_0^t e^{\alpha s} \left\langle -\nabla\log\pi(x_t), {(x_s + v_s)-(x_* + v_*)}\right\rangle ds}_{T5},\\
\end{split} 
\end{equation}
where $(i)$ holds by adding and subtracting terms related to $\nabla f\left( x_* \right)$ and $v_*$. Next, we proceed to bound the terms in \eqref{eq_ito_decompose} step-by-step. We start by bounding $T1$ as follows:
\begin{equation}
  \begin{split}
      T1 & = \left\langle{ \left(1 - \gamma\right)} ({v_s}-v_*) - u\left(\nabla f\left( {{x_s}} \right) - \nabla f\left( {{x_*}} \right)\right), {(x_s+v_s)-(x_* + v_*)}\right\rangle + \left\langle v_s-v_*, x_s-x_*\right\rangle\\
      & \stackrel{(i)}{\leq} -\frac{m}{2L}\left(\left\|x_s - x_*\right\|^2_2 + \left\|(x_s+v_s) - (x_* + v_*)\right\|^2_2\right),
  \end{split}
\end{equation}
where the elaborate procedure to derive $(i)$ can be found in Lemma \ref{lemma_contraction}. $T2$ can be easily obtained by integration rules:
\begin{equation}
  \begin{split}
      T2 = \frac{2\gamma  ud}{n\rho}\int_0^t e^{\alpha s}ds & = \frac{2\gamma ud}{n\rho\alpha}\left(e^{\alpha t} - 1\right) \leq  \frac{2\gamma ud}{n\rho\alpha}e^{\alpha t} . \\
  \end{split}
\end{equation}

We proceed to bound $T3$. Suppose $M_t=\int_0^t {{e^{\alpha s}}\left\langle{(x_s + v_s)-(x_* + v_*)}, dB_s \right\rangle }$, and $T3$ can be rewritten as $2\sqrt {\frac{\gamma  u}{n\rho}}M_t $. Then we can derive a bound for $M_t$ as follows:
\begin{equation}\label{eq_mt}
  \begin{split}
      \mathbb E \left[\sup\limits_{0\leq t\leq T}{\left| {{M_t}} \right|^p}\right] & \stackrel{(i)}{\leq} \left(8p\right)^{\frac{p}{2}} \mathbb E\left[\left\langle M_t,M_t\right\rangle_T^{\frac{p}{2}}\right]\\
      & \stackrel{(ii)}{\leq} \left(8p\right)^{\frac{p}{2}} \mathbb E\left[\left(\int_0^T {{e^{2\alpha s}} \left\|(x_s + v_s)-(x_* + v_*)\right\|_2^2 ds  }\right)^{\frac{p}{2}}\right]\\
      & \stackrel{}{=} \left(8p\right)^{\frac{p}{2}} \mathbb E\left[\left(\int_0^T {{e^{2\alpha s}} \left\|(x_s + v_s)-(x_* + v_*)\right\|_2^2 ds  }\right)^{\frac{p}{2}}\right]\\
      & \stackrel{(iii)}{\leq} \left(8p\right)^{\frac{p}{2}} \mathbb E\left[\left(\sup\limits_{0\leq t\leq T}  { \frac{1}{2}e^{\alpha t}\left\|(x_t+v_t)-(x_* + v_*)\right\|_2^2 \int_0^T 2e^{\alpha s} ds  }\right)^{\frac{p}{2}}\right]\\
      & = \left(8p\right)^{\frac{p}{2}} \mathbb E\left[\left(\sup\limits_{0\leq t\leq T}  { \frac{1}{2}e^{\alpha t}\left\|(x_t+v_t)-(x_* + v_*)\right\|_2^2 \frac{2e^{\alpha T} - 2}{\alpha}  }\right)^{\frac{p}{2}}\right]\\
      & \stackrel{(iv)}{\leq} \left(\frac{16pe^{\alpha T}}{\alpha}\right)^{\frac{p}{2}} \mathbb E\left[\left(\sup\limits_{0\leq t\leq T}  { \frac{1}{2}e^{\alpha t}\left\|(x_t+v_t)-(x_* + v_*)\right\|_2^2}\right)^{\frac{p}{2}}\right],\\
      \end{split}
\end{equation}
where $(i)$ is bounded by Burkholder-Davis-Gundy inequalities (Theorem 2 in \citet{ren2008burkholder}); $(ii)$ follows quadratic variation of It${\rm{\hat o}}$'s Lemma; $(iii)$ proceeds by factoring the supremum out of the integral; $(iv)$ holds because of the linearity of the expectation and ${e^{\alpha T} - 1}<e^{\alpha T}$. For $T4$, we have:

\begin{equation}\label{eq_boundt4}
  \begin{split}
      T4 & = \int_0^t e^{\alpha s} \left\langle{ \left(1 - \gamma\right)} v_* - u\nabla f\left( x_* \right), {(x_s + v_s)-(x_* + v_*)}\right\rangle ds\\
      & \stackrel{(i)}{\leq} \int_0^t e^{\alpha s} \left\| \left(\gamma-1 \right) v_* + u\nabla f\left( x_* \right) \right\|_2\left\| {(x_s + v_s) - (x_* + v_*)}\right\|_2 ds \\
      & \stackrel{(ii)}{\leq} \int_0^t e^{\alpha s} \left\| u\nabla f\left( x_* \right) \right\|_2\left\| {(x_s + v_s) - (x_* + v_*)}\right\|_2 ds \\
      & \stackrel{(iii)}{\leq} \frac{u^2L}{m}\int_0^t e^{\alpha s} \left\| \nabla f\left( x_* \right) \right\|_2^2 + \frac{m}{4L}\int_0^t e^{\alpha s} \| {(x_s + v_s)-(x_* + v_*)}\|_2^2 ds \\
      & \leq \frac{u^2L}{m\alpha} e^{\alpha t} \left\| \nabla f\left( x_* \right) \right\|_2^2 ds + \frac{m}{4L}\int_0^t e^{\alpha s} \| {(x_s + v_s)-(x_* + v_*)}\|_2^2 ds \\
  \end{split}
\end{equation}
where $(i)$ is from Cauchy-Schwartz inequality and the selection of $\gamma$ ($\gamma \ge 1$) demonstrated in Lemma \ref{lemma_contraction}, $(ii)$ proceeds by the choice of $\left\|v_*\right\|_2=0$, and $(iii)$ is obtained by Young's inequality for products. We finally bound $T5$ as follows:

\begin{equation}\label{eq_boundt5}
  \begin{split}
      T5 & = \int_0^t e^{\alpha s} \left\langle -\nabla\log\pi(x_s), {(x_s + v_s)-(x_* + v_*)}\right\rangle ds\\
      & \stackrel{(i)}{\leq} \frac{\log B}{\alpha}\left(e^{\alpha t} - 1\right) \\
      & \stackrel{}{\leq} \frac{\log B}{\alpha}e^{\alpha t} \\
  \end{split}
\end{equation}
where $(i)$ is from Proposition \ref{prop_prior}. Incorporating the upper bound of $T1$-$T5$ into \eqref{eq_ito_decompose}, we can now express  the upper bound for $V\left( x_t, v_t, t\right)$ as:

\begin{equation}\label{eq_vt_derive}
  \begin{split}
      V\left( x_t, v_t, t\right) & = \frac{1}{2}e^{\alpha t} \left\|(x_t, x_t+v_t) - (x_*, x_* + v_*)\right\|^2_2\\
      & \leq \frac{\alpha}{2}\int_0^t e^{\alpha s} \left(\left\| x_s - x_*\right\|^2_2 +  \left\| {(x_s + v_s) - (x_* + v_*)}\right\|^2_2\right)  ds \\
      & \ \ \ \ \ \ \ \ -  \frac{m}{2L}\int_0^t e^{\alpha s}\left(\left\|x_s - x_*\right\|^2_2 + \left\|(x_s + v_s) - (x_* + v_*)\right\|^2_2\right) ds \\
      & \ \ \ \ \ \ \ \ +  \frac{2\gamma ud}{n\rho\alpha}e^{\alpha t}  + 2\sqrt {\frac{\gamma  u}{n\rho}}M_t + \frac{u^2L}{m\alpha} e^{\alpha t} \left\| \nabla f\left( x_* \right) \right\|_2^2 \\
      & \ \ \ \ \ \ \ \ + \frac{m}{4L}\int_0^t e^{\alpha s} \| {(x_s + v_s)-(x_* + v_*)}\|_2^2 ds  + \frac{u\log B}{n\alpha}e^{\alpha t} \\
      & \stackrel{(i)}{\leq} \frac{4\gamma udL}{n\rho m}e^{\alpha t}  + 2\sqrt {\frac{\gamma  u}{n\rho}}M_t + \frac{2u^2L^2}{m^2} e^{\alpha t} \left\| \nabla f\left( x_* \right) \right\|_2^2 + \frac{2uL\log B}{nm}e^{\alpha t} \\
      & \stackrel{}{=} \left(\frac{4\gamma udL}{n\rho m} + \frac{2u^2L^2}{m^2}\left\| \nabla f\left( x_* \right) \right\|_2^2 + \frac{2uL}{nm}\log B \right) e^{\alpha t}  + 2\sqrt {\frac{\gamma  u}{n\rho}}M_t,
  \end{split} 
\end{equation}

where $(i)$ holds by the selection of {$\alpha = \frac{m}{2L}$} and $\left\|x_t - x_*\right\|^2_2 \geq 0$. We now turn our attention to a detailed exploration of the moment of $V(x_t, v_t, t)$:

\begin{equation}\label{eq_moment_energy}
  \begin{split}
      \mathbb E \left[\left(\sup\limits_{0\leq t\leq T} V(x_t, v_t, t) \right)^p \right]^{\frac{1}{p}} & \stackrel{}{\leq}
      \mathbb E \left[\left(\sup\limits_{0\leq t\leq T} \frac{1}{2}e^{\alpha t} \left(\left\|x_t - x_*\right\|^2_2 + \left\|(x_t+v_t) - (x_* + v_*)\right\|^2_2 \right) \right)^p\right]^{\frac{1}{p}} \\
       & \leq \mathbb E \left[\left(\sup\limits_{0\leq t\leq T}\left(\frac{4\gamma udL}{n\rho m} + \frac{2u^2L^2}{m^2}\left\| \nabla f\left( x_* \right) \right\|_2^2 + \frac{2uL}{nm}\log B \right) e^{\alpha t}+2\sqrt{\frac{\gamma  u}{n\rho}}\left| {{M_t}} \right| \right)^p\right]^\frac{1}{p}\\
       & \stackrel{(i)}{\leq} \underbrace{\mathbb E \left[\sup\limits_{0\leq t\leq T}\left[\left(\frac{4\gamma udL}{n\rho m} + \frac{2uL}{nm}\log B \right)e^{\alpha t} \right]^p\right]^\frac{1}{p}}_{T1} + \underbrace{\mathbb E \left[\sup\limits_{0\leq t\leq T}\left[\frac{2u^2L^2}{m^2}\left\| \nabla f\left( x_* \right) \right\|_2^2 e^{\alpha t} \right]^p\right]^\frac{1}{p}}_{T2}\\
       & \ \ \ \ \ \ \ \ + \underbrace{\mathbb E \left[\left(\sup\limits_{0\leq t\leq T}2\sqrt {\frac{\gamma  u}{n\rho}}\left| {{M_t}} \right|\right)^p\right]^\frac{1}{p}}_{T3}, \\
  \end{split}
\end{equation}

where $(i)$ is from Minkowski inequality \citep{yosida1967functional}. We first bound $T1$:

\begin{equation}\label{eq_moment_const}
  \begin{split}
      T1 & \stackrel{}{=} \mathbb E \left[\sup\limits_{0\leq t\leq T}\left[\left(\frac{4\gamma udL}{n\rho m} + \frac{2uL}{nm}\log B \right)e^{\alpha t} \right]^p\right]^\frac{1}{p}\\
       & \leq \left(\frac{4\gamma udL}{n\rho m} + \frac{2uL}{nm}\log B \right)e^{\alpha T}.\\
  \end{split}
\end{equation}

We proceed to bound $T2$:
\begin{equation}\label{eq_moment_t2}
  \begin{split}
      T2 & = \mathbb E \left[\sup\limits_{0\leq t\leq T}\left[\left(\frac{2u^2L^2}{m^2} \|{\nabla f\left( x_* \right)}\|_2^2 \right)e^{\alpha t} \right]^p\right]^\frac{1}{p} \\
       & \leq \frac{2u^2L^2}{m^2} \mathbb E \left[\|{\nabla f\left( x_* \right)}\|_2^{2p} \right]^\frac{1}{p}e^{\alpha T}.\\
       & \stackrel{(i)}{\leq}  \frac{16u^2L^4dp}{m^2n\nu} e^{\alpha T}, \\
  \end{split}
\end{equation}
where $(i)$ proceeds because from Proposition \ref{prop_likelihood}, we know that $\|{\nabla f\left( x_* \right)}\|_2$ is a $L\sqrt{\frac{d}{n \nu }}$-sub-Gaussian vector, and its corresponding expectation follows
\begin{equation*}
  \mathbb E \left[\|{\nabla f\left( x_* \right)}\|_2^{2p}\right]^{\frac{1}{p}} \leq \left(2L\sqrt{\frac{2dp}{n\nu}}\right)^2.
\end{equation*}

We then bound $T3$:
\begin{equation}\label{eq_moment1}
  \begin{split}
      T3 & \stackrel{}{=} \mathbb E \left[\left(\sup\limits_{0\leq t\leq T}2\sqrt {\frac{\gamma  u}{n\rho}}\left| {{M_t}} \right|\right)^p\right]^\frac{1}{p} \\
      & \stackrel{(i)}{\leq} \mathbb E\left[\left(\frac{64\gamma upL}{n\rho m}e^{\alpha T}\right)^{\frac{p}{2}}\left(\sup\limits_{0\leq t\leq T}  { \frac{1}{2}e^{\alpha t}\left\|(x_t+v_t) - (x_* + v_*)\right\|_2^2}\right)^{\frac{p}{2}}\right]^{\frac{1}{p}}\\
       & \stackrel{(ii)}{\leq} \mathbb E\left[2^{p-2}\left(\frac{64\gamma upL}{n\rho m}e^{\alpha T}\right)^{{p}}+\frac{1}{2^{p}}\left(\sup\limits_{0\leq t\leq T}  { \frac{1}{2}e^{\alpha t}\left\|(x_t+v_t) - (x_* + v_*)\right\|_2^2}\right)^{{p}}\right]^{\frac{1}{p}}\\
       & \stackrel{(iii)}{\leq} 128\mathbb E\left[\left(\frac{\gamma upL}{n\rho m}e^{\alpha T}\right)^{{p}}\right]^{\frac{1}{p}} + \frac{1}{2}\mathbb E\left[\left(\sup\limits_{0\leq t\leq T}  { \frac{1}{2}e^{\alpha t}\left\|(x_t+v_t) - (x_* + v_*)\right\|_2^2}\right)^{{p}}\right]^{\frac{1}{p}}\\
       & \stackrel{}{=} \frac{128\gamma upL}{n\rho m}e^{\alpha T} + \frac{1}{2}\mathbb E\left[\left(\sup\limits_{0\leq t\leq T}  { \frac{1}{2}e^{\alpha t}\left\|(x_t+v_t) - (x_* + v_*)\right\|_2^2}\right)^{p}\right]^{\frac{1}{p}},\\
  \end{split}
\end{equation}
where $(i)$ is from \eqref{eq_mt}; inequality $(ii)$ follows Young's inequality; $(iii)$ proceeds by Minkowski inequality, the linearity of the expectation, and $2^{\frac{p-2}{p}}\leq 2$. By incorporating \eqref{eq_moment_const}-\eqref{eq_moment1} into \eqref{eq_moment_energy}, we can get

\begin{equation}\label{eq_moment_energy2}
  \begin{split}
      \mathbb E \left[\left(\sup\limits_{0\leq t\leq T} V\left( x_t, v_t, t\right) \right)^p \right]^{\frac{1}{p}} & \leq T1 + T2 + T3        \\
      & \stackrel{(i)}{\leq} \left(\frac{4\gamma udL}{n\rho m} + \frac{2uL}{nm}\log B + \frac{16u^2L^4dp}{m^2n\nu} +\frac{128\gamma upL}{n\rho m}\right)e^{\alpha T} \\
      & \ \ \ \ \ \ \ \  + \frac{1}{2}\mathbb E \left[\left( \sup\limits_{0\leq t\leq T}\frac{1}{2} {e^{\alpha t}} \left\| (x_t+v_t) - (x_* + v_*) \right\|_2^2 \right)^{p} \right]^{\frac{1}{p}} \\
      & \stackrel{(ii)}{\leq} \left(\frac{4\gamma udL}{n\rho m} + \frac{2uL}{nm}\log B + \frac{16u^2L^4dp}{m^2n\nu} +\frac{128\gamma upL}{n\rho m}\right)e^{\alpha T} \\
      & \ \ \ \ \ \ \ \  + \frac{1}{2}\mathbb E \left[\left( \sup\limits_{0\leq t\leq T}\frac{1}{2} {e^{\alpha t}} \left\|(x_t,x_t+v_t) - (x_*,x_* + v_*)\right\|_2^2 \right)^{p} \right]^{\frac{1}{p}} ,
  \end{split}
\end{equation}
where $(i)$ follows \eqref{eq_moment_const}-\eqref{eq_moment1}, $(ii)$ holds because $ {e^{\alpha t}} \left\|x_t - x_*\right\|_2^2 \geq 0$. From \eqref{eq_moment_energy} we know that the last term on the right-hand side of \eqref{eq_moment_energy2} holds the equality 
\begin{equation*}
	\frac{1}{2}\mathbb E \left[\left(\sup\limits_{0\leq t\leq T} \frac{1}{2}e^{\alpha t} \left(\left\|x_t - x_*\right\|^2_2 + \left\|(x_t+v_t) - (x_* + v_*)\right\|^2_2 \right) \right)^p\right]^{\frac{1}{p}}=\frac{1}{2}\mathbb E \left[\left(\sup\limits_{0\leq t\leq T} V(x_t, v_t, t) \right)^p \right]^{\frac{1}{p}},
\end{equation*}
and thus we derive the following bounds:

\begin{equation}\label{eq_moment_energy3}
  \begin{split}
      \mathbb E \left[\left(\sup\limits_{0\leq t\leq T} V\left( x_t, v_t, t\right) \right)^p \right]^{\frac{1}{p}} & \leq \left(\frac{8\gamma udL}{n\rho m} + \frac{4uL}{nm}\log B + \frac{32u^2L^4dp}{m^2n\nu} +\frac{256\gamma upL}{n\rho m}\right)e^{\alpha T}\\
      & \stackrel{(i)}{=} \left(\frac{16d}{n\rho m}+\frac{4}{nm}\log B+\frac{32L^2dp}{m^2n\nu}+\frac{512p}{n\rho m}\right)e^{\alpha T}\\
      & \stackrel{(ii)}{=} \frac{2}{mn}\left(D+\Omega p\right)e^{\alpha T},
  \end{split}
\end{equation}
where $(i)$ is by the choice of $\gamma=2$, $u=\frac{1}{L}$ from Lemma \ref{lemma_contraction}, and $(ii)$ is by defining $D=\frac{8d}{\rho }+{2}\log B$ and $\Omega=\frac{16L^2d}{m\nu}+\frac{256}{\rho}$. With the defined bound on the moments of $V\left(x_t, v_t, t\right)$'s supremum and considering the known expression for $V(x, v, t)$, we proceed to determine the $p$-th moments of $\left\|(x_T, x_t+v_t)-(x_*, x_* + v_*)\right\|$:
\begin{equation}
  \begin{split}
      \mathbb{E}\left[\left\|(x_T, x_T+ v_T)-(x_*, x_* + v_*)\right\|^p\right]^{\frac{1}{p}} & =\mathbb{E}\left[e^{-\frac{p \alpha T}{2}} V\left(x_T, v_T, T\right)^{\frac{p}{2}}\right]^{\frac{1}{p}} \\
      & \stackrel{(i)}{\leq} \mathbb{E}\left[e^{-\frac{p \alpha T}{2}}\left(\sup _{0 \leq t \leq T} V\left(x_t, v_t, t\right)\right)^{\frac{p}{2}}\right]^{\frac{1}{p}} \\
      & =e^{-\frac{\alpha T}{2}}\left(\mathbb{E}\left[\left(\sup _{0 \leq t \leq T} V\left(x_t, v_t, t\right)\right)^{\frac{p}{2}}\right]^{\frac{2}{p}}\right)^{\frac{1}{2}} \\
      & \stackrel{(ii)}{\leq} e^{-\frac{\alpha T}{2}}\left[\frac{2}{mn}\left(D+\Omega p\right)e^{\alpha T}\right]^{\frac{1}{2}} \\
      & =\sqrt{\frac{2}{mn}\left(D+\Omega p\right)},
  \end{split}
\end{equation}
where $(i)$ holds by taking supremum of $V(\cdot, \cdot, \cdot)$, and $(ii)$ is from \eqref{eq_moment_energy3}. Upon taking the limit for $T \rightarrow \infty $ and by referring to Fatou's Lemma \citep{brezis2011functional}, we can upper bound the moments of $\mathbb{E}\left[\left\|x - x_*\right\|^p\right]^{\frac{1}{p}}$ with a confidence of $1-\delta_1$:
\begin{equation}
  \begin{split}
      \mathbb{E}\left[\left\|x_T - x_*\right\|^p\right]^{\frac{1}{p}} & \leq \mathbb{E}\left[\left\|(x_T, x_T+ v_T)-(x_*, x_* + v_*)\right\|^p\right]^{\frac{1}{p}} \\ 
      & \leq \sqrt{\frac{2}{mn}\left(D+\Omega p\right)}.
  \end{split}
\end{equation}
Taking the limit as $T \rightarrow \infty$ and using Fatou's Lemma \citep{brezis2011functional}, we therefore have that the moments of $\mathbb{E}\left[\left\|x - x_*\right\|^p\right]^{\frac{1}{p}}$, with probability at least $1-\delta_1$:
\begin{equation}\label{eq_bound2}
  \begin{split}
      \mathbb{E}\left[\left\|x - x_*\right\|^p\right]^{\frac{1}{p}} & \leq \lim \inf _{T \rightarrow \infty} \mathbb{E}\left[\left\|x_T - x_*\right\|^p\right]^{\frac{1}{p}} \\
      & =\sqrt{\frac{2}{mn}\left(D+\Omega p\right)}.
  \end{split}
\end{equation}
By employing Markov's inequality \citep{boucheron2013concentration}, we further adapt our bound as follows:
\begin{equation}
  \begin{split}
      \mathbb{P}_{{x} \sim \mu_a^{(n)}}\left(\left\|x - x_*\right\|\geq\epsilon\right) & \leq \frac{\mathbb{E}\left[\left\|x - x_*\right\|^p\right]}{\epsilon^p} \\
      & \leq\left(\frac{1}{\epsilon}\sqrt{\frac{2}{mn}\left(D+\Omega p\right)}\right)^p .
  \end{split}
\end{equation}

By setting $p=2 \log 1 / \delta_1$ and defining $\epsilon=\sqrt{\frac{2e}{mn}\left(D+\Omega p\right)}$, we derive the required bound:

\begin{equation}\label{eq_posterior_derivation2}
  \mathbb{P}_{{x} \sim \mu_a^{(n)}\left[\rho_a\right]}\left(\left\|x - x_*\right\|_2\geq \sqrt{\frac{2e}{mn}\left(D+2\Omega \log 1 / \delta_1\right)}\right)\leq\delta_1,
\end{equation}

{for $\delta_1 \leq e^{-0.5}$.}

\end{proof}

\subsection{Regret Analysis of Thompson Sampling}\label{section_regret_general_ts}

After establishing the contraction results for the posterior distributions, we next delve into the analysis of regret, building upon the foundations laid by the concentration results. Specifically, under the assumptions we have made, we demonstrate that using Thompson sampling with samples drawn from the posteriors can achieve optimal regret guarantees within a finite time frame. To elucidate this, we employ a typical approach used in regret proofs associated with Thompson sampling \citep{lattimore2020bandit}. Our primary goal is to quantify the number of times, denoted as $\mathcal L_a(N)$, the sub-optimal arm is selected up to time $N$. 

For clarity and simplicity in our discourse, we assume, without the loss of generality, that the first arm is optimal. The filtration corresponding to an execution of the algorithm is denoted as $\mathcal F_{n} = \sigma\left(A_1, {\mathcal R}_1, A_2, {\mathcal R}_2, A_3, {\mathcal R}_3, \cdots, A_{n}, {\mathcal R}_{n}\right)$, where it is a $\sigma$-algebra generated by Algorithm \ref{algorithm_thompson} after playing $n$ times. We also denote an event $\mathcal E_a(n)=\left\{{\mathcal R}_{a, n}\geq \bar{{\mathcal R}}_1-\epsilon\right\}$ to indicate the estimated reward of arm $a$ at round $n$ exceeds the expected reward of optimal arm by at least a positive constant $\epsilon$. We also define its probability as $\mathtt P\left(\mathcal E_a(n)|\mathcal F_{n-1}\right) = \mathcal G_{a, n}$. Next, we analyze the expected number of times that suboptimal arms are played by breaking it down into two parts:

\begin{equation}\label{equation_sub_play}
\begin{split}
\mathbb{E}\left[\mathcal L_a(N)\right]&=\mathbb{E}\left[\sum_{n=1}^N \mathbb{I}\left(A_n=a\right)\right]\\
&=\mathbb{E}\left[\sum_{n=1}^N \mathbb{I}\left(A_n=a, \mathcal E_a(n)\right)\right]+\mathbb{E}\left[\sum_{n=1}^N \mathbb{I}\left(A_n=a, \mathcal E_a^c(n)\right)\right].
\end{split}
\end{equation}

We now turn our attention to the proof of the upper bounds for the terms appearing in \eqref{equation_sub_play}.

\begin{lemma}\label{lemma_suboptimal1}

For a sub-optimal arm $a\in \mathcal A$, we have an upper bound as outlined below:

\begin{equation}\label{eq_bound_optarm1}
\mathbb{E}\left[\sum_{n=1}^N \mathbb{I}\left(A_n=a, \mathcal E_a^c(n)\right)\right]\leq \mathbb{E}\left[\sum_{s=0}^{N-1} \left(\frac{1}{\mathcal G_{1, s}} - 1\right)\right],
\end{equation}
where $\mathcal G_{1, n}\coloneqq \mathtt P\left({{\mathcal R}}_{1, n}>\bar{{\mathcal R}}_1-\epsilon|\mathcal F_{n-1}\right)$, for some $\epsilon >0$.
\end{lemma}

\begin{proof}
It should be noted that $A_n=\text{argmax}_{a\in \mathcal A} \left\langle \alpha_a, x_{a, n}\right\rangle$ is the arm that has the highest sample reward at round $n$. Additionally, we set arm $A'_n=\text{argmax}_{a\in \mathcal A, a\neq 1} \left\langle \alpha_a, x_{a, n}\right\rangle$ to be the one that attains the maximum sample reward except for the optimal arm $A_n$. It is worth noting that the following inequality holds:

\begin{equation}\label{eq_bound_optarm2}
\begin{split}
\mathbb{P}\left(A_n=a, \mathcal E_a^c(n) | \mathcal{F}_{n-1}\right) & \stackrel{(i)}{\leq} \mathbb{P}\left(A'_n=a, \mathcal E_a^c(n) , {\mathcal R}_{1, n} < \bar{{\mathcal R}}_1-\epsilon | \mathcal{F}_{n-1}\right) \\
& \stackrel{(ii)}{=} \mathbb{P}\left(A'_n=a, \mathcal E_a^c(n) | \mathcal{F}_{n-1}\right)\mathbb{P}\left({\mathcal R}_{1, n} < \bar{{\mathcal R}}_1-\epsilon | \mathcal{F}_{n-1}\right) \\
& =\mathbb{P}\left(A_n^{\prime}=a, \mathcal E_a^c(n) | \mathcal{F}_{n-1}\right)\left(1-\mathbb{P}\left(\mathcal E_1(n) | \mathcal{F}_{n-1}\right)\right)\\
& \stackrel{(iii)}{\leq} \frac{\mathbb{P}\left(A_n=1, \mathcal E_a^c(n) | \mathcal{F}_{n-1}\right)}{\mathbb{P}\left(\mathcal E_1(n) | \mathcal{F}_{n-1}\right)}\left(1-\mathbb{P}\left(\mathcal E_1(n) | \mathcal{F}_{n-1}\right)\right) \\
& \stackrel{(iv)}{\leq} \mathbb{P}\left(A_n=1 | \mathcal{F}_{n-1}\right)\left(\frac{1}{\mathbb{P}\left(\mathcal E_1(n) | \mathcal{F}_{n-1}\right)}-1\right)\\
& = \mathbb{P}\left(A_n=1 | \mathcal{F}_{n-1}\right) \left(\frac{1}{\mathcal G_{1, n}}-1\right),
\end{split}
\end{equation}
where $(i)$ holds because event $\left\{A_n=a, \mathcal E_a^c(n)\right\} \subseteq \left\{A'_n=a, \mathcal E_a^c(n) , {\mathcal R}_{1, n}<\bar{{\mathcal R}}_1-\epsilon\right\}$ given $\mathcal{F}_{n-1}$, $(ii)$ is valid because events $\left\{A'_n=a, \mathcal E_a^c(n)\right\}$ and $\left\{{\mathcal R}_{1, n}<\bar{{\mathcal R}}_1-\epsilon\right\}$ are independent, $(iii)$ is from $\mathbb{P}\left(A'_n=a, \mathcal E_a^c(n) | \mathcal{F}_{n-1}\right) \mathbb{P}\left(\mathcal E_1(n) | \mathcal{F}_{n-1}\right) \leq \mathbb{P}\left(A_n=1, \mathcal E_a^c(n) | \mathcal{F}_{n-1}\right)$ since $\left\{A'_n=a, \mathcal E_a^c(n), \mathcal E_1^c(n)\right\}\subseteq \left\{A_n=1, \mathcal E_a^c(n), \mathcal E_1^c(n)\right\}$. The fact that $\left\{A_n=1, \mathcal E_a^c(n)\right\}\subseteq \left\{A_n=1\right\}$ leads to the inequality $(iv)$. Now by summing the result in \eqref{eq_bound_optarm2} from $n=1$ to $N$ we have:

\begin{equation}\label{eq_bound_optarm3}
\begin{split}
\mathbb{E}\left[\sum_{n=1}^N \mathbb{I}\left(A_n=a, \mathcal E_a^c(n)\right)\right] & \stackrel{(i)}{\leq} \mathbb{E}\left[\sum_{n=1}^N \mathbb{P}\left(A_n=1 | \mathcal{F}_{n-1}\right)\left(\frac{1}{\mathcal G_{1, \mathcal L_1(n-1)}}-1\right)\right] \\
& \stackrel{(ii)}{=} \mathbb{E}\left[\sum_{n=1}^N \mathbb{I}\left(A_n=1\right)\left(\frac{1}{\mathcal G_{1, \mathcal L_1(n-1)}}-1\right)\right] \\
& \stackrel{(iii)}{\leq} \mathbb{E}\left[\sum_{s=0}^{N-1} \left(\frac{1}{\mathcal G_{1, s}}-1\right)\right],
\end{split}
\end{equation}
where $(i)$ holds because of the law of total expectation and \eqref{eq_bound_optarm2}, $(ii)$ is also derived from the law of total expectation, and $(iii)$ is valid due to the constraint that, for a given $s\in \llbracket 1, N \rrbracket$, both $\mathbb{I}\left(A_n=1\right)$ and $\mathbb{I}\left(\mathcal L_1(n)=s\right)$ can simultaneously hold true at most once.

\end{proof} 

\begin{lemma}\label{lemma_suboptimal2}

Considering a sub-optimal arm $a\in \mathcal A$, we derive an upper bound expressed as

\begin{equation}\label{eq_bound_optarm4}
\mathbb{E}\left[\sum_{n=1}^N \mathbb{I}\left(A_n=a, \mathcal E_a(n)\right)\right]\leq 1+ \mathbb{E}\left[\sum_{s=0}^{N-1} \mathbb{I}\left(\mathcal G_{a, s}>\frac{1}{N}\right)\right],
\end{equation}
where $\mathcal G_{a, n}\coloneqq \mathtt P\left({{\mathcal R}}_{a, n}>\bar{{\mathcal R}}_1-\epsilon|\mathcal F_{n-1}\right)$, for a certain $\epsilon >0$.
\end{lemma}

\begin{proof}
The derivation of the upper bound for Lemma \ref{lemma_suboptimal2} is identical to that presented in \citet{agrawal2012analysis,lattimore2020bandit}, and we revisit this proof herein for comprehensive exposition. With a defined set $\mathcal N=\left\{n: \mathcal G_{a, \mathcal L_{a}(n)}>\frac{1}{N}\right\}$, we decompose the expression in \eqref{eq_bound_optarm4} into two constituent terms:
\begin{equation}\label{eq_suboptimal_lemma2}
\begin{split}
  \mathbb{E}\left[\sum_{n=1}^N \mathbb{I}\left(A_n=a, \mathcal E_a(n)\right)\right]& \leq \underbrace{\mathbb{E}\left[\sum_{n\in\mathcal N} \mathbb{I}\left(A_n=a\right)\right]}_{T1} + \underbrace{\mathbb{E}\left[\sum_{n\notin\mathcal N} \mathbb{I}\left( \mathcal E_a(n)\right)\right]}_{T2}\\
\end{split}
\end{equation}

For term $T1$ , we have:
\begin{equation*}
	\begin{split}
		T1 & = \sum_{n=1}^N \mathbb{I}\left(A_n=a\right) \\
		& \stackrel{(i)}{\leq} \sum_{n=1}^N \sum_{s=1}^N \mathbb{I}\left(\mathcal L_a(n)=s, \mathcal L_a(n-1)=s-1, \mathcal G_{a, \mathcal L_{a}(n-1)}>\frac{1}{N}\right) \\
		& =\sum_{s=1}^N \mathbb{I}\left(\mathcal G_{a, s-1}>\frac{1}{N}\right) \sum_{n=1}^N \mathbb{I}\left(\mathcal L_a(n)=s, \mathcal L_a(n-1)=s-1\right) \\
		& \stackrel{(ii)}{=} \sum_{s=1}^N \mathbb{I}\left(\mathcal G_{a, s-1}>\frac{1}{N}\right),
	\end{split}
\end{equation*}

where $(i)$ uses the fact that when $A_n=a, \mathcal L_a(n)=s$ and $\mathcal L_a(n-1)=s-1$ for some $s \in \llbracket 1, N \rrbracket$ and $n \in \llbracket 1, N \rrbracket$ implies that $G_{a, s-1}>1 / N$, then $(ii)$ holds because for any $s \in \llbracket 1, N \rrbracket$, there is at most one time point $n \in \llbracket 1, N \rrbracket$ such that both $\mathcal L_a(n)=s$ and $\mathcal L_a(n-1)=s-1$ hold true. For the next inequality, note that
\begin{equation*}
	\begin{split}
		\mathbb{E}\left[\sum_{n \notin \mathcal{N}} \mathbb{I}\left(\mathcal E_a(n)\right)\right] & =
			\sum_{n \notin \mathcal{N}} \mathbb{E}\left[\mathbb{E}\left[\mathbb{I}\left(\mathcal E_a(n), \mathcal G_{a,\mathcal L_a(n-1)} \leq \frac{1}{N}\right) | \mathcal{F}_{n-1}\right]\right] \\
		& \stackrel{(i)}{=} \sum_{n \notin \mathcal{N}} \mathbb{E}\left[\mathcal G_{a,\mathcal L_a(n-1)} \mathbb{I}\left(\mathcal G_{a,\mathcal L_a(n-1)}  \leq \frac{1}{N}\right) \right] \\
		& \leq \sum_{n \notin \mathcal{N}} \mathbb{E}\left[\frac{1}{N}\mathbb{I}\left(\mathcal G_{a,\mathcal L_a(n-1)}  \leq \frac{1}{N}\right)\right] \\
		& \leq \mathbb{E}\left[\sum_{n \notin \mathcal{N}} \frac{1}{N}\right] \\
		& \leq 1,
	\end{split}
\end{equation*}
where $(i)$ stems from the fact that $\mathbb{I}\left(\mathcal G_{a,\mathcal L_a(n-1)} \leq \frac{1}{N}\right)$ is $\mathcal{F}_{n-1}$-measurable and the following equality holds:
\begin{equation*}
	\mathbb{E}\left[\mathbb{I}\left(\mathcal E_a(n)\right) | \mathcal{F}_{n-1}\right]=1-\mathbb{P}\left({\mathcal R}_{a, n} < \bar {{\mathcal R}}_1-\epsilon | \mathcal{F}_{n-1}\right)=\mathcal G_{a, \mathcal L_a(n-1)}.
\end{equation*}

By employing the upper bounds established for terms $T1$ and $T2$ as delineated in \eqref{eq_suboptimal_lemma2}, we achieve the targeted result in \eqref{eq_bound_optarm4}.
\end{proof}

\subsection{Regret Analysis of Exact Thompson Sampling}\label{section_regret_exact_ts}

Drawing from insights in Lemmas \ref{lemma_suboptimal1} and \ref{lemma_suboptimal2}, we delve deeper to derive the upper bound in the context of exact Thompson sampling. To derive the bound, we introduce two fundamental lemmas, which are crucial for the definitive proof of total expected regrets. Notably, our first lemma prescribes a probability threshold for an arm being thoroughly explored, as a function of prior quality.

\begin{lemma}\label{lemma_exact_smallT}
Suppose the likelihood satisfies Assumption \ref{assumption_likelihood}, and the prior satisfies Assumption \ref{assumption_prior}. Then for every $n=1, \cdots, N$ and with the selection of $\rho_1=\frac{\nu_1 m_1^2}{8 d L_1^3}$, we will have:
\begin{equation*}
	\mathbb{E}\left[\frac{1}{\mathcal G_{1, n}}\right] \leq 16 \sqrt{\kappa_1 B_1}.
\end{equation*}

\end{lemma}

\begin{proof}
For notational simplicity within this proof, we have chosen to exclude the arm-specific subscript, except when explicitly required. Our primary focus is directed towards $\left\|(x_*, v_*)-(x_u, v_u)\right\|_2^2$. In this context, $(x_u, v_u)$ serves as the mode of the posterior for the first arm once it has received $n$ samples in alignment with the following condition:
\begin{equation*}
\left(\gamma-1\right)v_u+u\nabla f\left(x_u\right)+\frac{u}{n}\nabla\log\pi(x_u)=0
\end{equation*}

Using the provided definition and taking $\bar{x}=x_u-x_*$ and $\bar{v}=v_u-v_*$, we have that:
\begin{equation}\label{eq_mode_fixed1}
\begin{split}
\left(\gamma-1\right)\bar v + u\left(\nabla f(x_u)-\nabla f(x_*)\right) & = -\left(\gamma-1\right) v_* - u\nabla f(x_*)-\frac{u}{n}\nabla\log\pi(x_u) \\
-\left\langle \bar x, \bar v \right\rangle - \left\langle\bar x + \bar v, \left(\gamma-1\right)\bar v + u\left(\nabla f(x_u)-\nabla f(x_*)\right)\right\rangle & =-\left\langle \bar x, \bar v \right\rangle - \left\langle \bar x + \bar v, \left(\gamma-1\right) v_* + u\nabla f(x_*)+\frac{u}{n}\nabla\log\pi(x_u)\right\rangle. \\
\end{split}
\end{equation}

From the above equation, we try to upper bound the distance between fixed points $x_*$ and the posterior mode $x_u$:
\begin{equation*}
\begin{split}
\left\|x_*-x_u\right\|_2^2 & \stackrel{}{\leq} \|\bar x\|_2^2 + \|\bar x+\bar v\|_2^2 \\
& \stackrel{(i)}{\leq} -\frac{2L}{m}\left(\left\langle \bar x, \bar v \right\rangle + \left\langle\bar x + \bar v, (\gamma-1)\bar v + u\left(\nabla f(x_u)-\nabla f(x_*)\right)\right\rangle\right)\\
& \stackrel{(ii)}{=} -\frac{2L}{m}\left(\left\langle \bar x, \bar v \right\rangle + \left\langle \bar x + \bar v, \left(\gamma-1\right) v_* + u\nabla f(x_*)+\frac{u}{n}\nabla\log\pi(x_u)\right\rangle\right)\\
  & \stackrel{(iii)}{\leq} -\frac{2L}{m}\left\langle \bar x, \bar v \right\rangle + \|\bar x + \bar v\|_2^2 + \frac{u^2L^2}{m^2}\|\nabla f(x_*)\|_2^2 + \frac{2uL}{mn} \log B\\
& \stackrel{(iv)}{\leq} \frac{2L}{m}\varepsilon + \|\bar x + \bar v\|_2^2 + \frac{u^2L^2}{m^2}\|\nabla f(x_*)\|_2^2 + \frac{2uL}{mn} \log B,\\
\end{split}
\end{equation*}
where we derive $(i)$ from Lemma \ref{lemma_contraction}. we refer to \eqref{eq_mode_fixed1} to derive $(ii)$, and $(iii)$ is from the application of Young's inequality and the insights of \eqref{eq_boundt4}. In $(iv)$, we adopt the constraint that $0\leq \left|\left\langle \bar x, \bar v \right\rangle\right|\leq \varepsilon$. Building upon the definition $\kappa=L / m$, we then proceed to derive the subsequent inequality:
\begin{equation*}
  \begin{split}
      \left\|x_*-x_u\right\|_2^2 & \leq \frac{2L}{m}\varepsilon + \frac{2uL}{mn} \log B + \frac{u^2L^2}{m^2}\|\nabla f(x_*)\|_2^2\\
      & \stackrel{}{=} 2\kappa\varepsilon + \frac{2u\kappa}{n} \log B + u^2\kappa^2\|\nabla f(x_*)\|_2^2.\\
  \end{split}
\end{equation*}

Noting that $\left|\left\langle \alpha, \bar x\right\rangle \right| \leq \sqrt{\omega^2\|x_* - x_u\|_2^2}$ ($\omega$ is the norm of $\alpha$) we find that:

\begin{equation*}
	\begin{split}
{\mathcal G}_{1, s} & ={\mathtt P}\left(\left\langle \alpha_1, x - x_u\right\rangle \geq \left\langle \alpha_1, x_* - x_u\right\rangle-\epsilon\right) \\
& \geq {\mathtt P}\left(\left\langle \alpha_1, x - x_u\right\rangle \geq \underbrace{\sqrt{2\omega_1^2\kappa_1\varepsilon + \frac{2\omega_1^2u\kappa}{n} \log B_1 + \omega_1^2u^2\kappa_1^2\|\nabla f(x_*)\|_2^2}}_{=t}\right).
\end{split}
\end{equation*}
Based on the information that the posterior over $(x, v)$ aligns proportionally to the marginal distribution described by $\exp\left(-\rho_1\left(nf(x)+\pi(x)\right)\right)$, and given its attributes as being $\rho_1 L_1(n+1)$ Lipschitz smooth and $\rho_1 m_1n$ strongly log-concave with mode $x_u$, we can infer from Theorem 3.8 in \citet{saumard2014log} that the marginal density of $\langle \alpha_1, x\rangle$ exhibits the same smoothness and log-concavity. Then we can derive the following relationship:
\begin{equation*}
	{\mathtt P}\left(\left\langle \alpha_1, x - x_u\right\rangle \geq t\right) \geq \sqrt{\frac{nm_1}{(n+1)L_1}} {\mathtt P}(\mathcal X \geq t)
\end{equation*}
where $\mathcal X \sim \mathcal{N}\left(0, \frac{\omega_1^2}{\rho_1(n+1)L_1}\right)$. Upon applying a lower bound derived from the cumulative density function of a Gaussian random variable, we ascertain that, for $\sigma^2=\frac{\omega_1^2}{\rho_1(n+1)L_1}$:
\begin{equation*}
	{\mathcal G}_{1, s} \geq \sqrt{\frac{n m_1}{2 \pi(n+1)L_1}}  
	\begin{cases}
	\frac{\sigma t}{t^2+\sigma^2} e^{-\frac{t^2}{2 \sigma^2}} & : t> \frac{\omega_1}{\sqrt{\rho_1(n+1)L_1}} \\ 
	0.34 & : t \leq \frac{\omega_1}{\sqrt{\rho_1(n+1)L_1}}
	\end{cases}
\end{equation*}

Thus we have that:

\begin{equation*}
	\begin{split}
		\frac{1}{{\mathcal G}_{1, s}} & \leq \sqrt{\frac{2 \pi(n+1)L_1}{n m_1}} 
		\begin{cases}
		\frac{t^2+\sigma^2}{\sigma t} e^{\frac{t^2}{2 \sigma^2}} & : t>\frac{\omega_1}{\sqrt{\rho_1(n+1)L_1}} \\
		\frac{1}{0.34} & : t \leq \frac{\omega_1}{\sqrt{\rho_1(n+1)L_1}}
		\end{cases} \\
		& \leq \sqrt{\frac{2 \pi(n+1)L_1}{n m_1}} 
		\begin{cases}\left(\frac{t}{\sigma}+1\right) e^{\frac{t^2}{2 \sigma^2}} & : t>\frac{\omega_1}{\sqrt{\rho_1(n+1)L_1}} \\
		3 & : t \leq \frac{\omega_1}{\sqrt{\rho_1(n+1)L_1}}
		\end{cases}
	\end{split}
\end{equation*}
After determining the expectation on both sides considering the samples $\mathcal R_1, \ldots, \mathcal R_n$, we deduce that:

\begin{equation*}
  \begin{split}
      \mathbb{E}\left[\frac{1}{{\mathcal G}_{1, s}}\right] & \leq 3 \sqrt{\frac{2 \pi(n+1)L_1}{n m_1}} + \sqrt{\frac{2 \pi(n+1)L_1}{n m_1}} \mathbb E\left[ \left(\frac{t}{\sigma}+1\right) e^{\frac{t^2}{2 \sigma^2}} \right]\\
      & \stackrel{(i)}{\leq} 6 \sqrt{2\pi \kappa_1}+ 2 \sqrt{2\pi \kappa_1} \mathbb E\left[ \left(\frac{t}{\sigma}+1\right) e^{\frac{t^2}{2 \sigma^2}} \right],\\
  \end{split}
\end{equation*}

where $(i)$ proceeds by $0<\frac{(n+1)L_1}{nm_1}\leq \frac{2L_1}{m_1}=2\kappa_1$. Our attention now turns to firstly bounding the term $\mathbb E\left[ \left(\frac{t}{\sigma}+1\right) e^{\frac{t^2}{2 \sigma^2}} \right]$, and then we proceed to derive an upper bound for $\mathbb{E}\left[\frac{1}{{\mathcal G}_{1, s}}\right]$ (we drop arm-specific parameters here to simplify the notation).
\begin{equation}\label{eq_bound_p1s}
  \begin{split}
      \mathbb E\left[ \left(\frac{t}{\sigma}+1\right) e^{\frac{t^2}{2 \sigma^2}} \right] & = \mathbb E \left[ \left(\sqrt{{\rho(n+1)L}\left(2\kappa\varepsilon + \frac{2u\kappa}{n} \log B + u^2\kappa^2\|\nabla f(x_*)\|_2^2\right)}+1\right)e^{\frac{\rho(n+1)L}{2}\left(2\kappa\varepsilon + \frac{2u\kappa}{n} \log B + u^2\kappa^2\|\nabla f(x_*)\|_2^2\right)}\right]\\
      & \stackrel{(i)}{=} \mathbb E \left[ \left(\sqrt{\rho(n+1)L\left(\frac{2L}{m}\varepsilon + \frac{2\log B}{mn} + \frac{\left\|\nabla f(x_*)\right\|_2^2}{m^2}\right)}+1\right)e^{\frac{\rho(n+1)L}{2}\left(\frac{2L}{m}\varepsilon + \frac{2}{mn}\log B + \frac{1}{m^2}\|\nabla f(x_*)\|_2^2\right)}\right]\\
      & \stackrel{(ii)}{\leq} e^{\frac{\rho(n+1)L}{mn}\left(Ln\varepsilon+\log B\right)}\left(\sqrt{\frac{2\rho(n+1)L\left(Ln\varepsilon+\log B\right)}{mn}}+1\right) \mathbb E \left[e^{\frac{\rho(n+1)L}{2m^2}\|\nabla f(x_*)\|_2^2}\right]\\
      & \ \ \ \ \ \ \ \ + e^{\frac{\rho(n+1)L}{mn}\left(Ln\varepsilon+\log B\right)}\sqrt{\frac{\rho (n+1)L}{m^2}}\sqrt{\mathbb E\left[\left\|\nabla f(x_*)\right\|_2^2\right]} \mathbb E \left[e^{\frac{\rho(n+1)L}{2m^2}\|\nabla f(x_*)\|_2^2}\right]\\
      & \stackrel{(iii)}{\leq} e^{\frac{\rho(n+1)L}{mn}\left(Ln\varepsilon+\log B\right) + \frac{1}{2}}\left(\sqrt{\frac{2\rho(n+1)L}{mn}\left(Ln\varepsilon+\log B\right)} + 1 + \sqrt{\frac{2\rho L^3 d}{m^2\nu}\frac{n+1}{n}}\right), \\
      & \stackrel{(iv)}{\leq} \sqrt e \left(B\right)^{1/4}\left(\sqrt{\log B} + 1 + \frac{\sqrt 2}{2}\right), \\
  \end{split}
\end{equation}
where from the selections of \(u = \frac{1}{L_1}\) and \(\kappa = \frac{L_1}{m_1}\), we obtain $(i)$. $(ii)$ is further extracted from an intrinsic inequality:
\begin{equation*}
\sqrt{\frac{2L_1}{m_1}\varepsilon + \frac{2\log B_1}{m_1n} + \frac{\left\|\nabla f(x_*)\right\|_2^2}{m_1^2}} \leq \sqrt{\frac{2L_1}{m_1}\varepsilon + \frac{2\log B_1}{m_1n}} + \frac{\left\|\nabla f(x_*)\right\|_2}{m_1}.
\end{equation*}
$(iii)$ holds because $\left\|\nabla f(x_*)\right\|_2$ is $L_1\sqrt{\frac{d}{n \nu_1}}$ sub-Gaussian and its squared $\left\|\nabla f(x_*)\right\|_2^2$ form as sub-exponential, then by selecting $\lambda < \frac{n\nu_1}{4dL_1^2}$ and $\rho_1 = \frac{m_1^2\nu_1}{8L_1^3d}$, we can bound the following random variables as:
\begin{equation*}
  \begin{split}
      \mathbb{E}\left[ e^{\lambda \left\|\nabla f(x_*)\right\|_2^2}\right] & \leq e, \\
      \mathbb{E}\left[\left\|\nabla f(x_*)\right\|_2^2\right ] & \leq  2 \frac{L_1^2d}{\nu_1 n} .
  \end{split}
\end{equation*}
We finally obtain $(iv)$ with the choice of $\varepsilon\leq \frac{\log B_1}{L_1n}$, $\frac{m_1}{L_1}\leq 1$, $\frac{1}{d}\leq 1$, $\frac{n+1}{n}\leq 2$, and $\frac{\nu_1}{L_1}\leq 1$ without loss of generality. Then we turn our attention to bound the expected value of $\frac{1}{{\mathcal G}_{1, s}}$:
\begin{equation*}
  \begin{split}
      \mathbb{E}\left[\frac{1}{{\mathcal G}_{1, s}}\right] & \stackrel{}{\leq} 3 \sqrt{2\pi \kappa_1}+ \sqrt{2\pi \kappa_1} \mathbb E\left[ \left(\frac{t}{\sigma}+1\right) e^{\frac{t^2}{2 \sigma^2}} \right]\\
      & \stackrel{(i)}{\leq} 3 \sqrt{2\pi \kappa_1} + \sqrt{2\pi \kappa_1 e} \left(B_1\right)^{1/4}\left(\sqrt{\log B_1} + 1 + \frac{\sqrt 2}{2}\right) \\
      & \stackrel{}{\leq}  \sqrt{2\pi \kappa_1 e} \left(B_1\right)^{1/4}\left(\sqrt{\log B_1} + 1+\frac{\sqrt 2}{2}+\frac{3}{\sqrt e}\right) \\
      & \stackrel{(ii)}{<} 16\sqrt{ \kappa_1 B_1},
  \end{split}
\end{equation*}
where $(i)$ is from \eqref{eq_bound_p1s}, and $(ii)$ uses the fact that 
\begin{equation*}
	\sqrt{2\pi e} x^{1/4}\left(\sqrt{\log x}+1+\frac{\sqrt 2}{2}+\frac{3}{\sqrt e}\right)< 16\sqrt{x}
\end{equation*}
when $x\geq 1$. We then derive the final upper bound.

\end{proof}

We introduce the next technical lemma below, which provides problem-specific upper-upper bounds for the terms described in Lemmas \ref{lemma_suboptimal1} and \ref{lemma_suboptimal2}.

\begin{lemma}
\label{lemma_exact_ts2}
	Suppose the likelihood, rewards, and priors satisfy Assumptions \ref{assumption_likelihood}-\ref{assumption_prior}, then for $a\in\mathcal A$ and $\rho_a=\frac{\nu_a m_a^2}{8dL_a^3}$:
	\begin{equation}\label{eq_exact_regret1}
	\sum_{s=0}^{N-1}\mathbb{E}\left[ \frac{1}{{\mathcal G}_{1,s}}-1\right] \leq 16\sqrt{\frac{L_1}{m_1}B_1} \left\lceil \frac{8e\omega_1^2}{m_1\Delta_a^2}(D_1+2\Omega_1 ) \right\rceil +1,
	\end{equation}
	\begin{equation}\label{eq_exact_regret2}
	\sum_{s=0}^{N-1} \mathbb{E}\left[ \mathbb{I}\left({\mathcal G}_{a,s}>\frac{1}{N}\right)\right]  \leq \frac{8 e \omega_a^2}{m_a\Delta_a^2}\left(D_a+2\Omega_a\log N\right) ,
	\end{equation}
	where from the preceding derivation, it follows that for all arms $a \in \mathcal{A}$,  $D_a$ and $\Omega_a$ are specified as $D_a={2}\log B_a + \frac{8d}{\rho_a}$ and $\Omega_a=\frac{16L_a^2d}{m_a\nu_a}+\frac{256}{\rho_a}$ respectively.
\end{lemma}

\begin{proof}
We begin by showing that for $a\in \mathcal A$, the lower bound for ${\mathcal G}_{a,s}$:
\begin{equation}\label{eq_pas_lower}
  \begin{split}
          {\mathcal G}_{a,s} 	& =\mathbb{P}({{\mathcal R}}_{a,s}>\bar{{\mathcal R}}_a-\epsilon |\mathcal{F}_{n-1})\\
			            & = 1-\mathbb{P}({{\mathcal R}}_{a,s}-\bar{{\mathcal R}}_a\leq-\epsilon |\mathcal{F}_{n-1})\\
			            & \geq 1-\mathbb{P}(|{{\mathcal R}}_{a,s}-\bar{{\mathcal R}}_a|\geq\epsilon |\mathcal{F}_{n-1})\\
			            & \stackrel{(i)}{\geq} 1-\mathbb{P}_{x \sim \mu^{(s)}_{a}}\left(\|x - x_*\| \geq \frac{\epsilon}{\omega_a} \right),\\
  \end{split}
\end{equation}
where $(i)$ lies in the $\omega_a$-Lipschitz continuity of both ${{\mathcal R}}_{a,s}$ and $\bar{{\mathcal R}}_a$ w.r.t. $x \sim \mu^{(s)}_a$ and $x_*$ respectively. We then use the above result to upper bound $\sum_s \mathbb{E}\left[ \frac{1}{{\mathcal G}_{1,s}}-1\right]$ and $\sum_s \mathbb{E}\left[ \mathbb{I}\left({\mathcal G}_{a,s}>\frac{1}{N}\right)\right]$. For \eqref{eq_exact_regret1}, we have:
\begin{equation}\label{eq_g1s_inverse1}
\begin{split}
    \sum_{s=0}^{N-1}\mathbb{E}\left[ \frac{1}{{\mathcal G}_{1,s}}-1 \right]
    & = \sum_{s=0}^{\mathbb \ell-1}\mathbb{E}\left[ \frac{1}{{\mathcal G}_{1,s}}-1 \right] + \sum_{s=\ell}^{N-1}\mathbb{E}\left[ \frac{1}{{\mathcal G}_{1,s}}-1 \right]\\
    & \stackrel{(i)}{\leq} \sum_{s=0}^{\mathbb \ell-1}\mathbb{E}\left[ \frac{1}{{\mathcal G}_{1,s}}-1 \right] + \sum_{s=\mathbb \ell}^{N-1} \left[ \frac{1}{1-\mathbb{P}_{x \sim \mu^{(s)}_{1}}\left(\|x - x_*\| \geq {\epsilon}/{\omega_1} \right)} -1 \right]\\
    &\leq \sum_{s=0}^{\mathbb \ell-1}\mathbb{E}\left[ \frac{1}{{\mathcal G}_{1,s}}-1 \right] + \int_{s=\mathbb \ell}^{\infty} \left[\frac{1}{1-\mathbb{P}_{x \sim \mu^{(s)}_{1}}\left(\|x - x_*\| \geq {\epsilon}/{\omega_1} \right)}-1\right] ds \\
    & \stackrel{(ii)}{\leq} \sum_{s=0}^{\mathbb \ell-1}\mathbb{E}\left[ \frac{1}{{\mathcal G}_{1,s}}-1 \right] + \int_{s=1}^{\infty} \left[\frac{1}{1-\frac{1}{2}\exp\left(-\frac{m_1\epsilon^2}{4e\Omega_1 \omega_1^2}s\right)}-1\right] ds \\
    & \stackrel{(iii)}{\leq} \sum_{s=0}^{\mathbb \ell-1}\mathbb{E}\left[ \frac{1}{{\mathcal G}_{1,s}}-1 \right] + \frac{16e \Omega_1\omega_1^2}{m_1\Delta_a^2}\log{2} +1 \\
    & \stackrel{(iv)}{\leq} 16\sqrt{\frac{L_1}{m_1}B_1} \left\lceil \frac{8e\omega_1^2}{m_1\Delta_a^2}(D_1+2\Omega_1) \right\rceil +1\\
\end{split}
\end{equation}
where $(i)$ is by inducing \eqref{eq_pas_lower}. $(ii)$ is because from Theorem \ref{theorem_posterior_concentration2} we have the following posterior concentration inequality:
\begin{equation}\label{eq_concentration2}
\begin{split}
 \mathbb{P}_{x \sim \mu^{(s)}_{1}}\left(\|x - x_*\|> \frac{\epsilon}{A_1} \right) \leq \exp\left(-\frac{1}{2\Omega_1}\left( \frac{m_1 n \epsilon^2}{2e \omega_1^2}-D_1\right) \right),
\end{split}\end{equation}
where the given bound holds substantive value only if $n>\frac{2e\omega_1^2}{\epsilon^2m_1}D_1$, and thus we specify that the lower bound of the definite integral in $(ii)$ is greater than $n$ ($\mathcal \ell\geq n>\frac{2e\omega_1^2}{\epsilon^2m_1}D_1$). Then $(iii)$ holds by the selection of $\epsilon=(\bar{{\mathcal R}}_1 -\bar{{\mathcal R}}_a)/2=\Delta_a/2$ and $\mathbb \ell$:
\begin{equation}\label{eq_ell}
\mathbb \ell=\left\lceil \frac{8e\omega_1^2}{m_1\Delta_a^2}(D_1+2\Omega_1 \log{2} ) \right\rceil,
\end{equation}
and by performing elementary manipulations on the definite integral:
\begin{equation*}
  \int_{s=1}^{\infty} \frac{1}{1-\frac{1}{2}g(s)}-1 ds = \frac{\log{2}- \log{(2 e^{c}-1)}}{c}+1 \leq \frac{\log{2}}{c}+1 = \frac{16e \Omega_1\omega_1^2}{m_1\Delta_a^2}\log{2} +1.
\end{equation*}
Here function $g(\cdot)$ is defined as $g(s)=\exp\left(-\frac{m_1\epsilon^2}{4e \Omega_1\omega_1^2}s\right)$ for $s\geq \left\lceil \frac{8e\omega_1^2}{m_1\Delta_a^2}\left(D_1+2\Omega_1 \log{2} \right) \right\rceil$ and $c=\frac{m_1\Delta_a^2}{16e \Omega_1\omega_1^2}$. Then the final step $(iv)$ proceeds by the choice of $\mathbb \ell$ mentioned in \eqref{eq_ell}, Lemma \ref{lemma_exact_smallT}, and some simple mathematical operations.

The verification of \eqref{eq_exact_regret2} is achieved through deriving a similar from as \eqref{eq_concentration2}:
\begin{equation*}\begin{split}
  \sum_{s=0}^{N-1} \mathbb{E}\left[ \mathbb{I}\left({\mathcal G}_{a,s}>\frac{1}{N}\right)\right] &= \sum_{s=0}^{N-1} \mathbb{E}\left[ \mathbb{I}\left(\mathbb{P}\left({{\mathcal R}}_{a,s}-\bar{{\mathcal R}}_a>\Delta_a-\epsilon \bigg|\mathcal{F}_{n-1}\right) >\frac{1}{N}\right)\right] \\
  & \stackrel{(i)}{=} \sum_{s=0}^{N-1} \mathbb{E}\left[ \mathbb{I}\left(\mathbb{P}\left({{\mathcal R}}_{a,s}-\bar{{\mathcal R}}_a>\frac{\Delta_a}{2} \bigg|\mathcal{F}_{n-1}\right) >\frac{1}{N}\right)\right] \\
  &\leq \sum_{s=0}^{N-1} \mathbb{E}\left[ \mathbb{I}\left(\mathbb{P}\left(|{{\mathcal R}}_{a,s}-\bar{{\mathcal R}}_a|>\frac{\Delta_a}{2}\bigg|\mathcal{F}_{n-1}\right) >\frac{1}{N}\right)\right] \\
   &\leq \sum_{s=0}^{N-1} \mathbb{E}\left[ \mathbb{I}\left(\mathbb{P}_{x \sim \mu^{(s)}_{a}[\rho_a]}\left(\|x - x_*\|>\frac{\Delta_a}{2\omega_a}\right) >\frac{1}{N}\right)\right] \\
   & \stackrel{(ii)}{\leq} \frac{8 e \omega_a^2}{m_a\Delta_a^2}\left(D_a+2\Omega_a\log N \right),
\end{split}\end{equation*}
where $(i)$ is from our previous defined $\epsilon=\Delta_a/2$, and $(ii)$ proceeds by inducing \eqref{eq_concentration2} with the choice of 
\begin{equation*}
n\ge \left\lceil \frac{8e\omega_1^2}{m_1\Delta_a^2}\left(D_1+2\Omega_1 \log N \right) \right\rceil,
\end{equation*}
which completes the proof.

\end{proof}
With the support of Lemma \ref{lemma_exact_smallT} and Lemma \ref{lemma_exact_ts2}, the proof of Theorem \ref{theorem_exact_regret} becomes straightforward and is detailed below.

\begin{theorem}[Regret of exact Thompson sampling]\label{theorem_exact_regret}
  For likelihood, reward distributions, and priors that meet Assumptions \ref{assumption_likelihood}-\ref{assumption_prior}, and given that $\rho_a=\frac{\nu_am_a^2}{8dL_a^3}$ holds for every $a\in\mathcal A$, the Thompson sampling with exact sampling yields the following total expected regrets after $N>0$ rounds:
  \begin{equation*}\begin{split}
      \mathbb{E}[{\mathfrak R}(N)]&\leq \sum_{a>1} \left[\frac{C_{a}}{\Delta_a}\left(\log B_a + d^2 + d\log N\right) + \frac{C_1}{\Delta_a}\sqrt{B_1}\left(\log B_1 + d^2\right) + 2\Delta_a \right],
  \end{split}\end{equation*}
  where $C_1$ and $C_a$ are universal constants and are unaffected by parameters specific to the problem.
\end{theorem}

\begin{proof}
From the above derivation, the total expected regrets by performing the exact Thompson sampling algorithm can be upper bounded by:
\begin{equation*}\begin{split}
  \mathbb{E}[{\mathfrak R}(N)] & = \sum_{a>1} \mathbb{E}\left[ \mathcal L_a(N) \right]\\
  & \stackrel{(i)}{\leq} \sum_{a>1} \left\{\sum_{s=0}^{N-1}\mathbb{E}\left[ \frac{1}{{\mathcal G}_{1,s}}-1 \right] + \Delta_a + \sum_{s=0}^{N-1} \mathbb{E}\left[ \mathbb{I}\left(1-{\mathcal G}_{a,s}>\frac{1}{N}\right)\right]\right\} \\
  & \stackrel{(ii)}{\leq} \sum_{a>1} \left[ \frac{8 e \omega_a^2}{m_a\Delta_a}\left(D_a+2\Omega_a\log N \right) +\sqrt{\kappa_1 B_1}\frac{128 e \omega_1^2}{m_1\Delta_a} \left(D_1+2\Omega_1 \right) + 2\Delta_a \right]\\
  & = \sum_{a>1}  \frac{16 e \omega_a^2}{m_a\Delta_a}\left(\log B_a + \frac{4d}{\rho_a}+16\left(\frac{L_a^2d}{m_a\nu_a}+\frac{32}{\rho_a}\right)\log N \right) \\
  & \ \ \ \ \ \ \ \ +\sqrt{\kappa_1 B_1}\frac{256 e \omega_1^2}{m_1\Delta_a} \left(\log B_1 + \frac{4d}{\rho_1}+16\left(\frac{L_1^2d}{m_1\nu_1}+\frac{32}{\rho_1}\right) \right) + 2\Delta_a  \\
  & \stackrel{}{\leq} \sum_{a>1} \left[\frac{C_{a}}{\Delta_a}\left(\log B_a + d^2 + d\log N\right) + \frac{C_1}{\Delta_a}\sqrt{B_1}\left(\log B_1 + d^2\right) + 2\Delta_a \right]
\end{split}\end{equation*}
where $(i)$ holds by Lemma \ref{lemma_suboptimal1} and Lemma \ref{lemma_suboptimal2}, the derivation of $(ii)$ draws directly from the result established in Lemma \ref{lemma_exact_ts2}. Through the appropriate selection of constants $C_1$ and $C_a$, the desired regret bound is consequently derived.
\end{proof}

\subsection{Supporting Proofs for Exact Thompson Sampling}

We start by presenting some propositions pertaining to the prior $\log \pi\left(x\right)$ and log-likelihood function $\log \mathtt P\left({\mathcal R}|x\right)$, which serve as foundational elements for the previous proof of Theorem \ref{theorem_posterior_concentration2}.

\begin{proposition}[Proposition 2 in \citet{mazumdar2020approximate}]\label{prop_prior}
  If the prior distribution over ${x}\in \mathbb R^d$ satisfies Assumption \ref{assumption_prior}, then the following statement holds for all $x$:
  \begin{equation*}
    \sup_{x} \left\langle\nabla \log \pi\left({x}\right), {x}-x_*\right\rangle \leq \max_{x}  \left[\log \pi\left({x}\right)\right] - \log \pi\left(x_*\right).
  \end{equation*}
  \end{proposition}

\begin{proof}

From Assumption \ref{assumption_prior} we have: 
\begin{equation*}
  \left\langle\nabla \log \pi\left({x}\right), {x}-x_*\right\rangle \leq \log \pi\left({x}\right) - \log \pi\left(x_*\right).
\end{equation*}
Taking the supremum of both sides completes the proof. 

\end{proof}
\begin{remark}
We here define $\log B\coloneqq \max_{x}  \left[\log \pi\left({x}\right)\right] - \log \pi\left(x_*\right)$. When the prior is ideally centered around the point $x_*$, the implication is that $\log B=0$. Consequently, $B$ becomes a key parameter affecting our posterior concentration rates.
\end{remark}

We proceed to demonstrate that the empirical likelihood function evaluated at $x_*$ can be characterized as a sub-Gaussian random variable.
\begin{proposition}[Proposition 3 in \citet{mazumdar2020approximate}]\label{prop_likelihood}
  Suppose $f_{n}=\frac{1}{n}\sum_{j=1}^n \log \mathtt P({\mathcal R}_j|x)$, where $\log \mathtt P({\mathcal R}_j|x)$ follows Assumption \ref{assumption_reward}. Then $\left\|\nabla f_{ n}\left(x_*\right)\right\|_2$ is $L \sqrt{\frac{d}{n \nu}}$-sub-Gaussian random variable.
\end{proposition}

\begin{proof}
We start by showing that $\nabla \log \mathtt P\left({\mathcal R} | x_*\right)$ is $\frac{L}{\sqrt{\nu}}$ sub-Gaussian random variable. Let $\mathtt{u} \in \mathbb{S}_{d}$ be an arbitrary point in the $d$-dimensional sphere. Then we have:
\begin{equation*}
  \begin{split}
    \left|\left\langle\nabla \log \mathtt P\left({\mathcal R}_1 | x_*\right), \mathtt{u}\right\rangle -\left\langle\nabla \log \mathtt P\left({\mathcal R}_2 | x_*\right), \mathtt{u}\right\rangle \right| & = \left|\left\langle\nabla \log \mathtt P\left({\mathcal R}_1 | x_*\right) - \nabla \log \mathtt P\left({\mathcal R}_2 | x_*\right), \mathtt{u}\right\rangle\right| \\
    & \stackrel{(i)}{\leq} \left\| \nabla \log \mathtt P\left({\mathcal R}_1 | x_*\right) - \nabla \log \mathtt P\left({\mathcal R}_2 | x_*\right)\right\|_2\|\mathtt{u}\|_2 \\
    & =\left\|\nabla \log \mathtt P\left({\mathcal R}_1 | x_*\right) - \nabla \log \mathtt P\left({\mathcal R}_2 | x_*\right)\right\|_2 \\
    & \stackrel{(ii)}{\leq} L\left\|{\mathcal R}_1 - {\mathcal R}_2\right\|
  \end{split}
\end{equation*}
where $(i)$ follows by Cauchy-Schwartz inequality, and $(ii)$ by Assumption \ref{assumption_reward}. Following an elementary invocation of Proposition 2.18 in \citet{ledoux2001concentration}, we arrive at the conclusion that $\left\langle\nabla \log \mathtt P\left({\mathcal R} | x_*\right), \mathtt{u}\right\rangle$ is sub-Gaussian with parameter $\frac{L}{\sqrt{\nu}}$.

Given that the projection of $\nabla \log \mathtt P\left({\mathcal R} | x_*\right)$ onto an arbitrary unit vector $\mathtt{u}$ adheres to sub-Gaussian properties with a $\frac{L}{\sqrt{\nu}}$-independent parameter, we deduce that the random vector $\nabla \log \mathtt P\left({\mathcal R} | x_*\right)$ is also sub-Gaussian, characterized by the same parameter \( \frac{L}{\sqrt{\nu}} \). It follows that $\nabla f_{a, n}\left(x_*\right)$, which is the mean of $n$ independent and identically distributed (i.i.d.) sub-Gaussian vectors, remains sub-Gaussian with parameter $\frac{L}{\sqrt{n \nu}}$. Therefore, as a direct application of Lemma 1 from \citet{jin2019short}, the vector $\nabla f_{a, n}\left(x_*\right)$ can be further identified as norm sub-Gaussian with parameter $L\sqrt{\frac{d}{n \nu}}$, which completes the proof.
\end{proof}

\begin{lemma}\label{lemma_contraction}
Let pairs $(x_t, v_t)$ and $(x_*, v_*)$ are both in $\mathbb R^{2d}$ and suppose $f(x): \mathbb R^d\rightarrow \mathbb R$ fulfills Assumption \ref{assumption_likelihood}, then the following inequality holds:

\begin{equation}
  \left\langle{ \left(\gamma-1\right)} \bar v + u\left(\nabla f\left( {{x_t}} \right) - \nabla f\left( {{x_*}} \right)\right), {\bar x + \bar v}\right\rangle - \left\langle \bar v, \bar x\right\rangle
   {\geq} \frac{m}{2L}\left(\left\|\bar x\right\|^2_2 + \left\|\bar x + \bar v\right\|^2_2\right),
\end{equation}
where $\bar x = x_t - x_*$ and $\bar v= {v_t}-v_*$.
\end{lemma}

\begin{proof}

According to the Mean Value Theorem for Integrals, we derive the relationship between the gradient of $ f(\cdot)$ and the Hessian matrix of $ f(\cdot)$:
\begin{equation*}
  \begin{split}
      \nabla f\left(x_t\right)-\nabla f\left(x_*\right)&=\int_0^1 \frac{d}{dh}\nabla f\left(h x_t+(1-h)x_*\right)\\
      &=\underbrace{\left[\int_0^1 \nabla^2 f\left(x_t+h\left(x_*-x_t\right)\right) d h\right]}_{\coloneqq \psi_t} (x_t-x_*).
  \end{split}
\end{equation*}

Under the given Assumption \ref{assumption_likelihood}, the Hessian matrix of $f(\cdot)$ is positive definite over the domain, with the inequality $m \mathbf{I}_{d \times d} \preceq \nabla^2 f(\cdot) \preceq L \mathbf{I}_{d \times d}$ holds. Next we follow the equations given in \eqref{eq_sde} to derive:
\begin{equation}\label{eq_sde_derivative}
  \begin{split}
    &\left\langle{ \left(\gamma-1\right)} ({v_t}-v_*) + u \left(\nabla f\left( {{x_t}} \right) - \nabla f\left( {{x_*}} \right)\right), {(x_t+v_t)-(x_* + v_*)}\right\rangle - \left\langle x_t-x_*, v_t-v_*\right\rangle \\
    \stackrel{(i)}{=}& \left[ {(x_t - x_*)^\intercal,(x_t - x_*)^\intercal + (v_t - v_*)^\intercal} \right]
    \underbrace{\begin{bmatrix}
      {  {\mathbf{I}_{d \times d}}}&{{-\mathbf{I}_{d \times d}}}\\
      {(1-\gamma)\mathbf{I}_{d \times d} + u\psi_t}&{(\gamma -1){\mathbf{I}_{d \times d}}}
    \end{bmatrix}}_{\coloneqq {\rm P}_t}
    \begin{bmatrix}
      {x_t - x_*}\\
      {(x_t - x_*) + (v_t - v_*)}
      \end{bmatrix}\\
    \stackrel{(ii)}{=}& \left[ {(x_t - x_*)^\intercal,(x_t - x_*)^\intercal + (v_t - v_*)^\intercal} \right]\left(\frac{{\rm P}_t+{\rm P}_t^\intercal}{2}\right)
    \begin{bmatrix}
      {x_t - x_*}\\
      {(x_t - x_*) + (v_t - v_*)}
      \end{bmatrix},
  \end{split}
\end{equation}
where $(i)$ can be obtained through simple matrix operations, and $(ii)$ is because for any vector $(x, v) \in \mathbb R^{2d}$ and matrix ${\rm P}\in \mathbb R^{2d\times 2d}$, the quadratic form $(x, v)^\intercal {\rm P} (x, v)$ is equal to the from with a symmetric matrix $(x, v)^\intercal ({\rm P}+{\rm P}^\intercal) (x, v) / 2$. We now focus on the analysis of the eigenvalues of the matrix $({\rm P}_t+{\rm P}_t^\intercal)/2$. Taking the determinant of $({\rm P}_t+{\rm P}_t^\intercal)/2-\Lambda \mathbf{I}_{2d \times 2d}$ to be 0, we have:
\begin{equation*}
  \text{det}\begin{bmatrix}
      { (1-\Lambda) {\mathbf{I}_{d \times d}}}&{\frac{u\psi_t-\gamma \mathbf{I}_{d \times d}}{2}}\\
      {\frac{u\psi_t-\gamma \mathbf{I}_{d \times d}}{2}}&{(\gamma-1-\Lambda ){\mathbf{I}_{d \times d}}}
  \end{bmatrix}=0.
\end{equation*}
From Lemma \ref{lemma_matirix}, we can rewrite the above equation as:
\begin{equation*}
  \text{det}\begin{pmatrix}(1-\Lambda)(\gamma-1-\Lambda )-\frac{1}{4}\left(\gamma - u\tilde\Lambda_j\right)^2
  \end{pmatrix}=0.
\end{equation*}
where $\tilde\Lambda_j$ ($j=1,2,\cdots,d$) are the eigenvalues of the matrix $\psi_t$. According to Assumption \ref{assumption_likelihood}, we know that for any $j$ the inequality $0<m\leq \tilde\Lambda_j \leq L$ holds. By the selection of $\gamma=2$ and $u=\frac{1}{L}$ we have the solution to each eigenvalue $\Lambda_j$:
\begin{equation*}
  \Lambda_j = 1 \pm \left(1-\frac{\tilde\Lambda_j}{2L}\right).
\end{equation*}
As the value of $\tilde\Lambda_j$ is between $m$ and $L$, it can be concluded that the minimum eigenvalue among $\Lambda_j$ is guaranteed to be greater than or equal to $\frac{m}{2L}$. Therefore, we have
\begin{equation*}
  \begin{split}
    & \left[ {(x_t - x_*)^\intercal,(x_t - x_*)^\intercal + (v_t - x_*)^\intercal} \right]
    \begin{bmatrix}
      {  {\mathbf{I}_{d \times d}}}&{{-\mathbf{I}_{d \times d}}}\\
      {(1-\gamma)\mathbf{I}_{d \times d} + u\psi_t}&{(\gamma -1){\mathbf{I}_{d \times d}}}
    \end{bmatrix}
    \begin{bmatrix}
      {x_t - x_*}\\
      {(x_t - x_*) + (v_t - x_*)}
      \end{bmatrix}\\
    \stackrel{}{=}& \left[ {(x_t - x_*)^\intercal,(x_t - x_*)^\intercal + (v_t - x_*)^\intercal} \right]\left(\frac{{\rm P}_t+{\rm P}_t^\intercal}{2}\right)
    \begin{bmatrix}
      {x_t - x_*}\\
      {(x_t - x_*) + (v_t - x_*)}
      \end{bmatrix}\\
      \geq & \frac{m}{2L} \left[ {(x_t - x_*)^\intercal,(x_t - x_*)^\intercal + (v_t - x_*)^\intercal} \right]\begin{bmatrix}
      {x_t - x_*}\\
      {(x_t - x_*) + (v_t - x_*)}
      \end{bmatrix}\\
      =&\frac{m}{2L}\left(\left\|x_t - x_*\right\|^2_2 + \left\|(x_t+v_t) - (x_* + v_*)\right\|^2_2\right).
  \end{split}
\end{equation*}

By incorporating the findings from \eqref{eq_sde_derivative}, we deduce the conclusive inequality.
\end{proof}

\begin{lemma}[Theorem 3 in \citet{silvester2000determinants}]\label{lemma_matirix}
Considering square matrices ${\rm P}_1, {\rm P}_2, {\rm P}_3$ and ${\rm P}_4$ with dimension $d$, and given the commutativity of ${\rm P}_3$ and ${\rm P}_4$, we can the following results:
\begin{equation*}
	\operatorname{det}\left(
	\begin{bmatrix}
		{\rm P}_1 & {\rm P}_2 \\
		{\rm P}_3 & {\rm P}_4
	\end{bmatrix}\right)=\operatorname{det}({\rm P}_1 {\rm P}_4-{\rm P}_2 {\rm P}_3)
\end{equation*}
\end{lemma}

\begin{lemma}[Sandwich Inequality]\label{lemma_sandwich}
Suppose $x, x_*\in \mathbb R^d$ and $v, v_*\in \mathbb R^d$ are position and velocity terms correspondingly. Then with the Euclidean norm, the following relationship holds:
   \begin{equation}\label{eq_sandwich}
      \left\|(x, v)-(x_*, v_*)\right\|_2 \leq 2\left\|(x, x + v)-(x_*, x_* + v_*)\right\|_2 \leq 4\left\|(x, v)-(x_*, v_*)\right\|_2.
  \end{equation}
\end{lemma}

\begin{proof}
We begin our analysis by the derivation of the first inequality:
\begin{equation*}
	\begin{split}
		\left\|\left(x, v\right)-\left(x_*, v_*\right)\right\|_2^2 & = \left\|x-x_*\right\|_2^2 + \left\|v - v_*\right\|_2^2 \\
		& = \left\|x-x_*\right\|_2^2 + \left\|\left(x+v-x\right) - \left(x_*+v_*-x_*\right)\right\|_2^2 \\
		& \stackrel{(i)}{\leq} \left\|x-x_*\right\|_2^2 + 2\left\|\left(x+v\right) - \left(x_*+v_*\right)\right\|_2^2 + 2\left\|x-x_*\right\|_2^2\\
		& \leq 4\left\|\left(x, x+v\right)-\left(x_*, x_*+v_*\right)\right\|_2^2,
	\end{split}
\end{equation*}
where $(i)$ proceeds by Young's inequality. Then we move our attention to the second inequality:
\begin{equation*}
	\begin{split}
		\left\|\left(x, x+v\right)-\left(x_*, x_*+v_*\right)\right\|_2^2 & = \left\|x-x_*\right\|_2^2 + \left\|\left(x+v\right) - \left(x_*+v_*\right)\right\|_2^2 \\
		& \stackrel{(i)}{\leq} \left\|x-x_*\right\|_2^2 + 2\left\|x-x_*\right\|_2^2 + 2\left\|v - v_*\right\|_2^2 \\
		& \leq 4\left\|\left(x, v\right)-\left(x_*, v_*\right)\right\|_2^2
	\end{split}
\end{equation*}
where $(i)$ is another application of Young's inequality. Taking square roots from the derived inequalities gives us the required results.
\end{proof}

\clearpage
\section{INTRODUCTION TO UNDERDAMPED LANGEVIN MONTE CARLO}\label{section_ULMC}

We first construct a continuous-time Langevin dynamics to target the posterior concentration of $\mu_a^{(n)}$. The continuous-time underdamped Langevin dynamics is derived from the subsequent stochastic differential equations:
\begin{equation}\label{eq_diffusion2}
  \begin{split}
      d v_t & =-\gamma v_t d t-u \nabla U\left(x_t\right) d t+\sqrt{2 \gamma u} d B_t, \\
      d x_t & =v_t d t.
  \end{split}
\end{equation}
Suppose $t$ is the time between rounds $n$ and $n+1$, and we have a set of rewards up to round $n$: $\left\{\mathcal R_1, \mathcal R_1, \ldots, \mathcal R_n\right\}$, then we can express $U\left(x\right)$ using the following equation:
\begin{equation}\label{eq_potential}
	\begin{split}
		U(x_t) & = \sum_{j=1}^n \mathbb{I}\left(A_j=a\right)\log \mathtt P_a ({{\mathcal R}_{j}|x_t}) + \log \pi_a(x_t),
	\end{split}
\end{equation}
where $\mathbb{I}\left(\cdot\right)$ is the indicator function. Underdamped Langevin dynamics is characterized by its invariant distribution, notably proportional to $e^{-U(x_t)}$, which enables the sampling of the unscaled posterior distribution $\mu_a^{(n)}$. By designating the potential function as \eqref{eq_potential}, we obtain a continuous-time dynamic that guarantees trajectories converging rapidly toward the posterior distribution  \(\mu_a^{(n)}\). To transition from theoretical continuous-time dynamics to a practical algorithm, we integrate the discrete underdamped Langevin dynamics to derive \eqref{eq_diffusion2}. Namely, for time $t$ between round $n$ and round $n+1$, we uniformly discretize the time as $I$ segments, where $i\in \llbracket 1, I \rrbracket$ denote the step-index between rounds $n$ and $n+1$, which leads to the following underdamped Langevin Monte Carlo:
\begin{equation}\label{eq_sample_ula3}
\begin{split}
	\begin{bmatrix}
			x_{i+1}\\
			v_{i+1}\\
	\end{bmatrix} \sim \mathcal{N}\left(
	\begin{bmatrix}
		\mathbb E\left[x_{i+1}\right] \\ 
		\mathbb E\left[v_{i+1}\right]
	\end{bmatrix}, \begin{bmatrix}
		\mathbb V\left(x_{i+1}\right) 			&  \mathbb K\left(x_{i+1}, v_{i+1}\right)\\ 
		\mathbb K\left(v_{i+1}, x_{i+1}\right)	&  \mathbb V\left(v_{i+1}\right) \\
	\end{bmatrix}\right),
\end{split}
\end{equation}
where $x_i$, $v_i$ are positions and velocities at step $i$, $\mathbb E\left[x_{i+1}\right]$, $\mathbb E\left[v_{i+1}\right]$, $\mathbb V\left(x_{i+1}\right)$, $\mathbb V\left(v_{i+1}\right)$, and $\mathbb K\left(v_{i+1}, x_{i+1}\right)$ are obtained from the following computations:

\begin{equation*}
	\begin{split}
	\mathbb E\left[{v_{i+1}}\right] &= v_{i} e^{-\gamma {h}} - \frac{u}{\gamma}(1-e^{-\gamma {h}}) {\nabla} U(x_{i})\\
	\mathbb E\left[{x_{i+1}}\right] &= x_{i} + \frac{1}{\gamma}(1-e^{-\gamma {h}})v_{i} - \frac{u}{\gamma} \left( {h} - \frac{1}{\gamma}\left(1-e^{-\gamma {h}}\right) \right) {\nabla}  U(x_{i})\\
	\mathbb V(x_{i+1}) & = \mathbb E\left[{\left(x_{i+1} - \mathbb E\left[{x_{i+1}}\right]\right) \left(x_{i+1} - \mathbb E\left[{x_{i+1}}\right]\right)^{\top}}\right] = \frac{2u}{\gamma} \left[{h}-\frac{1}{2\gamma}e^{-2\gamma{h}}-\frac{3}{2\gamma}+\frac{2}{\gamma	}e^{-\gamma{h}}\right] \cdot \mathbf{I}_{d\times d}\\
	\mathbb V(v_{i+1}) & = \mathbb E\left[{\left(v_{i+1} - \mathbb E\left[{v_{i+1}}\right]\right) \left(v_{i+1} - \mathbb E\left[{v_{i+1}}\right]\right)^{\top}}\right] = u(1-e^{-2\gamma  {h}})\cdot \mathbf{I}_{d\times d}\\
	\mathbb K\left({x_{i+1}}, {v_{i+1}}\right) & = \mathbb E\left[{\left(x_{i+1} - \mathbb E\left[{x_{i+1}}\right]\right) \left(v_{i+1} - \mathbb E\left[{v_{i+1}}\right]\right)^{\top}}\right] = \frac{u}{\gamma} \left[1+e^{-2\gamma{h}}-2e^{-\gamma{h}}\right] \cdot \mathbf{I}_{d \times d},
	\end{split}   
\end{equation*}

where ${h}$ is the step size of the algorithm. A detailed proof can be found in Appendix A of \citet{cheng2018underdamped}. Within this update rule, the gradient of the potential function, $\nabla U(x_{i})$, is proportional to the dataset size $n$:
\begin{equation}\label{eq_gradient_u}
	\nabla U(x_{i}) = \sum_{j=1}^{\mathcal L_a(n)} \nabla \log \mathtt P_a \left({{\mathcal R}_{a,j}|x_{i}}\right)  + \nabla \log \pi_a(x_{i}).
\end{equation}

To address the increasing terms in $\nabla U(x_{i})$, we adopt stochastic gradient methods. Specifically, we define the calculation of the stochastic gradient, $\hat U(x_{i})$, as:
\begin{equation}\label{eq_gradient_hatu}
	\nabla \hat U(x_{i}) = \frac{\mathcal L_a(n)}{|\mathcal{S}|} \sum_{\mathcal R_{a,j}\in\mathcal{S} } \nabla \log \mathtt P_a \left({{\mathcal R}_{a,j}|x_{i}}\right) + \nabla \log \pi_a(x_{i}),
\end{equation}
with $\mathcal{S}$ representing a subset of the dataset. Typically, $\mathcal{S}$ is obtained through subsampling from $\{{\mathcal R}_{a,1},\cdots,{\mathcal R}_{a,j}\cdots,{\mathcal R}_{a,\mathcal L_a(n)}\}$. We further clarify that the cardinality of $\mathcal{S}$, denoted as $|\mathcal S|$ or $k$, is the batch size for the stochastic gradient estimate. It is noteworthy that, early in the Thompson sampling algorithm, the round $n$ being played might be less than the designated batch size $k$. Consequently, we redefine the batch size to be $|\mathcal S|=\min \left\{\mathcal L_a(n), k\right\}$. Incorporating the stochastic gradient $\nabla \hat{U}$ as a replacement for the full gradient $\nabla U$ within our update procedures yields the formulation of the Stochastic Gradient underdamped Langevin Monte Carlo, as is shown in Algorithm \ref{algorithm_langevin_mcmc2}:
\begin{algorithm*}[!htbp]
\caption{(Stochastic Gradient) underdamped Langevin Monte Carlo for arm $a$ at round $n$.}\label{algorithm_langevin_mcmc2}
{\hspace*{\algorithmicindent} \textbf{Input} Data $\{{\mathcal R}_{a,1}, {\mathcal R}_{a,2}, \cdots, {\mathcal R}_{a,\mathcal L_a(n-1)}\}$;}\\
{\hspace*{\algorithmicindent} \textbf{Input} Sample $(x_{a, Ih^{(n-1)}}, v_{a, Ih^{(n-1)}})$ from last round;}
\begin{algorithmic}[1]
\State {Initialize $x_{0}=x_{a,Ih^{(n-1)}}$ and $v_{0}=v_{a,Ih^{(n-1)}}$.}
\For{$i=0,1,\cdots, I-1$}
\State {Uniformly subsample data set $\mathcal{S} \subseteq \{{\mathcal R}_{a,1}, {\mathcal R}_{a,2}, \cdots, {\mathcal R}_{a,\mathcal L_a(n-1)}\}$. }
\State {Compute $\nabla{U}(x_i)$ with \eqref{eq_gradient_u} (or $\nabla{\hat U}(x_i)$ with \eqref{eq_gradient_hatu}).}
\State {Sample $\left(x_{i+1}, v_{i+1}\right)$ with \eqref{eq_sample_ula3} based on $\nabla{U}(x_i)$ ($\nabla{\hat U}(x_i)$) and $(x_i, v_i)$.}
\EndFor
\State $x_{a,I h^{(n)}} \sim \mathcal{N}\left( x_{I} , \frac{1}{nL_a\rho_a}\mathbf{I}_{d\times d}\right)$ and $v_{a,I h^{(n)}}=v_{I}$
\end{algorithmic}
\hspace*{\algorithmicindent} \textbf{Output} Sample $(x_{a, Ih^{(n)}}, v_{a, Ih^{(n)}})$ from current round;\\
\end{algorithm*}

\clearpage
\section{ANALYSIS OF APPROXIMATE THOMPSON SAMPLING}

This section initially investigates the potential for underdamped Langevin Monte Carlo to accelerate the convergence rate in Thompson sampling, incorporating a quantitative examination of the required sample complexity for effective posterior approximation. Next, we evaluate the concentration properties of approximate samples yielded by underdamped Langevin Monte Carlo, covering both full and stochastic gradient cases. At the end of this section, we provide the regret analysis for the proposed Thompson sampling with underdamped Langevin Monte Carlo, under the condition of effective posterior sample approximation to reach sub-linear regrets.

Similar to our earlier analysis on exact Thompson sampling, we begin by providing a summary table outlining the notation that will be invoked in the following sections.

\subsection{Notation}

To further our investigations into approximate Thompson sampling, we have to introduce additional notations, as shown in Table \ref{table_notation2}. This table elaborates the symbols in the analysis of underdamped Langevin Monte Carlo and the regret analysis of approximate Thompson sampling. As delineated in Section \ref{section_ULMC}, underdamped Langevin Monte Carlo encompasses two algorithmic versions based on gradient estimation: the full gradient and the stochastic gradient versions. The symbols in full gradient version are represented using the ``$\ \tilde{}\ $'' notation (as in $\tilde x$, $\tilde v$, $\tilde \mu$), and the stochastic gradient version adopts the ``$\ \hat{}\ $'' notation. It should be noted that, for regret analysis of approximate Thompson sampling, we ensure a congruent posterior concentration rate across both versions, facilitated by an informed choice of step size and number of samples.

\begin{table}[!htbp]
\centering
\caption{Additional notation in the analysis of approximate Thompson sampling.}\label{table_notation2}
\begin{tabular}{|c|c|}
\hline
Symbols & Explanations   \\ \hline  \hline
  $i$       &   number of steps for the current round $n$ $(i\in\llbracket 1, I \rrbracket)$  \\ \hline 
  $k$       &   batch size for the stochastic gradient estimate  \\ \hline 
   $\zeta$    & optimal couplings between two measures \\ \hline
  $\bar\rho_{a}$    &   parameter to scale approximate posterior for arm $a$         \\ \hline
  $h^{(n)}$   &   \begin{tabular}{@{}c@{}}step size of the approximate sampling algorithm \\ for round $n$ from one of the arm  \end{tabular} \\ \hline 
  $\delta\left(\cdot\right)$       &   Dirac delta distribution  \\ \hline\hline
   $\tilde x_a\ (\tilde v_a)$          &  \begin{tabular}{@{}c@{}}sampled position (velocity) of arm $a$ \\ follows underdamped Langevin Monte Carlo with full gradient \end{tabular}    \\ \hline
  $\hat x_a\ (\hat v_a)$    &   \begin{tabular}{@{}c@{}}sampled position (velocity) of arm $a$ \\ follows underdamped Langevin Monte Carlo with stochastic gradient \end{tabular}  \\ \hline
   $\tilde\mu_a^{(n)} \ \left(\hat\mu_a^{(n)}\right)$  &  \begin{tabular}{@{}c@{}}probability measure of posterior distribution after $n$ rounds \\ approximated by underdamped Langevin Monte Carlo\end{tabular}		\\ \hline
   $\tilde\mu_{ih^{(n)}}\ \left(\hat\mu_{ih^{(n)}}\right)$ & \begin{tabular}{@{}c@{}}probability measure of posterior distribution at $i$th step of round $n$ \\ approximated by underdamped Langevin Monte Carlo\end{tabular}    \\ \hline
   $\tilde\mu_a^{(n)}[\bar\rho_a]\ \left(\hat\mu_a^{(n)}[\bar\rho_a]\right)$  &    \begin{tabular}{@{}c@{}}probability measure of scaled posterior distribution after $n$ rounds \\ approximated by underdamped Langevin Monte Carlo\end{tabular}     \\ \hline 
\end{tabular}
\end{table}

\subsection{Posterior Convergence Analysis}\label{section_posterior_convergence}

When the likelihood meets Assumption \ref{assumption_lipschitz_global_main}, complemented by the prior conforming to Assumption \ref{assumption_prior}, we can employ underdamped Langevin Monte Carlo for the sampling process, ensuring convergence within the 2-Wasserstein distance. Algorithm \ref{algorithm_langevin_mcmc2} indicates that for $n$-th round, the sample starts from the last step of the $(n-1)$-th round and ends on the last step (step $I$) of the $n$-th round. Based on the above analysis, we provide convergence guarantees among rounds from $1$ to $N$.

\begin{theorem}[Posterior Convergence of underdamped Langevin Monte Carlo with full gradient]
\label{theorem_ula_deterministic}
Suppose the log-likelihood function follows Assumption \ref{assumption_lipschitz_global_main}, the prior follows Assumption \ref{assumption_prior}, and the posterior distribution fulfills the concentration inequality $\mathbb E_{(x, v)\sim\mu_a^{(n)}}{\left[\|\left(x, v\right) - \left(x_*, v_*\right)\|_2^2\right]}^{\frac{1}{2}}\leq \frac{1}{\sqrt{n}} \tilde{D}_a$.

By selecting the step size $h^{(n)} = \frac{\tilde D_a}{80\kappa_a}\sqrt{\frac{m_a}{nd}} = \tilde O\left(\frac{1}{\sqrt{d}}\right)$ and the number of steps ${I}\geq \frac{2\kappa_a}{h^{(n)}}\log 24 =\tilde O\left(\sqrt{d}\right)$, we are able to bound the convergence of Algorithm \ref{algorithm_langevin_mcmc2} in 2-Wasserstein distance w.r.t. the posterior distribution $\mu_a^{(n)}$: 
$W_2\left(\tilde\mu_a^{(n)}, \mu_a^{(n)}\right) \leq \frac{2}{\sqrt{n}} \tilde{D}_a$, where $\tilde{D}_a\geq 8\sqrt{\frac{d}{m_a}}$.
\end{theorem}

\begin{proof}
To prove Theorem \ref{theorem_ula_deterministic}, we employ an inductive approach, with the starting point being $n=1$. When $n=1$, we invoke Theorem \ref{theorem_ULA2} with initial condition $\mu_0=\delta(x_0, v_0)$ to have the convergence of Algorithm \ref{algorithm_langevin_mcmc2} in the $i$-th iteration after the first pull to arm a:
\begin{equation}\label{eq_ula_coonverge1}
  \begin{split}
    W_2\left(\tilde\mu_a^{(1)}, \mu_a^{(1)}\right) & = W_2\left(\tilde\mu_{ih^{(1)}}, \mu_a^{(1)}\right) \\
    & \stackrel{(i)}{\leq} 4e^{-{ih^{(1)}}/2\kappa_a}W_2\left(\delta(x_0, v_0), \mu_a^{(1)}\right) + \frac{20\left(h^{(1)}\right)^2}{1-e^{-{h^{(1)}}/ 2\kappa_a}} \sqrt{\frac{d}{m_a}} \\
    & \stackrel{(ii)}{\leq} 4e^{-{ih^{(1)}}/2\kappa_a}\left[W_2\left(\delta(x_*, v_*), \mu_a^{(1)}\right)+W_2\left(\delta(x_0, v_0), \delta(x_*, v_*)\right)\right] + \frac{20\left(h^{(1)}\right)^2}{1-e^{-{h^{(1)}}/ 2\kappa_a}} \sqrt{\frac{d}{m_a}} \\
    & \stackrel{(iii)}{\leq} 4e^{-{ih^{(1)}}/2\kappa_a}\left(6\sqrt{\frac{d}{m_a}}+\sqrt{\frac{6d}{m_a}}\right) + \frac{20\left(h^{(1)}\right)^2}{1-e^{-{h^{(1)}}/ 2\kappa_a}} \sqrt{\frac{d}{m_a}} \\
    & \stackrel{(iv)}{\leq} 24\sqrt{2}e^{-{ih^{(1)}}/2\kappa_a}\sqrt{\frac{d}{m_a}} + 80h^{(1)}\kappa_a \sqrt{\frac{d}{m_a}} \\
  \end{split}
\end{equation}
where $(i)$ holds by the results of Theorem \ref{theorem_ULA2}, $(ii)$ is from triangle inequality, $(iii)$ is from Lemma \ref{lemma_posterior_concentrate} and Lemma \ref{lemma_initial_energy}, $(iv)$ is because $6+\sqrt{6}<6\sqrt{2}$ and $1/(1-e^{-{h^{(1)}}/ 2\kappa_a})\leq 4\kappa_a / h^{(1)}$ holds when $h^{(1)}/\kappa_a < 1$. With the selection of $h^{(1)}= \frac{\tilde D_a}{80\kappa_a}\sqrt{\frac{m_a}{d}}$ and ${I}\geq \frac{2\kappa_a}{h^{(1)}}\log\left(\frac{24}{\tilde D_a}\sqrt{\frac{2d}{m_a}}\right)$, we can guarantee that Algorithm \ref{algorithm_langevin_mcmc2} converge to $2\tilde D_a$.

After pulling arm $a$ for the $(n-1)$-th round and before the $n$-th round, following the result given by \eqref{eq_ula_coonverge1}, we know that $W_2 \left({\tilde \mu _a^{(n-1)}, \mu_a^{(n-1)}}\right) \leq \frac{2}{\sqrt{n-1}} \tilde{D}_a$ can be guaranteed. We now continue to prove that, post the $n$-th pull, the constraint further refines to $W_2 \left({\tilde \mu _a^{(n)}, \mu_a^{(n)}}\right) \leq \frac{2}{\sqrt{n}} \tilde{D}_a$. From the above analysis, we have:

\begin{equation}\label{eq_ula_coonverge2}
  \begin{split}
    W_2\left(\tilde\mu_a^{(n)}, \mu_a^{(n)}\right) & = W_2\left(\tilde\mu_{{i}h^{(n)}}, \mu_a^{(n)}\right) \\
    & \stackrel{}{\leq} 4e^{-{ih^{(n)}}/2\kappa_a}W_2\left(\tilde\mu_a^{(n-1)}, \mu_a^{(n)}\right) + \frac{20\left(h^{(n)}\right)^2}{1-e^{-{h^{(n)}}/ 2\kappa_a}} \sqrt{\frac{d}{m_a}} \\
    & \stackrel{}{\leq} 4e^{-{ih^{(n)}}/2\kappa_a}\left[W_2\left(\tilde\mu_a^{(n-1)}, \mu_a^{(n-1)}\right) + W_2\left(\mu_a^{(n-1)}, \mu_a^{(n)}\right)\right] + \frac{20\left(h^{(n)}\right)^2}{1-e^{-{h^{(n)}}/ 2\kappa_a}} \sqrt{\frac{d}{m_a}} \\
    & \stackrel{(i)}{\leq} 4e^{-{ih^{(n)}}/2\kappa_a}\left[\frac{3}{\sqrt{n}}\tilde D_a + W_2\left(\mu_a^{(n-1)}, \delta(x_*, v_*)\right) + W_2\left(\delta(x_*, v_*), \mu_a^{(n)}\right)\right] + \frac{20\left(h^{(n)}\right)^2}{1-e^{-{h^{(n)}}/ 2\kappa_a}} \sqrt{\frac{d}{m_a}} \\
    & \stackrel{(ii)}{\leq} 24e^{-{ih^{(n)}}/2\kappa_a}\frac{\tilde D_a}{\sqrt{n}}  + \frac{20\left(h^{(n)}\right)^2}{1-e^{-{h^{(n)}}/ 2\kappa_a}} \sqrt{\frac{d}{m_a}} \\
    & \stackrel{}{\leq} 24e^{-{ih^{(n)}}/2\kappa_a}\frac{\tilde D_a}{\sqrt{n}} + 80h^{(n)}\kappa_a \sqrt{\frac{d}{m_a}}, \\
  \end{split}
\end{equation}
where $(i)$ is because $W_2\left(\tilde\mu_a^{(n-1)}, \mu_a^{(n-1)}\right)\leq \frac{2}{\sqrt{n-1}}\tilde D_a \leq \frac{3}{\sqrt{n}}\tilde D_a $, $(ii)$ proceeds by our posterior assumption:
\begin{equation*}
	W_2\left(\mu_a^{(n-1)}, \delta(x_*)\right) + W_2\left(\delta(x_*), \mu_a^{(n)}\right) \leq \left(\frac{1}{\sqrt{n-1}} + \frac{1}{\sqrt{n}}\right)\tilde D_a \leq \frac{3}{\sqrt{n}} \tilde D_a.
\end{equation*}
With the selection of $h^{(n)} = \frac{\tilde D_a}{80\kappa_a}\sqrt{\frac{m_a}{nd}} = \tilde O\left(\frac{1}{\sqrt{d}}\right)$ and $I\geq \frac{2\kappa_a}{h^{(n)}}\log 24 =\tilde O\left(\sqrt{d}\right)$, we can guarantee that the approximate error after $(n-1)$-th round and before $n$-th round can be upper bounded by ${2}\tilde{D}_a / \sqrt{n}$.
\end{proof}

Given the log-likelihood function also satisfies Assumption \ref{assumption_sgld_lipschitz_main}, we observe similar convergence guarantees for the stochastic gradient version of the underdamped Langevin Monte Carlo.

\begin{theorem}[Posterior Convergence of underdamped Langevin Monte Carlo with stochastic gradient]
\label{theorem_ula_stochastic}
Given that the log-likelihood aligns with Assumptions \ref{assumption_lipschitz_global_main} and \ref{assumption_sgld_lipschitz_main}, and the prior adheres to Assumption \ref{assumption_prior}, we define the parameters for Algorithm \ref{algorithm_langevin_mcmc2} as follows: stochastic gradient samples are taken as $k=\tilde O\left(\kappa_a^2\right)$, step size as $h^{(n)} = \tilde O\left(\frac{1}{\sqrt{d}}\right)$, and the total steps count as $I=\tilde O\left(\sqrt{d}\right)$. Given the posterior distribution follows the prescribed concentration inequality $\mathbb E_{(x, v)\sim\mu_a^{(n)}}{\left[\|\left(x, v\right) - \left(x_*, v_*\right)\|_2^2\right]}^{\frac{1}{2}}\leq \frac{1}{\sqrt{n}} \hat{D}_a$, we have convergence of the underdamped Langevin Monte Carlo in 2-Wasserstein distance to the posterior $\mu_a^{(n)}$: 
$W_2\left(\hat \mu_a^{(n)}, \mu_a^{(n)}\right) \leq \frac{2}{\sqrt{n}} \hat{D}_a$, where $\hat{D}_a\geq 8\sqrt{\frac{d}{m_a}}$.
\end{theorem}

\begin{proof}

To demonstrate this theorem, we employ an inductive methodology as outlined in Theorem \ref{theorem_ula_deterministic}. Specifically, for $n=1$, we examine the convergence after $i$-th iterations to the first pull of arm $a$:
\begin{equation}\label{eq_pos_convergencer1}
  \begin{split}
    W_2\left(\hat\mu_a^{(1)}, \mu_a^{(1)}\right) & = W_2\left(\hat\mu_{ih^{(1)}}, \mu_a^{(1)}\right) \\
    & \stackrel{(i)}{\leq} 4e^{-{ih^{(1)}}/2\kappa_a}W_2\left(\delta(x_0, v_0), \mu_a^{(1)}\right) + \frac{4}{1-e^{-{h^{(1)}}/ 2\kappa_a}}\left(5\left(h^{(1)}\right)^2\sqrt{\frac{d}{m_a}} + 4u{h^{(1)}} {L}_a\sqrt{\frac{nd}{k_a\nu_a}}\right)\\   
    & \stackrel{(ii)}{\leq} 4e^{-{ih^{(1)}}/2\kappa_a}\left[W_2\left(\delta(x_*, v_*), \mu_a^{(1)}\right)+W_2\left(\delta(x_0, v_0), \delta(x_*, v_*)\right)\right]  \\
    & \ \ \ \ + \frac{4}{1-e^{-{h^{(1)}}/ 2\kappa_a}}\left(5\left(h^{(1)}\right)^2\sqrt{\frac{d}{m_a}} + 4u{h^{(1)}} {L}_a\sqrt{\frac{nd}{k_a\nu_a}}\right)\\
    & \stackrel{(iii)}{\leq} 4e^{-{ih^{(1)}}/2\kappa_a}\left(6\sqrt{\frac{d}{m_a}}+\sqrt{\frac{6d}{m_a}}\right) + \frac{4}{1-e^{-{h^{(1)}}/ 2\kappa_a}}\left(5\left(h^{(1)}\right)^2\sqrt{\frac{d}{m_a}} + 4u{h^{(1)}} {L}_a\sqrt{\frac{nd}{k_a\nu_a}}\right)\\
    & \stackrel{(iv)}{\leq} 24\sqrt{2}e^{-{ih^{(1)}}/2\kappa_a}\sqrt{\frac{d}{m_a}} + 80\kappa_a h^{(n)} \sqrt{\frac{d}{m_a}} + 64\kappa_a u{L}_a\sqrt{5 \frac{nd  }{k\nu_a} }, \\
  \end{split}
\end{equation}
where $(i)$ holds by the results of Theorem \ref{theorem_ULA3}, $(ii)$ is from triangle inequality, $(iii)$ follows Lemma \ref{lemma_posterior_concentrate} and Lemma \ref{lemma_initial_energy}, $(iv)$ is because $6+\sqrt{6}<6\sqrt{2}$ and $1/(1-e^{-{h^{(1)}}/ 2\kappa_a})\leq 4\kappa_a / h$ hold when $h^{(1)}/\kappa_a < 1$. With the selection of $h^{(1)} = \frac{\tilde D_a}{120\kappa_a}\sqrt{\frac{m_a}{nd}}$ and ${I}\geq \frac{2\kappa_a}{h^{(1)}}\log\left(\frac{36}{\tilde D_a}\sqrt{\frac{2d}{m_a}}\right)$, we can guarantee that the first two terms converges to $\frac{2\hat D_a}{3}$. If we select $k=720\frac{{L}_a^2}{m_a\nu_a}$, then we can upper bound this term by $\frac{2\hat D_a}{3}$:
\begin{equation*}
	 64\kappa_a u{L}_a\sqrt{\frac{5d}{k\nu_a}} = \frac{16}{3}\sqrt{\frac{d}{m_a}} \leq \frac{2\hat D_a}{3},
\end{equation*}
which means Algorithm \ref{algorithm_langevin_mcmc2} converge to $2\hat D_a$. With this upper bound, we continue to bound the distance between $\hat \mu_a^{(n)}$ and $\mu_a^{(n)}$ given $W_2\left(\hat \mu_a^{(n-1)}, \mu_a^{(n-1)}\right)\leq \frac{2}{\sqrt{n-1}} \hat{D}_a$:
\begin{equation*}
  \begin{split}
    W_2\left(\hat \mu_a^{(n)}, \mu_a^{(n)}\right) & = W_2\left(\hat \mu_{ih^{(n)}}, \mu_a^{(n)}\right) \\
    & \stackrel{(i)}{\leq} 4e^{-{ih^{(n)}}/2\kappa_a}W_2\left(\hat \mu_a^{(n-1)}, \mu_a^{(n)}\right) + \frac{4}{1-e^{-{h^{(n)}}/ 2\kappa_a}}\left(5\left(h^{(n)}\right)^2\sqrt{\frac{d}{m_a}} + 4u{h^{(n)}} {L}_a\sqrt{\frac{nd}{k_a\nu_a}}\right)\\   
    & \stackrel{}{\leq} 4e^{-{ih^{(n)}}/2\kappa_a}\left[W_2\left(\hat \mu_a^{(n-1)}, \mu_a^{(n-1)}\right) + W_2\left(\mu_a^{(n-1)}, \mu_a^{(n)}\right)\right] + \frac{4}{1-e^{-{h^{(n)}}/ 2\kappa_a}}\left(5\left(h^{(n)}\right)^2\sqrt{\frac{d}{m_a}} + 4u{h^{(n)}} {L}_a\sqrt{\frac{nd}{k_a\nu_a}}\right)\\
    & \stackrel{(ii)}{\leq} 4e^{-{ih^{(n)}}/2\kappa_a}\left[\frac{3}{\sqrt{n}}\hat D_a + W_2\left(\mu_a^{(n-1)}, \delta(x_*, v_*)\right) + W_2\left(\delta(x_*, v_*), \mu_a^{(n)}\right)\right] \\
    & \ \ \ \ \ \ \ \ + \frac{4}{1-e^{-{h^{(n)}}/ 2\kappa_a}}\left(5\left(h^{(n)}\right)^2\sqrt{\frac{d}{m_a}} + 4u{h^{(n)}} {L}_a\sqrt{\frac{nd}{k_a\nu_a}}\right) \\
    & \stackrel{(iii)}{\leq} 24e^{-{ih^{(n)}}/2\kappa_a}\frac{\hat D_a}{\sqrt{n}} + \frac{4}{1-e^{-{h^{(n)}}/ 2\kappa_a}}\left(5\left(h^{(n)}\right)^2\sqrt{\frac{d}{m_a}} + 4u{h^{(n)}} {L}_a\sqrt{\frac{nd}{k_a\nu_a}}\right)\\
    & \stackrel{}{\leq} 24e^{-{ih^{(n)}}/2\kappa_a}\frac{\hat D_a}{\sqrt{n}} + \underbrace{80\kappa_a h^{(n)} \sqrt{\frac{d}{m_a}}}_{T1} + \underbrace{64\kappa_a u{L}_a\sqrt{5 \frac{nd  }{k\nu_a} }}_{T2}, \\
  \end{split}
\end{equation*}
where $(i)$ is from Theorem \ref{theorem_ULA3}, $(ii)$ and $(iii)$ follow the same idea as $(i)$ and $(ii)$ in \eqref{eq_ula_coonverge2}. 

For terms $T1$ and $T2$, if we take step size $h^{(n)} = \frac{\hat D_a}{120\kappa_a}\sqrt{\frac{m_a}{nd}} = \tilde O\left(\frac{1}{\sqrt{d}}\right)$ and $k=720\frac{{L}_a^2}{m_a\nu_a}$, we will have $T1\leq \frac{2\hat D_a}{3\sqrt{n}}$ and $T2\leq \frac{2\hat D_a}{3\sqrt{n}}$. If the number of steps taken in the SGLD algorithm from $(n-1)$-th pull till $n$-th pull is selected as $I\geq \frac{2\kappa_a}{h^{(n)}}\log16=\tilde O\left(\sqrt{d}\right)$, we then have
\begin{equation*}
	\begin{split}
		W_2\left(\hat \mu_a^{(n)}, \mu_a^{(n)}\right) & \stackrel{}{\leq} 24e^{Ih^{(n)}/ 2\kappa_a}\frac{\hat D_a}{\sqrt{n}} + 8\kappa_a u{L}_a\sqrt{5 \frac{nd  }{k\nu_a} } + 16\kappa_a h^{(n)} \sqrt{\frac{104}{5}\frac{d}{m}} \\
		& \leq \frac{2}{\sqrt{n}} \hat{D}_a.
	\end{split}
\end{equation*}
As we move from the $(n-1)$-th pull to the $n$-th for arm $a$ and given that the first round has been proved to converge to $2 \hat{D}_a$, the number of steps for the $n$-th round can be determined as $\tilde O\left(\sqrt d\right)$.
\end{proof}

\subsection{Posterior Concentration Analysis}\label{section_approx_concentration}

Based on the convergence analysis of the 2-Wasserstein distance metrics in underdamped Langevin Monte Carlo, we continue to deduce the following posterior concentration bounds of the algorithm.

\begin{theorem}\label{theorem_empirical_concentrate1}
Under conditions where the log-likelihood and the prior are consistent with Assumptions \ref{assumption_prior} and \ref{assumption_lipschitz_global_main}, and considering that arm $a$ has been chosen $\mathcal L_a(n)$ times till the $n$th round of the Thompson sampling procedure. By setting the step size $h^{(n)} = \tilde O\left(\frac{1}{\sqrt{d}}\right)$ and number of steps $I = \tilde O\left(\sqrt{d}\right)$ in Algorithm \ref{algorithm_langevin_mcmc2}, then the following concentration inequality holds:

\begin{equation}
	\mathbb{P}_{x_{a,n} \sim \tilde \mu^{(n)}_{a} [\bar \rho_a]} \left(\|x_{a,n} - x_{*}\|_2 > 6\sqrt{\frac{e}{m_a n} \left( D_a + 2\Omega_a\log{1/\delta_1} + 2\tilde\Omega_a\log{1/\delta_2} \right)} \bigg| \mathcal Z_{n-1} \right)<\delta_2,
\end{equation}
where $D_a = 2\log B_a + 8d $, $\Omega_a = 256 +\frac{16d L_a^2}{m_a\nu_a}$, $\tilde\Omega_a = 256 + \frac{16d L_a^2}{m_a\nu_a} +\frac{m_a d}{18L_a \bar\rho_a}$, and we define $\mathcal Z_{n-1}$ as the following set:
\begin{equation*}
    \mathcal Z_{n-1} =\left\{ \left\| x_{a,n-1}-x_{*}\right\|_2 \leq \beta(n) \right\},\\ 
\end{equation*}
where $x_{a,n-1}$ is from the output of Algorithm \ref{algorithm_langevin_mcmc2} from the round $n-1$, and $\beta:\mathbb N\rightarrow \left(0,\infty\right)$ is a constructed confidence interval related to round $n$:
\begin{equation*}
    \beta(n) \coloneqq 3\sqrt{\frac{2e}{ m_a n} \left(D_a + 2\Omega_a \log1/\delta_1\right)}.
\end{equation*}
\end{theorem}
\begin{remark}\label{remark_posterior}
Using extensions from Theorems \ref{theorem_posterior_concentration2} and \ref{theorem_ula_deterministic}, we can establish the upper bound for $\beta(n)$: for arm $a$ in round $n$, Theorem \ref{theorem_posterior_concentration2} provides an upper bound $\sqrt{\frac{2e}{ m_a n} \left(D_a + 2\Omega_a \log1/\delta_1\right)}$ for the posterior concentration, while Theorem \ref{theorem_ula_deterministic} caps the distance between the posterior and its approximation as $\sqrt{\frac{8e}{ m_a n} \left(D_a + 2\Omega_a \log1/\delta_1\right)}$.
\end{remark}

\begin{proof}
We start by analyzing the upper bound of the 2-Wasserstein distance between $\tilde \mu_a^{(n)}[\bar\rho_a]$ and $\delta\left({x_{*}}, {v_{*}}\right)$, and then conclude the posterior concentration rates by taking the marginal posterior distribution. By triangle inequality, we can separate the distance into three terms and we will upper bound these terms one by one:

\begin{equation*}
  \begin{split}
       W_2 \left({\tilde \mu_a^{(n)}[\bar\rho_a], \delta\left({x_{*}}, {v_{*}}\right)}\right) & {\leq} \underbrace{W_2 \left({\tilde \mu_a^{(n)}[\bar\rho_a], {\tilde\mu_{Ih^{(n)}}}}\right)}_{T1} + \underbrace{W_2 \left({\tilde\mu_{Ih^{(n)}}}, \mu_a^{(n)}\right)}_{T2} + \underbrace{W_2 \left({\mu_a^{(n)}, \delta\left({x_{*}}, {v_{*}}\right)}\right)}_{T3} \\
  \end{split}
\end{equation*}

For term $T1$, since the sample returned by the Langevin Monte Carlo is given by: $\tilde x_{a}=\tilde x_{I h^{(n)}}+\mathcal X_a$, $\tilde v_{a}=\tilde v_{I h^{(n)}}$, where $\mathcal X_a\sim \mathcal{N}\left(0,\frac{1}{nL_a\bar\rho_a} \mathbf{I}_{d\times d} \right)$, it remains to bound the distance between the approximate posterior $\tilde \mu_a^{(n)}$ of ${x_a}$ and the distribution of $x_{I h^{(n)} }$. 
Since $\tilde x_{a} - \tilde x_{I h^{(n)}}$ follows a normal distribution, we can upper bound $T1$ by taking expectation of $\mathcal X_a$:
\begin{equation}\label{eq_approx_posterior_t1}
	\begin{split}
		T1 & = W_2 \left({\tilde \mu_a^{(n)}[\bar\rho_a], {\mu_{Ih^{(n)}}}}\right) = \left({ \inf_{ \bar\rho\in\Gamma{\left( \tilde \mu_a^{(n)}[\bar\rho_a], {\mu_{Ih^{(n)}}}\right)}} \int \left\|{{\tilde x_a} - \tilde x_{I h^{(n)}}}\right\|_2^2 d {\tilde x_a} d \tilde x_{I h^{(n)}}
		 }\right)^{1/2} \\
		& \leq 					\sqrt{\mathbb{E}\left[{\left\|\tilde x_{a} - \tilde x_{I h^{(n)}}\right\|_2^2}\right]}\\
		& \stackrel{(i)}{\leq} 	\mathbb{E}\left[{\left\|\tilde x_{a} - \tilde x_{I h^{(n)}}\right\|^p}\right]^{1/p}\\
		& \stackrel{(ii)}{\leq} \sqrt{\frac{d}{nL_a\bar\rho_a}} \left({ \frac{2^{p/2}\mathbf{\Gamma}\left(\frac{p+1}{2}\right)}{\sqrt{\pi}} }\right)^{1/p} \\
		& \stackrel{(iii)}{\leq}	\sqrt{\frac{dp}{nL_a\bar\rho_a}},
	\end{split}
\end{equation}
where $(i)$ follows a simple application of the H$\ddot{ \text{o}}$lder's inequality for any even integer $p\geq 2$, $(ii)$ proceeds by Stirling's approximation for the Gamma function \citep{boas2006mathematical}, and $(iii)$ holds because $\mathbf{\Gamma}\left(\frac{p+1}{2}\right)\leq \sqrt \pi \left(\frac{p}{2}\right)^{p/2}$. We next bound $T2$ following the idea from Theorem \ref{theorem_ula_deterministic}:
\begin{equation}\label{eq_approx_posterior_t2}
	\begin{split}
		T2 & = W_2 \left({\tilde\mu_{Ih^{(n)}}}, \mu_a^{(n)}\right) \\
		& \stackrel{(i)}{\leq}  4e^{-I{h^{(n)}}/ 2\kappa_a}W_2\left(\delta(x_{a, n-1}, v_{a, n-1}), \mu_a^{(n)}\right) + \frac{20\left(h^{(n)}\right)^2}{1-e^{-{h^{(n)}}/ 2\kappa_a}} \sqrt{\frac{d}{m_a}} \\
		& \stackrel{(ii)}{\leq} 4e^{-I{h^{(n)}}/ 2\kappa_a}\left[W_2\left(\delta(x_{a, n-1}, v_{a, n-1}), \delta(x_{*}, v_{*})\right) + W_2\left(\delta(x_{*}, v_{*}), \mu_a^{(n)}\right)\right] + \frac{20\left(h^{(n)}\right)^2}{1-e^{-{h^{(n)}}/ 2\kappa_a}} \sqrt{\frac{d}{m_a}} \\
		& \stackrel{(iii)}{\leq} 4e^{-I{h^{(n)}}/ 2\kappa_a}\left[W_2\left(\delta(x_{a, n-1}, v_{a, n-1}), \mu_a^{(n-1)}\right) + W_2\left(\mu_a^{(n-1)}, \delta(x_{*}, v_{*})\right) + \frac{\tilde D'_a}{\sqrt{n}}\right] + \frac{20\left(h^{(n)}\right)^2}{1-e^{-{h^{(n)}}/ 2\kappa_a}} \sqrt{\frac{d}{m_a}} \\
		& \stackrel{(iv)}{\leq} 4e^{-I{h^{(n)}}/ 2\kappa_a}\left[\frac{3}{\sqrt{n-1}}\tilde D_a + \frac{\tilde D'_a}{\sqrt{n}}\right] + \frac{20\left(h^{(n)}\right)^2}{1-e^{-{h^{(n)}}/ 2\kappa_a}} \sqrt{\frac{d}{m_a}} \\
		& \stackrel{(v)}{\leq} 4e^{-I{h^{(n)}}/ 2\kappa_a} \frac{7}{\sqrt n}\bar D_a + \frac{20\left(h^{(n)}\right)^2}{1-e^{-{h^{(n)}}/ 2\kappa_a}} \sqrt{\frac{d}{m_a}} \\
		& \stackrel{(vi)}{\leq}  \frac{2\bar D_a}{\sqrt n},
	\end{split}
\end{equation}
where $(i)$ is derived from Theorem \ref{theorem_ULA2}, $(ii)$ proceeds by triangle inequality, $(iii)$ arise from the concentration rate for arm $a$ at round $n$ as delineated in Theorem \ref{theorem_posterior_concentration2}: 
\begin{equation*}
	W_2\left(\delta(x_{*}, v_{*}), \mu_a^{(n)}\right) \leq \frac{\tilde D'_a}{\sqrt n} = \sqrt{\frac{2}{m_an}\left(D_a + \tilde\Omega_a p\right)}.
\end{equation*}
Subsequently, $(iv)$ proceeds by the facts in Theorem \ref{theorem_ula_deterministic} that $W_2\left(\delta(x_{a, n-1}, v_{a, n-1}), \mu_a^{(n-1)}\right)\leq \frac{2}{\sqrt{n-1}}\tilde D_a$ (posterior concentration results in Theorem \ref{theorem_posterior_concentration2}) and $W_2\left(\mu_a^{(n-1)}, \delta(x_{*}, v_{*})\right)\leq \frac{1}{\sqrt{n-1}}\tilde D_a$. For $(v)$, we define
\begin{equation*}
	\bar D_a = \sqrt{\frac{2}{m_a}\left(D_a + 2\Omega_a\log1/\delta_1+\tilde\Omega_a p\right)},
\end{equation*}
which can be easily validated that $\bar D_a \geq \max\left\{\sqrt{\frac{2}{m_a}\left(D_a + 2\Omega_a\log1/\delta_1\right)}, \sqrt{\frac{2}{m_a}\left(D_a+\tilde\Omega_a p\right)}\right\}$. Then we follow a similar procedure of proving Theorem \ref{theorem_ula_deterministic} to arrive at the fact in $(vi)$.

For $T3$, we can easily bound it with the result of posterior concentration inequality in Theorem \ref{theorem_ula_deterministic}:
\begin{equation}\label{eq_approx_posterior_t3}
		T3 = W_2 \left({\mu_a^{(n)}, \delta(x_{*}, v_{*})}\right) \leq \frac{\tilde D_a}{\sqrt n}.
\end{equation}

We now move our attention back to upper bound $W_2 \left({\tilde \mu_a^{(n)}[\bar\rho_a], \delta\left({x_{*}}, {v_{*}}\right)}\right)$ and derive the following results:
\begin{equation*}
	  \begin{split}
	       W_2 \left({\tilde \mu_a^{(n)}[\bar\rho_a], \delta\left({x_{*}}, {v_{*}}\right)}\right) & \stackrel{}{\leq} W_2 \left({\tilde \mu_a^{(n)}[\bar\rho_a], {\tilde\mu_{Ih^{(n)}}}}\right) + W_2 \left({\tilde\mu_{ih^{(n)}}}, \mu_a^{(n)}\right) + W_2 \left({\mu_a^{(n)}, \delta\left({x_{*}}, {v_{*}}\right)}\right) \\
	       & \stackrel{(i)}{\leq} \sqrt{\frac{dp}{nL_a\bar\rho_a}} + \frac{2\bar D_a}{\sqrt n} + \frac{\tilde D_a}{\sqrt n}\\
	       & \stackrel{(ii)}{\leq} \sqrt{\frac{dp}{nL_a\bar\rho_a}} + \frac{3\bar D_a}{\sqrt n} \\
	       & \stackrel{}{\leq} \sqrt{\frac{36}{m_an}\left(d+\log B_a + 2\Omega_a\log 1/\delta_1+\left(\Omega_a+\frac{d}{18L_a\bar\rho_a}\right)p\right)}
	  \end{split}
\end{equation*}
where $(i)$ proceeds by the results of \eqref{eq_approx_posterior_t1}-\eqref{eq_approx_posterior_t3}, $(ii)$ holds because $\tilde D_a\leq \bar D_a$. Following the same idea in \eqref{eq_bound2}-\eqref{eq_posterior_derivation2} and with the selection of $p=2\log1/\delta_2$ and $\tilde\Omega_a = \Omega_a+\frac{d}{18L_a\bar\rho_a}$, the derived upper bound finally indicates a marginal posterior concentration inequality:
\begin{equation*}
 	\mathbb{P}_{x_{a,n} \sim \tilde \mu^{(n)}_{a} [\bar\rho_a]} \left(\|x_{a,n} - x_{*}\|_2 > 6\sqrt{\frac{e}{m_a n} \left( D_a + 2\Omega_a\log{1/\delta_1} + 2\tilde\Omega_a\log{1/\delta_2} \right)} \bigg| \mathcal Z_{n-1} \right)<\delta_2. 
\end{equation*}

\end{proof}

Using a similar reasoning process, we can deduce that the next Theorem stays true.

\begin{theorem}
\label{theorem_empirical_concentrate2}
Under conditions where the log-likelihood is consistent with Assumptions \ref{assumption_lipschitz_global_main}, the prior aligns with Assumption \ref{assumption_prior}, and considering that arm $a$ has been chosen $\mathcal L_a(n)$ times till the $n$th round of the Thompson sampling. By setting the batch size for stochastic gradient estimates as $k=\tilde O\left(\kappa_a^2\right)$, the step size and number of steps in Algorithm \ref{algorithm_thompson} as $h^{(n)} =\hat O\left(\frac{1}{\sqrt{d}}\right)$ and $I = \hat O\left(\sqrt{d}\right)$ accordingly, we can derive the following posterior concentration bound:
\begin{equation*}
	\mathbb{P}_{x_{a,n} \sim \hat \mu^{(n)}_{a} [\bar\rho_a]} \left(\|x_{a,n}-x_*\|_2 > \sqrt{\frac{36e}{m_a n} \left( D_a +2\Omega_a\log{1/\delta_1}+2\hat\Omega_a\log{1/\delta_2} \right)} \bigg| \mathcal Z_{n-1} \right)<\delta_2,
\end{equation*}
where $D_a = 2\log B_a + 8d $, $\Omega_a = 256 +\frac{16d L_a^2}{m_a\nu_a} $, $\hat\Omega_a =256 +\frac{16d L_a^2}{m_a\nu_a}+\frac{m_a d}{18L_a \rho_a}$, and the definition of $\mathcal Z_{n}$ can refer to Theorem \ref{theorem_empirical_concentrate1} and Remark \ref{remark_posterior}:
\begin{equation*}
    \mathcal Z_{n} =\left\{ \left\| x_{a,n}-x_{*}\right\|_2 \leq 3\sqrt{\frac{2e}{ m_a n} \left(D_a + 2\Omega_a \log1/\delta_1\right)} \right\}.\\ 
\end{equation*}
\end{theorem}

\begin{proof}
	Similar to the proof in Theorem \ref{theorem_empirical_concentrate1}, we first upper bound the distance between $\hat \mu_a^{(n)}[\bar\rho_a]$ and $\delta\left({x_*}\right)$:
\begin{equation}\label{eq_approx_posterior_t5}
  \begin{split}
       W_2 \left({\tilde \mu_a^{(n)}[\bar\rho_a], \delta\left({x_{*}}, {v_{*}}\right)}\right) & \stackrel{}{\leq} W_2 \left({\hat \mu_a^{(n)}[\bar\rho_a], {\hat\mu_{Ih^{(n)}}}}\right) + W_2 \left({\hat\mu_{Ih^{(n)}}}, \mu_a^{(n)}\right) + W_2 \left({\mu_a^{(n)}, \delta\left({x_{*}}, {v_{*}}\right)}\right) \\
       & \stackrel{(i)}{\leq} \sqrt{\frac{dp}{nL_a\bar\rho_a}} + W_2 \left({\hat\mu_{Ih^{(n)}}}, \mu_a^{(n)}\right) + W_2 \left({\mu_a^{(n)}, \delta\left({x_{*}}, {v_{*}}\right)}\right)\\
       & \stackrel{(ii)}{\leq} \sqrt{\frac{dp}{nL_a\bar\rho_a}} + \underbrace{W_2 \left({\hat\mu_{Ih^{(n)}}}, \mu_a^{(n)}\right)}_{T1} + \frac{\hat D_a}{\sqrt n},\\
  \end{split}
\end{equation}
where in $(i)$, the distance between $\hat \mu_a^{(n)}[\bar\rho_a]$ and ${\hat\mu_{Ih^{(n)}}}$ parallels the re-sampling mechanism at the last step of Langevin Monte Carlo, characterized by $\hat x_{a}-{\hat x_{Ih^{(n)}}}\sim \mathcal{N}\left(0,\frac{1}{nL_a\bar\rho_a} \mathbf{I}_{d\times d} \right)$. Similar to the upper bound $W_2 \left({\tilde \mu_a^{(n)}[\bar\rho_a], {\tilde\mu_{Ih^{(n)}}}}\right)\leq \sqrt{\frac{dp}{nL_a\bar\rho_a}}$ by applying Stirling's approximation for the Gamma function in Theorem \ref{theorem_ula_stochastic}, we establish the same upper bound for $W_2 \left({\hat \mu_a^{(n)}[\bar\rho_a], {\hat\mu_{Ih^{(n)}}}}\right)$. Additionally, $(ii)$ derives its result by incorporating the insights from Theorem \ref{theorem_ula_stochastic} and an extension of \eqref{eq_approx_posterior_t3}. Following the idea in \eqref{eq_approx_posterior_t2}, we can subsequently bound term $T1$ as follows:
\begin{equation}\label{eq_approx_posterior_t4}
	\begin{split}
		T1 & = W_2 \left({\hat\mu_{Ih^{(n)}}}, \mu_a^{(n)}\right) \\
		& \stackrel{}{\leq}  4e^{I h^{(n)}/2\kappa_a}W_2\left(\delta(x_{a, n-1}, v_{a, n-1}), \mu_a^{(n)}\right) + \frac{4}{1-e^{-{h^{(n)}}/ 2\kappa_a}}\left(5\left(h^{(n)}\right)^2\sqrt{\frac{d}{m_a}} + 4u{h^{(n)}} {L}_a\sqrt{\frac{nd}{k_a\nu_a}}\right) \\
		& \stackrel{}{\leq} 4e^{I h^{(n)}/2\kappa_a}\left[W_2\left(\delta(x_{a, n-1}, v_{a, n-1}), \delta(x_{*}, v_{*})\right) + W_2\left(\delta(x_{*}, v_{*}), \mu_a^{(n)}\right)\right] \\
            & \ \ \ \ + \frac{4}{1-e^{-{h^{(n)}}/ 2\kappa_a}}\left(5\left(h^{(n)}\right)^2\sqrt{\frac{d}{m_a}} + 4u{h^{(n)}} {L}_a\sqrt{\frac{nd}{k_a\nu_a}}\right)\\
		& \stackrel{}{\leq} 4e^{I h^{(n)}/2\kappa_a}\left[W_2\left(\delta(x_{a, n-1}, v_{a, n-1}), \mu_a^{(n-1)}\right) + W_2\left(\mu_a^{(n-1)}, \delta(x_{*}, v_{*})\right) + W_2\left(\delta(x_{*}, v_{*}), \mu_a^{(n)}\right)\right] \\
		& \ \ \ \ + \frac{4}{1-e^{-{h^{(n)}}/ 2\kappa_a}}\left(5\left(h^{(n)}\right)^2\sqrt{\frac{d}{m_a}} + 4u{h^{(n)}} {L}_a\sqrt{\frac{nd}{k_a\nu_a}}\right) \\
		& \stackrel{(i)}{\leq} 4e^{I h^{(n)}/2\kappa_a}\left[\frac{3}{\sqrt{n-1}}\hat D_a + \frac{\hat D'_a}{\sqrt{n}}\right] + \frac{4}{1-e^{-{h^{(n)}}/ 2\kappa_a}}\left(5\left(h^{(n)}\right)^2\sqrt{\frac{d}{m_a}} + 4u{h^{(n)}} {L}_a\sqrt{\frac{nd}{k_a\nu_a}}\right) \\
            & \stackrel{}{\leq} 4e^{I h^{(n)}/2\kappa_a}\left[\frac{3}{\sqrt{n-1}}\hat D_a + \frac{\hat D'_a}{\sqrt{n}}\right] + 80\kappa_a h^{(n)} \sqrt{\frac{d}{m_a}} + 64\kappa_a u{L}_a\sqrt{5 \frac{nd  }{k\nu_a} }\\
		& \stackrel{}{\leq}  \frac{2\bar D_a}{\sqrt n}.
	\end{split}
\end{equation}
For $(i)$ we define the following concentration rate for arm $a$ at round $n$ as: 
\begin{equation*}
	\begin{split}
		W_2\left(\delta(x_{*}, v_{*}), \mu_a^{(n)}\right) \leq \frac{\hat D'_a}{\sqrt n} = \sqrt{\frac{2}{m_an}\left(D_a + \hat\Omega_a p\right)}\\
		W_2\left(\mu_a^{(n-1)}, \delta(x_{*}, v_{*})\right)\leq \frac{1}{\sqrt{n-1}}\hat D_a \leq \frac{2}{\sqrt{n}}\hat D_a.
	\end{split}
\end{equation*}
We further define $\bar D_a = \sqrt{\frac{2}{m_a}\left(D_a + 2\Omega_a\log1/\delta_1+\hat\Omega_a p\right)}$ and derive the upper bound for $T1$. Applying this finding to \eqref{eq_approx_posterior_t5}, we derive the inequality:
\begin{equation*}
  \begin{split}
       W_2 \left({\hat \mu_a^{(n)}[\bar\rho_a], \delta\left({x_{*}}, {v_{*}}\right)}\right) \leq \sqrt{\frac{36}{m_an}\left(d+\log B_a + 2\Omega_a\log 1/\delta_1+\hat \Omega_a p\right)}.
  \end{split}
\end{equation*}
Following the idea in \eqref{eq_bound2}-\eqref{eq_posterior_derivation2} and with the selection of $p=2\log1/\delta_2$, we derive the approximate marginal posterior concentration inequality:
\begin{equation*}
 	\mathbb{P}_{x_{a,n} \sim \hat \mu^{(n)}_{a} [\bar\rho_a]} \left(\|x_{a,n} - x_*\|_2 > 6\sqrt{\frac{e}{m_a n} \left( D_a + 2\Omega_a\log{1/\delta_1} + 2\hat\Omega_a\log{1/\delta_2} \right)} \bigg| \mathcal Z_{n-1} \right)<\delta_2. 
\end{equation*}

\end{proof}
\subsection{Regret Analysis of Approximate Thompson Sampling}\label{section_regret_approx_ts}
Employing either the full gradient or stochastic gradient estimates, we can obtain a consistent posterior concentration rate as highlighted in Theorems \ref{theorem_empirical_concentrate1}-\ref{theorem_empirical_concentrate2}. Building on this, we extend our study to the regrets considering the approximate Thompson sampling algorithm. While the proof strategy for Theorem \ref{theorem_approximate_ts_regret} is similar to that of Theorem \ref{theorem_exact_regret}, the regret analysis becomes more intricate due to our choice of sampling strategy. This complexity emerges because the generated samples lose their conditional independence when the filtration starts with the previous sample. To address this, it requires an additional related lemma and we start with the definition of the following event:
\begin{equation*}
    \mathcal Z_{a}(N)=\bigcap_{n=1}^{N-1} \mathcal Z_{a,n},
\end{equation*}
where the definition of $\mathcal Z_{a,n}$ utilized in Theorems \ref{theorem_empirical_concentrate1}-\ref{theorem_empirical_concentrate2} adheres to the formulation initially given in Lemma \ref{theorem_empirical_concentrate1}:
\begin{equation*}
		 \mathcal Z_{a,n}=\left\{\left\|x_{a,n} - x_* \right\|_2 < 6\sqrt{\frac{e}{ m_a n} \left(D_a + 2\Omega_a \log1/\delta_1\right)} \right \},
\end{equation*}
where $D_a = 2\log B_a+8d$ and $\Omega_a = \frac{16d L_a^2}{m_a\nu_a}+256$.

\begin{lemma}\label{lemma_bound_supoptimal}
Given that the likelihood, rewards, and priors adhere to Assumptions \ref{assumption_prior}-\ref{assumption_sgld_lipschitz_main}, and taking into account the settings from Theorems \ref{theorem_ula_deterministic}-\ref{theorem_ula_stochastic} for step size, number of steps, and stochastic gradient estimates, the regret from Thompson sampling using approximate posterior sampling can be expressed as:
\begin{equation*}\label{eq_regret_decompose2}
\begin{split}
  \mathbb{E}[{\mathfrak R}(N)] \leq \sum_{a>1} \Delta_a \mathbb{E}\left[\mathcal L_a(N)  \Bigg| \mathcal Z_{a}(N) \cap \mathcal Z_{1}(N)\right] + 2\Delta_a
\end{split}
\end{equation*}

\end{lemma}

\begin{proof}
The derivation follows the same idea presented in \citet{mazumdar2020approximate}, and we revisit this proof herein for comprehensive exposition. The general idea is because of Lemma \ref{lemma_exact_smallT}, the probability of each event in $\mathcal Z_{a}(N)^c$ and $\mathcal Z_{1}(N)^c$ is less than $\Delta_a$. Then we can construct ``$\mathbb{P}(\mathcal Z_{a}(N)^c  \cup \mathcal Z_{1}(N)^c )$'' and thus we can upper bound this term through a function of $\Delta_a$.
\begin{equation*}
	\begin{split}
		\mathbb{E}\left[\mathcal L_a(N)\right] & = \mathbb{E}\left[\mathcal L_a(N)  \Bigg|  \mathcal Z_{a}(N) \cap \mathcal Z_{1}(N)\right]\mathtt P\left(\mathcal Z_{a}(N) \cap \mathcal Z_{1}(N)\right) \\
		& \ \ \ \ \ \ \ \ +\mathbb{E}\left[\mathcal L_a(N)  \Bigg|  \mathcal Z_{a}(N)^c \cup \mathcal Z_{1}(N)^c\right] \mathtt P\left(\mathcal Z_{a}(N)^c \cup \mathcal Z_{1}(N)^c\right) \\
		& \stackrel{(i)}{\leq}  \mathbb{E}\left[\mathcal L_a(N)  \Bigg|  \mathcal Z_{a}(N) \cap \mathcal Z_{1}(N)\right] +\mathbb{E}\left[\mathcal L_a(N)  \Bigg|  (\mathcal Z_{a}(N)^c \cup \mathcal Z_{1}(N)^c)\right]\mathtt P\left(\mathcal Z_{a}(N)^c \cup \mathcal Z_{1}(N)^c\right)\\
		&\stackrel{(ii)}{\leq} \mathbb{E}\left[\mathcal L_a(N)  \Bigg| \mathcal Z_{a}(N) \cap \mathcal Z_{1}(N)\right] +2N\delta_2 \mathbb{E}\left[\mathcal L_a(N)  \Bigg|  (\mathcal Z_{a}(N)^c \cup \mathcal Z_{1}(N)^c)\right]\\
		& \stackrel{(iii)}{\leq} \mathbb{E}\left[\mathcal L_a(N)  \Bigg| \mathcal Z_{a}(N) \cap \mathcal Z_{1}(N)\right] + 2 ,
	\end{split}
\end{equation*}
where $(i)$ use the fact that $1-\mathtt P\left(\mathcal Z_{a}(N) \cap \mathcal Z_{1}(N)\right)\leq 1$, and $(ii)$ holds because from Lemma \ref{theorem_empirical_concentrate1} we know that for $a\in\mathcal A$ we have:
\begin{equation*}
	\mathbb{P}\left(\mathcal Z_{a}(N)^c  \cup \mathcal Z_{1}(N)^c \right) \leq \mathbb{P}(\mathcal Z_{1}(N)^c )+\mathbb{P}\left(\mathcal Z_{a}(N)^c \right) = 2N\delta_2.
\end{equation*}
For part $(iii)$, we invoke a straightforward observation $\mathcal L_a(N) \leq N$ and the choice of $\delta_2 = \frac{1}{N^2}$. By summing $\mathbb{E}\left[\Delta_a\mathcal L_a(N)\right]$ over the range of the second arm ($a=2$) to the last arm ($a=\left|\mathcal A\right|$), the proof for Lemma \ref{lemma_bound_supoptimal} is thereby concluded.
\end{proof}

Employing a similar method delineated in Lemma \ref{lemma_exact_smallT}, we can extend anti-concentration guarantees to the approximate posterior distributions.

\begin{lemma}
\label{lemma_approximate_smallT}
Suppose the likelihood, rewards, and priors satisfy Assumptions \ref{assumption_reward}-\ref{assumption_sgld_lipschitz_main}. We follow the sampling schemes given in Theorems \ref{theorem_ula_deterministic}-\ref{theorem_ula_stochastic}, where the step size is $h = \tilde O\left(\frac{1}{\sqrt{d}}\right)$, number of steps $I = \tilde O\left(\sqrt{d}\right)$, and batch size $k=\tilde O\left(\kappa_a^2\right)$. With the choice of letting $\bar\rho_1=\frac{m_1}{8L_1\Omega_1}$, it follows that the underdamped Langevin Monte Carlo yields approximate distributions that are capped by the following upper bound guarantee for each round $n\in\llbracket 1, N \rrbracket$:
\begin{equation*}
	\mathbb{E}\left[ \frac{1}{{\hat {\mathcal G}}_{1,n}} \right] \leq 33\sqrt{B_1}.
\end{equation*}
\end{lemma}
\begin{remark}
Following the definition of ${\mathcal G}$ given in Section \ref{section_regret_general_ts}, here we denote $\hat {\mathcal G} = \mathtt P\left(\hat {\mathcal E}_a(n)|\mathcal F_{n-1}\right)$, where the event $\hat {\mathcal E}_a(n)=\left\{\hat{\mathcal R}_{a, n}\geq \bar{{\mathcal R}}_1-\epsilon\right\}$ indicates the estimated reward of arm $a$ at round $n$ (the reward is estimated by employing approximate Thompson sampling with underdamped Langevin Monte Carlo) exceeds the expected reward of optimal arm by at least a positive constant $\epsilon$.
\end{remark}
\begin{proof}
For any rounds $n$ encompassed within $\llbracket 1, N \rrbracket$, our initial analysis focuses on how samples generated from Algorithm \ref{algorithm_thompson} exhibit the requisite anti-concentration characteristics. Specifically, with $x_{1,n}$ following a normal distribution given by $x_{1,n}\sim \mathcal{N}(x_{1,Ih},\frac{1}{nL_1\bar\rho_1 }\mathbf{I}_{d \times d})$, we can infer a corresponding lower bound for $\hat {\mathcal G}_{1,n}$:

\begin{equation*}\begin{split}
\hat {\mathcal G}_{1,n}&={\mathtt P}\left(\langle \alpha_1, x_{1, n} - x_{1,Ih} \rangle \ge \langle \alpha_1, x_*-x_{1,Ih} \rangle -\mathbb \epsilon \right)\\
&\ge {\mathtt P}\left  (Z\ge \underbrace{\omega_1\|x_{1,Ih}-x_*\|_2}_{:=t}\right),
\end{split}\end{equation*}
where $Z \sim \mathcal{N}(0,\frac{\omega_1^2}{nL_1\bar\rho_1 }\mathbf{I}_{d \times d})$ by construction. Building upon the analysis presented in Lemma \ref{lemma_exact_smallT}, and considering $\sigma^2=\frac{\omega_1^2}{nL_1\bar\rho_1 }$, we arrive that the cumulative density function $\hat {\mathcal G}_{1,n}$ is bounded by:
\begin{equation*}
\hat {\mathcal G}_{1,n}\ge \sqrt{\frac{1}{2 \pi}}  \left\{
  \begin{aligned}
    \frac{\sigma t}{t^2+\sigma^2} e^{-\frac{t^2}{2\sigma^2}},\ \ \ \ \ \ \ \ & t>\frac{\omega_1}{\sqrt{nL_1\bar\rho_1 }}, \\ 
    0.34,  \ \ \ \ \ \ \ \ & t\leq \frac{\omega_1}{\sqrt{nL_1\bar\rho_1 }}.
  \end{aligned}\right.
\end{equation*}
We subsequently derive an upper bound for $\frac{1}{\hat {\mathcal G}_{1,n}}$:
\begin{equation*}
  \frac{1}{\hat {\mathcal G}_{1,n}} \leq \sqrt{2 \pi} \left\{
  \begin{aligned} 
    \left(\frac{t}{ \sigma}+1\right) e^{\frac{t^2}{2\sigma^2}},\ \ \ \ \ \ \ \ & t> \frac{\omega_1}{\sqrt{nL_1\bar\rho_1 }}, \\ 
    3,  \ \ \ \ \ \ \ \ & t\leq \frac{\omega_1}{\sqrt{nL_1\bar\rho_1 }} .
  \end{aligned} \right.
\end{equation*}

By setting the parameters in $\omega_1\|x_{1,Ih}-x_*\|_2$ as $\sigma=\frac{\omega_1}{\sqrt{nL_1\bar\rho_1 }}$, and subsequently taking the expectation relative to rewards ${\mathcal R}_1, {\mathcal R}_2, \dots, {\mathcal R}_n$, we can infer that:
\begin{equation*}
  \begin{split}
      \mathbb{E}\left[\frac{1}{\hat {\mathcal G}_{1,n}}\right]&\leq 3\sqrt{2\pi}+\sqrt{2\pi}\mathbb{E}\left[ \left(\sqrt{nL_1\rho_1 }\|x_{1,Ih}-x_*\|_2+1\right) e^{nL_1\rho_1  \|x_{1,Ih}-x_*\|_2^2} \right]\\
      & = 3\sqrt{2\pi} + \sqrt{2\pi}\mathbb{E}\left[e^{{nL_1\rho_1 }\|x_{1,Ih}-x_*\|_2^2} \right] \\
      &\ \ \ \ \ \ \ \ +\sqrt{2\pi nL_1\rho_1 }\sqrt{\mathbb{E}\left[\|x_{1,Ih}-x_*\|_2^2\right]}{\mathbb{E}\left[e^{nL_1\rho_1  \|x_{1,Ih}-x_*\|_2^2}\right]} \\
      & \stackrel{(i)}{\leq} 3\sqrt{2\pi} + {\sqrt{2\pi} \left(9\sqrt{\frac{ L_1\rho_1 }{2m_1}\left(D_1+2\Omega_1\right)} + \frac{3}{2}\right) \left(e^{\frac{4L_1 D_1}{m_1}\rho_1 } + 1\right)} \\
      & \stackrel{(ii)}{=} 3\sqrt{2\pi} + \sqrt{2\pi} \left(9\sqrt{\frac{ L_1\rho_1 }{m_1}\left(4d + \log B_1+\frac{16L_1^2d}{m_1\nu_1} + 256\right)} + \frac{3}{2}\right) \left(e^{\frac{4L_1}{m_1}\left(8d_1 + 2\log B_1\right)\rho_1 } + 1\right) \\
      & \stackrel{(iii)}{\leq} 3\sqrt{2\pi} + \frac{3}{2} \sqrt{2\pi} \left(\frac{3}{2}\sqrt{\frac{1}{136}\log B_1 + \frac{75}{34}} + 1\right) \left(e^{\frac{1}{68}+\frac{1}{272}\log B_1} + 1\right) \\
      & \stackrel{(iv)}{\leq} 3\sqrt{2\pi} + 25 \sqrt{B_1} \\
      & \stackrel{(v)}{\leq} 33 \sqrt{B_1}\\
  \end{split}
\end{equation*}

where $(i)$ is from Lemma \ref{eq_mgf_bound}, $(ii)$ is by having $D_1=8d+2\log B_1$ and $\Omega_1=\frac{16L_1^2d}{m_1\nu_1}+256$. $(iii)$ proceeds by the same selection of $\bar\rho_1   \leq \frac{m_1}{8L_1\Omega_1}$ in Lemma \ref{eq_mgf_bound} and because without loss of generality we know that $ \frac{m_1}{L_1}\leq1$, $ \frac{m_1^2\nu_1}{L_1^3}\leq1$ and $\frac{1}{d}\leq 1$ hold, we thus have $\bar\rho_1  \leq \frac{m_1}{8L_1\Omega_1}\leq\frac{m_1}{2176L_1d}\leq\frac{m_1}{2176L_1}$. $(iv)$ holds by $\frac{3\sqrt{2\pi}}{2} \left(\frac{3}{2}\sqrt{\frac{1}{136}\log x + \frac{75}{34}} + 1\right) \left(e^{\frac{1}{68}}x^{\frac{1}{272}} + 1\right) < 25 \sqrt{x}$ for $x \geq 1$, and finally $(v)$ is because of $3\sqrt{2\pi} + 25\sqrt{ x} < 33 \sqrt{x}$ for $x \geq 1$. 
\end{proof}

With the insights gained from Lemma \ref{lemma_exact_ts2}, we can extend it to finalize the proof of Lemma \ref{lemma_approximate_ts2} and then Theorem \ref{theorem_approximate_ts_regret}.

\begin{lemma}
\label{lemma_approximate_ts2}
   Suppose the likelihood, rewards, and priors satisfy Assumptions \ref{assumption_reward}-\ref{assumption_sgld_lipschitz_main} and given that the samples are drawn according to the sampling methods delineated in Theorems \ref{theorem_ula_deterministic} and \ref{theorem_ula_stochastic} $\left(h = \tilde O\left(\frac{1}{\sqrt{d}}\right),\ I = \tilde O\left(\sqrt{d}\right),\text{ and }k=\tilde O\left(\kappa_a^2\right)\right)$. With the choice of $\bar\rho_a=\frac{m_a}{8L_a \Omega_a}$, the following inequalities for approximate probability concentrations holds true:
   \begin{equation}\label{eq_approx_regret_1}
   \sum_{s=0}^{N-1}\mathbb{E}\left[ \frac{1}{\hat {\mathcal G}_{1,s}}-1 \bigg| \mathcal Z_{1}(N)\right] \leq  33\sqrt{B_1} \left\lceil \frac{144e\omega_1^2}{m_1\Delta_a^2}(8d + 2\log B_1 + 4\Omega_1\log N+3d\Omega_1 ) \right\rceil +1
   \end{equation}
   \begin{equation}
   \label{eq_approx_regret_2}
   \sum_{s=0}^{N-1} \mathbb{E}\left[ \mathbb{I}\left(\hat {\mathcal G}_{a,s}>\frac{1}{N}\right)\bigg| \mathcal Z_{a}(N)\right]  \leq \frac{144 e \omega_a^2}{m_a\Delta_a^2}\left(8d+2\log B_a+7d\Omega_a\log N \right),
   \end{equation}
where the approximate posterior $\hat \mu_a$ yields a distribution, $\hat {\mathcal G}_{a,s}$, after gathering $s$ samples. Additionally, for arm $a \in \mathcal{A}$, and the relation $\Omega_a = \frac{16d L_a^2}{m_a\nu_a}+256$ holds.
\end{lemma}

\begin{proof}

According to the definition of ${\mathcal G}_{a,s}$ and the derivation in \eqref{eq_pas_lower}, we can derive a similar lower bound for $\hat {\mathcal G}_{a,s}$ as follows:
\begin{equation}\label{eq_hatpas_lower}
\begin{split}
  \hat {\mathcal G}_{1,s}&=\mathbb{P}({{\mathcal R}}_{1,s}>\bar{{\mathcal R}}_1-\mathbb \epsilon |\mathcal{F}_{n-1})\\
  &=1-\mathbb{P}({{\mathcal R}}_{1,s}-\bar{{\mathcal R}}_1<-\mathbb \epsilon |\mathcal{F}_{n-1})\\
  &\ge 1-\mathbb{P}(|{{\mathcal R}}_{1,s}-\bar{{\mathcal R}}_1|>\mathbb \epsilon |\mathcal{F}_{n-1})\\
  &\ge 1-\mathbb{P}_{{x} \sim \hat \mu^{(s)}_{1}}\left(\|{x} -x_*\|> \frac{\mathbb \epsilon}{\omega_1} \right).
\end{split}
\end{equation}
Then the result given in \eqref{eq_approx_regret_1} can be derived following the steps given below:
\begin{equation*}
	\begin{split}
	    \sum_{s=0}^{N-1}\mathbb{E}\left[ \frac{1}{{\hat {\mathcal G}}_{1,s}}-1 \bigg| \mathcal Z_{1}(N)\right]
	    & = \sum_{s=0}^{\mathbb \ell-1}\mathbb{E}\left[ \frac{1}{{\hat {\mathcal G}}_{1,s}}-1 \bigg| \mathcal Z_{1}(N)\right] + \sum_{s=\ell}^{N-1}\mathbb{E}\left[ \frac{1}{{\hat {\mathcal G}}_{1,s}}-1 \bigg| \mathcal Z_{1}(N)\right]\\
	    & \stackrel{(i)}{\leq} \sum_{s=0}^{\mathbb \ell-1}\mathbb{E}\left[ \frac{1}{\hat {\mathcal G}_{1,s}}-1\bigg| \mathcal Z_{1}(N) \right] + \sum_{s=\mathbb \ell}^{N-1} \left[ \frac{1}{1-\mathbb{P}_{x \sim \hat\mu^{(s)}_{1}}\left(\|x - x_*\| \geq {\epsilon}/{\omega_1} \right)} -1 \right]\\
	    & \stackrel{}{\leq} \sum_{s=0}^{\mathbb \ell-1}\mathbb{E}\left[ \frac{1}{\hat {\mathcal G}_{1,s}}-1\bigg| \mathcal Z_{1}(N) \right] + \int_{s=\mathbb \ell}^{\infty} \left[\frac{1}{1-\mathbb{P}_{x \sim \hat\mu^{(s)}_{1}}\left(\|x - x_*\| \geq {\epsilon}/{\omega_1} \right)}-1\right] ds \\
	    & \stackrel{(ii)}{\leq} \sum_{s=0}^{\mathbb \ell-1}\mathbb{E}\left[ \frac{1}{{\hat {\mathcal G}}_{1,s}}-1 \bigg| \mathcal Z_{1}(N) \right] + \frac{144e\omega_1^2}{m_1\Delta_a^2}3d\Omega_1 \log{2} +1 \\
	    & \stackrel{(iii)}{\leq} 33\sqrt{B_1} \left\lceil \frac{144e\omega_1^2}{m_1\Delta_a^2}(8d + 2\log B_1 + 4\Omega_1\log N+3d\Omega_1 ) \right\rceil +1,\\
	\end{split}
\end{equation*}
where $(i)$ is from \eqref{eq_hatpas_lower}, and $(ii)$ is because from Lemmas \ref{theorem_empirical_concentrate1} and \ref{theorem_empirical_concentrate2}, with the selection of $\delta_1 = \frac{1}{Te^2}$ and $\bar\rho_a=\frac{m_a}{8L_a \Omega_a}$ we have:
\begin{equation}
	\mathbb{P}_{x \sim \tilde \mu^{(n)}_{a} [\bar \rho_a]} \left(\|x - x_*\|_2 > 6\sqrt{\frac{e}{m_1 n} \left( D_1 + 4\Omega_1\log N + 3d\Omega_1\log{1/\delta_2} \right)} \bigg| \mathcal Z_{n-1} \right)<\delta_2.
\end{equation}
Equivalently, we can have:
\begin{equation}\label{eq_concentration_approx}
\begin{split}
 \mathbb{P}_{{x} \sim \bar \mu^{(s)}_{1}[\rho_1]}\left(\|{x} -x_*\|> \frac{\mathbb \epsilon}{\omega_1} \right) \leq \exp\left(-\frac{1}{3d\Omega_1}\left( \frac{m_1 n \mathbb \epsilon^2}{36e \omega_1^2} - 8d - 2\log B_1 - 4\Omega_1\log N\right) \right),
\end{split}\end{equation}
which holds true when $n\geq\frac{36e\omega_1^2}{\mathbb \epsilon^2m_1} \left( 8d + 2\log B_1 + 4\Omega_1\log N\right)$. If we select the following $\mathbb \ell$:
\begin{equation*}
	\mathbb \ell=\left\lceil \frac{144 e\omega_1^2}{m_1\Delta_a^2}(8d + 2\log B_1 + 4\Omega_1\log N+3d\Omega_1 \log{2} ) \right\rceil
\end{equation*}
and let $\epsilon = \Delta_a / 2$, following the same idea in \eqref{eq_g1s_inverse1} we can get $(ii)$. Then $(iii)$ proceeds by the selection of $\ell$ and Lemma \ref{lemma_approximate_smallT}.

To show that \eqref{eq_approx_regret_2} holds, we have:
\begin{equation*}\begin{split}
  \sum_{n=1}^N \mathbb{E}\left[ \mathbb{I}\left({\hat {\mathcal G}}_{a,s}>\frac{1}{N}\right)\bigg| \mathcal Z_{1}(N)\right] &= \sum_{n=1}^N \mathbb{E}\left[ \mathbb{I}\left(\mathbb{P}\left({{\mathcal R}}_{a,s}-\bar{{\mathcal R}}_a>\Delta_a-\epsilon |\mathcal{F}_{n-1}\right) >\frac{1}{N}\right)\bigg| \mathcal Z_{a}(N)\right] \\
  & \stackrel{(i)}{\leq} \sum_{n=1}^N \mathbb{E}\left[ \mathbb{I}\left(\mathbb{P}_{x \sim \mu^{(s)}_{a}[\rho_a]}\left(\|x - x_*\|>\frac{\Delta_a}{2\omega_a}\right) >\frac{1}{N}\right)\bigg| \mathcal Z_{a}(N)\right] \\
   & \stackrel{(ii)}{\leq} \frac{144 e \omega_a^2}{m_a\Delta_a^2}\left(8d+2\log B_a+10d\Omega_a\log N \right),
\end{split}\end{equation*}
where $(i)$ is from our previous defined $\epsilon=\Delta_a/2$, and $(ii)$ proceeds by inducing \eqref{eq_concentration_approx} with the choice of 
\begin{equation*}
n\ge \left\lceil \frac{8e\omega_1^2}{m\Delta_a^2}(8d + 2\log B_a + 4\Omega_a\log N+3d\Omega_a \log N ) \right\rceil,
\end{equation*}
which completes the proof.
\end{proof}

Combining the conclusions drawn from Lemma \ref{lemma_approximate_smallT} and Lemma \ref{lemma_approximate_ts2}, we present our concluding theorem:
\begin{theorem}[Regret of approximate Thompson sampling with (stochastic gradient) underdamped Langevin Monte Carlo]
\label{theorem_approximate_ts_regret}
When the likelihood, rewards, and priors satisfy Assumptions \ref{assumption_reward}-\ref{assumption_sgld_lipschitz_main}, scaling parameter $\bar\rho_a=\frac{m_a}{8L_a\Omega_a}$, and the sampling schemes of Thompson sampling follow the setting in Theorems \ref{theorem_ula_deterministic}-\ref{theorem_ula_stochastic}, namely, the step size is $h = \tilde O\left(\frac{1}{\sqrt{d}}\right)$, number of steps $I = \tilde O\left(\sqrt{d}\right)$, and batch size $k=\tilde O\left(\kappa_a^2\right)$. We then have that the total expected regrets after $N>0$ rounds of Thompson sampling with the (stochastic gradient) underdamped Langevin Monte Carlo method satisfy:
\begin{equation*}\begin{split}
      \mathbb{E}[\mathfrak R(N)] & \leq \sum_{a>1} \frac{\hat C_a }{\Delta_a}\left( \log B_a + d + d^2\kappa_a^2\log N\right)\\ &\qquad \qquad  +\frac{\hat C_1 \sqrt{B_1}}{\Delta_a}\left( \log B_1 + d^2\kappa_1^2 + d\kappa_1\log N\right) +4\Delta_a.
\end{split}\end{equation*}    
  where $\hat C_a, \hat C_1 >0$ are universal constants that are independent of problem-dependent parameters and $\kappa_a=L_a/m_a$.
\end{theorem}
\begin{proof}
Our approach to proving Theorem \ref{theorem_approximate_ts_regret} begins by taking the expected number of arm $a$ pulls over $N$ rounds, and multiplying it with the distance in expected rewards from the optimal and the chosen arm $a$. It is important to highlight that the expected number of arm $a$ selections over $N$ rounds is conditional upon $\mathcal Z_{1}(N) \cap \mathcal Z_{a}(N)$. This condition implies the concentration of the log-likelihood and the inclination of the approximate samples to be predominantly located in the high-probability areas of the posteriors. This upper bound is further detailed using the findings from Lemma \ref{lemma_approximate_ts2}. The detailed proof is given below:

\begin{equation*}
\begin{split}
      \mathbb{E}[\mathfrak R(N)] & \stackrel{(i)}{\leq} \sum_{a>1} \Delta_a \mathbb{E}\left[\mathcal L_a(N)  \Bigg| \mathcal Z_{a}(N) \cap \mathcal Z_{1}(N)\right] + 2\Delta_a \\
      & \stackrel{(ii)}{\leq} \sum_{a>1}\left\{\Delta_a\sum_{s=0}^{N-1}\mathbb{E}\left[ \frac{1}{{\mathcal G}_{1,s}}-1 \bigg|  \mathcal Z_{1}(N)\right] + \Delta_a\sum_{s=0}^{N-1} \mathbb{E}\left[ \mathbb{I}\left(1-{\mathcal G}_{a,s}>\frac{1}{N}\right)\bigg|  \mathcal Z_{a}(N)\right] + 3\Delta_a\right\} \\
      & \stackrel{(iii)}{\leq} \sum_{a>1} 33\sqrt{B_1}\left\lceil \frac{144e\omega_1^2}{m_1\Delta_a}\left( 8d+2\log B_1+4\Omega_1 \log N +3d\Omega_1 \right) \right\rceil \\
      & \ \ \ \ \ \ \ \ + \sum_{a>1}\frac{144 e \omega_a^2}{m_a\Delta_a}\left(8d + 2\log B_a+7d\Omega_a\log N \right) + 4\sum_{a>1}\Delta_a \\
      & \leq \sum_{a>1} \left[\frac{\hat C_{a}}{\Delta_a}\left(\log B_a + d + d^2\kappa_a^2\log N\right) + \frac{\hat C_1}{\Delta_a}\sqrt{B_1}\left(\log B_1 + d^2\kappa_1^2 + d\kappa_1\log N\right) + 4\Delta_a \right]. \\
\end{split}
\end{equation*} 
where $(i)$ is drawn from Lemma \ref{lemma_bound_supoptimal}, $(ii)$ is attributed to Lemmas \ref{lemma_suboptimal1} and \ref{lemma_suboptimal2}, and $(iii)$ can derived by applying Lemma \ref{lemma_approximate_ts2}. By selecting constants $\hat C_1$ and $\hat C_a$ to upper bound some problem-independent constants, we arrive at the final regret bound.

\end{proof} %

\subsection{Supporting Proofs for Approximate Thompson Sampling}
The following section provides additional theoretical support for the analysis of Langevin Monte Carlo and approximate Thompson sampling.
\subsubsection*{Continuous Time Analysis}

\begin{theorem}[Convergence for the Continuous-Time Process]\label{theorem_diffusion}
Suppose $(x_t, v_t)\sim\mu_t$ follows Langevin dynamics described by \eqref{eq_diffusion2}, $\mu_*$ the posterior distribution described by $(x_*, v_*)$. Then the following inequality holds:
\begin{equation}\label{eq_contration_diffusion}
    W_2(\mu_t, \mu_*) \leq 2e^{- t/2\kappa} W_2(\mu_0, \mu_*).
\end{equation}
\end{theorem}

\begin{proof}
We start by defining $\eta_t$ to be the distribution of $(x_t, x_t+v_t)$, $\eta_*$ the posterior distribution of $(x'_t, x'_t+v'_t)$, and let $\zeta_t\in \Gamma(\eta_t,\eta_*)$ be the optimal coupling between $\eta_t$ and $\eta_*$ such that $\mathbb{E}_{\zeta_0}\left[\left\| x_0 - x'_0 \right\|_2^2 + \| x_0 - x'_0 + v_0 -v'_0 \|_2^2 \right] = W_2^2(\eta_0,\eta_*)$. Then we have:
\begin{equation*}
  \begin{split}
      W_2^2\left(\mu_t, \mu_*\right) & \stackrel{(i)}{\leq} W_2^2\left(\eta_t, \eta_*\right) \\
      & \stackrel{(ii)}{\leq} {\mathbb{E}_{\zeta_0}\left[\mathbb{E}_{ \zeta_t}\left[\left\|x_t-x'_t\right\|_2^2+\left\|x_t-x'_t+v_t-v'_t\right\|_2^2 | x_0, x'_0, v_0, v'_0\right]\right]} \\
      & \stackrel{(iii)}{\leq} {\mathbb{E}_{\zeta_0}\left[e^{-t / \kappa}\left(\left\|x_0-x'_0\right\|_2^2+\left\|x_0-x'_0+v_0-v'_0\right\|_2^2\right)\right]} \\
      & \stackrel{(iv)}{=} e^{-t / \kappa} W_2^2\left(\eta_0, \eta_*\right)\\
      &  \stackrel{(v)}{=} 4e^{-t / \kappa} W_2^2\left(\mu_0, \mu_*\right),\\
  \end{split}
\end{equation*}
where $(i)$ and $(v)$ proceed by Lemma \ref{lemma_sandwich}. $(ii)$ is a direct consequence of the definition of 2-Wasserstein distance, which is based on optimal coupling, and the tower property of expectation. Subsequently, $(iii)$ is derived from Lemma \ref{lemma_contraction2}, while $(iv)$ is from the specific choice of $\zeta_0$ as the optimal coupling. Taking square roots from both sides completes the proof.
\end{proof}

\begin{lemma}\label{lemma_contraction2}
Let diffusion $(x_t, x_t+v_t)\sim \eta_t$ and posterior distribution $(x'_t, x'_t+v'_t)\sim \eta_*$ are both in $\mathbb R^{2d}$ and follow the dynamics dictated by SDEs \eqref{eq_sde}. Suppose $f(\cdot)$ fulfills Assumption \ref{assumption_lipschitz_global_main}, then the following inequality holds:
\begin{equation}\label{eq_contration_diffusion2}
    {\mathbb{E}\left[\left\|x_t-x'_t\right\|_2^2+\left\|x_t-x'_t+v_t-v'_t\right\|_2^2\right]}  {\leq} e^{-t / \kappa}{\mathbb{E}\left[\left(\left\|x_0-x'_0\right\|_2^2+\left\|x_0-x'_0+v_0-v'_0\right\|_2^2\right)\right]} \\
\end{equation}
\end{lemma}

\begin{proof}
This proof substantially relies on the results of Lemma \ref{lemma_contraction}, and we consequently exclude some repetitive parts. The proof starts with the derivation of an upper bound for the time-dependent derivative of $\left\|x_t - x'_t\right\|^2_2 + \left\|(x_t+v_t) - (x'_t+v'_t)\right\|^2_2$ w.r.t. time $t$:

\begin{equation*}
  \begin{split}
      &\frac{d}{dt}\left(\left\|x_t - x'_t\right\|^2_2 + \left\|(x_t+v_t) - (x'_t+v'_t)\right\|^2_2\right) \\
      \stackrel{(i)}{=} &  -2\left\langle x_t-x'_t, v_t-v'_t\right\rangle + 2\left\langle (x_t-x'_t)+(v_t-v'_t), \left(\gamma-1\right)\left(v_t-v'_t\right)+u\left(\nabla f(x_t)-\nabla f(x'_t)\right)\right\rangle \\
      \stackrel{(ii)}{\leq} & -\frac{m}{L}\left(\left\|x_t - x'_t\right\|^2_2 + \left\|(x_t+v_t) - (x'_t+v'_t)\right\|^2_2\right) \\
  \end{split}
\end{equation*}
where $(i)$ proceed by taking derivatives of $x_t$, $x'_t$, $v_t$, and $v'_t$ w.r.t. time, $(ii)$ applies the conclusion draw from Lemma \ref{lemma_contraction}. The convergence rate in Lemma \ref{lemma_contraction2} can be derived as a direct consequence of Gr{\"o}nwall's inequality, as detailed in Corollary 3 from \citet{dragomir2003some}.
\end{proof}

\subsubsection*{Discretization Analysis}
We turn our attention to a single step within the underdamped Langevin diffusion in a discrete version, characterized by the following stochastic differential equations:
\begin{equation}\label{eq_langevin_discrete}
  \begin{split}
  d \tilde{v}_t & =-\gamma \tilde{v}_t d t-u \nabla f\left(\tilde{x}_0\right) d t + \left(\sqrt{2 \gamma u}\right) d B_t \\
  d \tilde{x}_t & =\tilde{v}_t d t,
  \end{split}
\end{equation}
where $t\in\left[0, {h}\right)$ for some small ${h}$, and ${h}$ represents a single step in the underdamped Langevin Monte Carlo. The solution $(\tilde{x}_t,\tilde{v}_t)$ of \eqref{eq_langevin_discrete} can be expressed as follows
\begin{equation}\label{eq_langevin_discrete_sol}
  \begin{split}
  \tilde v_t &= \tilde v_0 e^{-\gamma t} - u \left(\int_0^t e^{-\gamma(t-s)} \tilde{\nabla} f( \tilde x_0) ds \right) + \sqrt{2\gamma u} \int_0^t e^{-\gamma (t-s)} dB_s\\
  \tilde x_t &= \tilde x_0 + \int_0^t  v_s ds.
  \end{split}
\end{equation}
To get the first term in \eqref{eq_langevin_discrete_sol}, one may multiply the velocity term in \eqref{eq_langevin_discrete} by $e^{-\gamma t}$ on both sides and then compute the definite integral from 0 to $t$. Then the second term in \eqref{eq_langevin_discrete_sol} can be derived by taking the definite integral of the position term. We skip the detailed derivation here. 

Our focus in the following lemma is to quantify the discretization error between the continuous and discrete processes, as represented by \eqref{eq_sde} and \eqref{eq_langevin_discrete} respectively, both initiating from identical distributions. Specifically, our bound pertains to $W_2 (\mu_t ,\tilde{\mu}_t )$ over the interval $t\in \left[i{h}, (i+1){h}\right)$.  While $\mu_t$ is a continuous-time diffusion detailed in Theorem \ref{theorem_diffusion}, $\tilde{\mu}_t$ is a measure corresponding to \eqref{eq_langevin_discrete}.Our findings here illuminate the convergence rate defined in Lemma \ref{lemma_langevin_discrete}.

Throughout this analysis, we work under the assumption that the kinetic energy, representing the second moment of velocity in a continuous-time process, adheres to the boundary defined as follows: $\mathbb{E}_{\mu_t} {\|v_t\|_2^2}\leq 26(d/m +  \mathcal{D}^2)$. A comprehensive explanation of this bound can be found in Lemma \ref{lemma_kinetic_bound}. Following this, we provide a detailed exploration of our discretization analysis findings.
\begin{lemma}[Discretization Analysis of underdamped Langevin Monte Carlo with full gradient]\label{lemma_langevin_discrete}
Consider $\left(x_t, v_t\right)$ distributed as $\mu_t$ and $\left(\tilde x_t, \tilde v_t\right)$ as $\tilde{\mu}_t$, representing the measures from SDEs \eqref{eq_sde} and \eqref{eq_langevin_discrete}, respectively.
 Given an initial distribution $\mu_0$ and a step size $h<1$, we select $u = 1/L$ and $\gamma = 2$. Consequently, the difference between continuous-time and discrete-time processes within:
\begin{equation}\label{eq_langevin_discrete2}
  W_2\left(\mu_t,\tilde{\mu}_t\right) \leq 2h^2\sqrt{\frac{26}{5}\frac{d}{m}}.
\end{equation}

\end{lemma}

\begin{proof}
Considering the following synchronous coupling $\zeta_t\in \Gamma(\mu_t, \tilde{\mu}_t)$, where $\mu_t$ and $\tilde{\mu}_t$ are both anchored to the initial distribution $\mu_0$ and are governed by the same Brownian motion, $B_t$.  Our focus initially settles on bounding the velocity deviation, and referencing the descriptions of $v_t$ and $\tilde{v}_t$ in \eqref{eq_langevin_discrete}, we deduce that:
\begin{equation}\label{eq_bound_velocity}
  \begin{split}
  \mathbb{E}_{\zeta_t}{\left[\left\|{v_t -\tilde{v}_t }\right\|_2^2 \right]} & \overset{(i)}{=} \mathbb{E} \left[ \left\| u \int_0^s e^{-2(t-s)}\left( \nabla f(x_s) - \nabla f(x_0) \right)ds\right\|_2^2\right]\\
  & = u^2 \mathbb{E} \left[ \left\|  \int_0^s e^{-2(t-s)}\left( \nabla f(x_s) - \nabla f(x_0) ds\right)\right\|_2^2\right]\\
  & \overset{(ii)}{\le} t u^2 \int_0^s \mathbb{E} \left[ \left\| e^{-2(t-s)}\left( \nabla f(x_s) - \nabla f(x_0) \right)\right\|_2^2\right]ds\\
  & \overset{(iii)}{\le} t u^2 \int_0^s \mathbb{E} \left[ \left\| \left( \nabla f(x_s) - \nabla f(x_0) \right)\right\|_2^2\right]ds\\
  & \overset{(iv)}{\le} t u^2 L^2 \int_0^s \mathbb{E} \left[\left\| x_s - x_0 \right \|_2^2\right]ds\\
  & \overset{(v)}{=} t u^2 L^2 \int_0^s  \mathbb{E} \left[\left\| \int_0^r v_w dw \right \|_2^2\right]ds\\
  & \overset{(vi)}{\le} t u^2 L^2\int_0^s r \left(\int_0^r \mathbb{E}\left[ \left\| v_w \right\|_2^2\right] dw\right)ds\\
  & \overset{(vii)}{\le} 26t u^2 L^2 \left(\frac{d}{m}+  \mathcal{D}^2\right) \int_0^s r \left(\int_0^r dw\right) ds\\
  & = \frac{26t^4 u^2 L^2}{3}\left(\frac{d}{m}+  \mathcal{D}^2\right),
  \end{split}
\end{equation}
where $(i)$ follows from \eqref{eq_langevin_discrete} and $v_0=\tilde{v}_0$, $(ii)$ proceeds by an application of Jensen's inequality, $(iii)$ follows as $\left\| e^{-4(t-s)}\right\| \leq 1$, $(iv)$ is by Lipschitz smooth property of $f(x)$, $(v)$ arises from the explicit definition of $x_r$. The inferences in $(vi)$ and $(vii)$ stem from Jensen's inequality and the uniform kinetic energy bound, where an in-depth exploration of this bound is offered in Lemma \ref{lemma_kinetic_bound}. 

Having delineated the bounds for the velocity aspect, our subsequent task is to determine the discretization error's bound within the position variable.
\begin{equation*}
  \begin{split}
      \mathbb E_{\zeta_t}{\left[\left\|{x_t - \hat x_t}\right\|_2^2\right]} & \stackrel{(i)}{=} \mathbb{E} {\left[\left\|{\int_0^t \left(v_s - \hat v_s\right) ds}\right\|_2^2\right]}\\
      & \stackrel{(ii)}{\leq} t\int_0^t \mathbb{E} \left[\left\| v_s - \tilde{v}_s \right\|_2^2\right]ds\\
      & \stackrel{(iii)}{\leq} t \int_0^t \left(\frac{d}{m}+  \mathcal{D}^2\right)\frac{26u^2L^2s^4}{3} ds\\
      & = \left(\frac{d}{m}+  \mathcal{D}^2\right) \frac{26 u^2 L^2 }{15} t^6, \\
  \end{split}
\end{equation*}
where $(i)$ relies on the coupling associated with the initial distribution $\mu_0$, $(ii)$ is by the following Jensen's inequality:
\begin{equation*}
	\left\|{\int_0^t v_s ds}\right\|_2^2  =  \left\|{\frac{1}{N}\int_0^t t \cdot v_s ds}\right\|_2^2 \leq t\int_0^t \left\|v_s\right\|_2^2 ds,
\end{equation*}
 $(iii)$ uses the derived bound in \eqref{eq_bound_velocity}. With the selection of $t=h$ and $u = \frac{1}{L}$, we identify an upper bound for the 2-Wasserstein distance between $\mu_h$ and $\tilde{\mu}_h$:
\begin{equation*}
	\begin{split}
		W_2(\mu_h ,\tilde{\mu}_h) & \leq \sqrt{26 \left(\frac{h^4}{3}+ \frac{h^6}{15}\right)\left(\frac{d}{m}+  \mathcal{D}^2\right)}\\
		& \stackrel{(i)}{\leq} h^2\sqrt{\frac{52}{5}\left(\frac{d}{m}+  \mathcal{D}^2\right)} \\
		& \stackrel{(ii)}{\leq} 2h^2\sqrt{\frac{26}{5}\frac{d}{m}},
	\end{split}
\end{equation*}

where $(i)$ holds because without loss of generality $h$ is assumed to be less than 1, and $(ii)$ proceeds by ${\mathcal D}^2=d/m$ from Lemma \ref{lemma_bound_position}, which completes the proof.

\end{proof}

With the convergence rate of the continuous-time SDE captured in \eqref{eq_contration_diffusion} and a discretization error bound delineated in \eqref{eq_langevin_discrete}, our next step is to formulate the following conclusion relative to underdamped Langevin Monte Carlo with full gradient.

\begin{theorem}[Convergence of the discrete underdamped Langevin Monte Carlo with full gradient]\label{theorem_ULA2}
Suppose $(x_t, v_t)\sim\mu_t$ and $(x_t, x_t+v_t)\sim \eta_t$ are corresponding to the measures of distribution described by SDEs \eqref{eq_sde}, $(\tilde x_t, \tilde v_t)\sim \tilde{\mu}_t$ and $(\tilde x_t, \tilde x_t + \tilde v_t)\sim \tilde \eta_t$ by SDEs \eqref{eq_langevin_discrete} respectively. Suppose $\mu_*$ and $\eta_*$ the posterior distribution described by $(x_*, v_*)$ and $(x_*, x_*+v_*)$. By defining $u=\frac{1}{L}$ and setting $\gamma=2$, we can establish an upper bound on the 2-Wasserstein distance differentiating the continuous-time from the discrete-time process as follows:
\begin{equation}\label{eq_langein_discretefinal}
	W_2(\mu_{i{h}}, \mu_*) \leq 4e^{-{i{h}}/2\kappa}W_2(\mu_{0}, \mu_*) + \frac{20{h}^2}{1-e^{-{h}/ 2\kappa}} \sqrt{\frac{d}{m}},
\end{equation}
where we denote $t\in\left[ih,(i+1)h\right)$, $i\in \llbracket 1, I \rrbracket$.
\end{theorem}

\begin{proof}
From Theorem \ref{theorem_diffusion} we have that for any $i \in \llbracket 1, I \rrbracket$:
\begin{equation*}
	W_2(\eta_{{i{h}}}, \eta_*)\leq e^{- {h}/2\kappa}W_2(\eta_{(i-1){h}}, \eta_*),
\end{equation*}
Through the application of the discretization error bound provided in Lemma \ref{lemma_langevin_discrete} and Lemma \ref{lemma_sandwich}, we can obtain:
\begin{equation}\label{eq_discrete_sandwich}
	W_2(\eta_{i{h}}, \tilde \eta_{i{h}})\leq 2W_2(\mu_{i{h}}, \tilde \mu_{i{h}})\leq 4h^2\sqrt{\frac{26}{5}\frac{d}{m}}.
\end{equation}

Invoking the triangle inequality yields the following results:

\begin{equation}\label{eq_single_discrete}
\begin{split}
W_2(\tilde\eta_{{i{h}}}, \eta_*) & \leq W_2(\eta_{{i{h}}}, \eta_*) + W_2(\eta_{{i{h}}}, \tilde \eta_{{i{h}}}) \\
& \leq e^{-{h}/2\kappa}W_2(\eta_{(i-1){h}}, \eta_*) + 4h^2\sqrt{\frac{26}{5}\frac{d}{m}}. \\
\end{split}\end{equation}

With the above results, we derive the upper bound for the distance between $\tilde\mu_{i{h}}$ and $\mu_*$ shown in \eqref{eq_langein_discretefinal}:
\begin{equation*}
	\begin{split}
		W_2(\tilde\mu_{i{h}}, \mu_*) & \stackrel{(i)}{\leq} 2W_2(\tilde\eta_{i{h}}, \eta_*) \\
		& \stackrel{(ii)}{\leq}  2e^{-{ih}/ 2\kappa}W_2(\eta_{0}, \eta_*)+ \left(1+ e^{-{h}/ 2\kappa} + \cdots
		+ e^{-{(i-1)h}/ 2\kappa} \right)8h^2\sqrt{\frac{26}{5}\frac{d}{m}}\\ 
		& \stackrel{(iii)}{\leq} 4e^{-{ih}/ 2\kappa}W_2(\mu_{0}, \mu_*) + \frac{8{h}^2}{1-e^{-{h}/ 2\kappa}}\sqrt{\frac{26}{5}\frac{d}{m}} \\
		& \leq 4e^{-{i{h}}/2\kappa}W_2(\mu_{0}, \mu_*) + \frac{20{h}^2}{1-e^{-{h}/ 2\kappa}} \sqrt{\frac{d}{m}}.
	\end{split}
\end{equation*}

In this derivation, step $(i)$ is from Lemma \ref{lemma_sandwich}, $(ii)$ holds by applying \eqref{eq_single_discrete} iteratively $i$ times, $(iii)$ is concluded by summing up the geometric series and by another application of Lemma \ref{lemma_sandwich}. Hence, we deduce the desired final bound.
\end{proof}

\subsubsection*{Stochasticity Analysis}\label{section_stochastic}

Taking into account stochastic gradients, we commence by outlining the underdamped Langevin Monte Carlo (Algorithm \ref{algorithm_langevin_mcmc2} considering stochastic noise in gradient estimate). 

Suppose $(\hat{x}_t,\hat{v}_t)$ is the pair of position and velocity at time $t$ considering the stochastic gradient underdamped Langevin diffusion:
\begin{equation}\label{eq_langevin_stochastic}
  \begin{split}
  d \hat{v}_t & =-\gamma \hat{v}_t d t-u \nabla \hat f\left(\hat{x}_0\right) d t + \left(\sqrt{2 \gamma u}\right) d B_t \\
  d \hat{x}_t & =\hat{v}_t d t,
  \end{split}
\end{equation}
Then the solution $(\hat{x}_t,\hat{v}_t)$ of \eqref{eq_langevin_stochastic} is
\begin{equation}
  \begin{split}
  \hat v_t &= \hat v_0 e^{-\gamma t} - u \left(\int_0^t e^{-\gamma(t-s)} \nabla \hat f( \hat x_0) ds \right) + \sqrt{2\gamma u} \int_0^t e^{-\gamma (t-s)} dB_s\\
  \hat x_t &= \hat x_0 + \int_0^t  v_s ds.
  \end{split}
\end{equation}

It is noteworthy that by setting $t=0$ and adopting the same initial conditions as in \eqref{eq_langevin_discrete}, the resultant differential equations are almost identical to \eqref{eq_langevin_discrete}, except the stochastic gradient estimation. Therefore, we opt to omit detailed proof in this context.

Drawing on the notational framework and Assumptions outlined in \ref{assumption_lipschitz_global_main}, our next step is to delineate the discretization discrepancies. These discrepancies arise when comparing discrete process \eqref{eq_langevin_discrete} with full gradient and another discrete process in \eqref{eq_langevin_stochastic} with stochastic gradient, where both are from the identical initial distribution.

\begin{lemma}[Convergence of underdamped Langevin Monte Carlo with stochastic gradient]\label{lemma_langevin_stochastic}
Consider $\left(\tilde x_t, \tilde v_t\right)\sim \tilde{\mu}_t$ and $\left(\hat x_t, \hat v_t\right)\sim \hat{\mu}_t$ be measures corresponding to the distributions as defined in \eqref{eq_langevin_discrete} and \eqref{eq_langevin_stochastic}. Then for any $h$ satisfying $0<{h}<1$, the subsequent relationship is valid:
\begin{equation*}
	W_2(\tilde{\mu}_h, \hat{\mu}_h) \leq u{h} L\sqrt{\frac{15 nd}{k\nu}},
\end{equation*}
where $k$ is the number of data points for stochastic gradient estimates, $n$ the total number of data points we have for gradient estimate ($n\geq k$), and $L$ is a Lipschitz constant for the stochastic estimate of gradients, which follows Assumption \ref{assumption_sgld_lipschitz_main}.
\end{lemma}

\begin{proof}
Drawing from the dynamics outlined in \eqref{eq_langevin_discrete} and \eqref{eq_langevin_stochastic}, and taking into account the definition of ${\nabla} \hat f(x)$, we can deduce the following relationship:
\begin{equation}\label{eq_langevin_diff1}
\begin{split}
v_h &= \hat v_h + u\left({\int_0^h e^{-\gamma(s-h)} ds}\right) \xi\\
x_h &= \hat x_h + u{\left(\int_0^h {\left(\int_0^r e^{-\gamma(s-r)} ds\right)} dr\right)} \xi
\end{split}
\end{equation}
where $\xi$ is a zero-mean random variance to represent the difference between full gradient and stochastic gradient estimate, and it is independent of the Brownian motion. Let $ \zeta$ be the optimal coupling between $\tilde{\mu}_{h} $ and $\hat \mu_{h}$. Then we can upper bound the distance between $\hat{\mu}_{h}$ and $\tilde{\mu}_{h}$ as follows:
\begin{equation*}
  \begin{split}
  W_2^2(\tilde{\mu}_h, \hat{\mu}_h) & = \mathbb E_{\zeta}\left[\left\|{\left(\hat x, \hat v\right) - \left(\tilde x, \tilde v\right)}\right\|_2^2\right] \\
  & \overset{}{=} \mathbb E_{\zeta}\left[\left\|
		\begin{array}{*{20}{c}}
			u{\left( {\int_0^h  {\left( {\int_0^r {{e^{ - \gamma (s - r)}}} ds} \right)} dr} \right)\xi }\\
			u{\left( {\int_0^h  {\left( {\int_0^r {{e^{ - \gamma (s - r)}}} ds} \right)} dr + \int_0^h  {{e^{ - \gamma (s - h )}}} ds} \right)\xi }
		\end{array}
	\right\|_2^2\right]\\
  & \overset{(i)}{\leq} 3u^2\left[{\left(\int_0^h {\left(\int_0^r e^{-\gamma(s-r)} ds\right)} dr\right)}^2 + {\left(\int_0^h e^{-\gamma(s-h)} ds\right)}^2\right]\mathbb E{\left\|\xi\right\|_2^2}\\
  & \overset{(ii)}{\leq} 3u^2{\left(\frac{h^4}{4} + h^2\right)}\mathbb E{\left\|\xi\right\|_2^2}\\
  & \overset{(iii)}{<} \frac{15}{4}u^2h^2\mathbb E{\left\|\xi\right\|_2^2}\\
  & \overset{(iv)}{<} \frac{15 nd}{k\nu} u^2{h}^2 L^2,
  \end{split}
\end{equation*}
where $(i)$ is by independence and unbiasedness of $\xi$ and the inequality $a^2+(a+b)^2\leq 3(a^2+b^2)$, $(ii)$ uses the upper bound $e^{-\gamma(s-r)}\leq 1$ and $e^{-\gamma(s-t)}\leq 1$, $(iii)$ proceeds by the assumption, made without loss of generality, that $h< 1$, and finally $(iv)$ is derived from a result in Lemma \ref{lemma_stochastic_grad_bound}. Taking square roots from both sides completes the proof.

\end{proof}

By integrating Theorem \ref{theorem_diffusion}, Lemma \ref{lemma_langevin_discrete}, and Lemma \ref{lemma_langevin_stochastic},  we establish a convergence bound between $\hat \mu_{n{h}}$ and $\mu_*$ as follows:

\begin{theorem}[Convergence of the underdamped Langevin Monte Carlo with stochastic gradient]\label{theorem_ULA3}
Let $(x_t, v_t)\sim\mu_t$ and $(x_t, x_t+v_t)\sim \eta_t$ are corresponding to the measures of distribution described by SDEs \eqref{eq_sde}, $(\tilde x_t, \tilde v_t)\sim\hat{\mu}_t$ and $(\tilde x_t, \tilde x_t + \tilde v_t)\sim \tilde \eta_t$ by SDEs \eqref{eq_langevin_stochastic} respectively. Let $\mu_*$ and $\eta_*$ the posterior distribution described by $(x_*, v_*)$ and $(x_*, x_*+v_*)$. Following the previous choice of $u = 1/L$ and $\gamma = 2$, and let $t\in\left[ih,(i+1)h\right)$, an upper bound for the 2-Wasserstein distance between the continuous-time and discrete-time processes, with stochastic gradient, is given by:
\begin{equation}\label{eq_langein_stochasticfinal}
W_2(\hat \mu_{{ih}}, \mu_*) \leq 4e^{-{ih}/ 2\kappa}W_2(\mu_{0}, \mu_*) + \frac{4}{1-e^{-{h}/ 2\kappa}}\left(5h^2\sqrt{\frac{d}{m}} + 4u{h} L\sqrt{\frac{nd}{k\nu}}\right).
\end{equation}
\end{theorem}

\begin{proof}
We follow the proof framework in Theorem \ref{theorem_ULA2} to derive this bound. Suppose $i \in \llbracket 1, I \rrbracket$, then by the triangle inequality we will have: 

\begin{equation}\label{eq_single_stochastic}
\begin{split}
  W_2(\hat\eta_{ih}, \eta_*) & \leq  W_2(\eta_{ih}, \eta_*) + W_2(\eta_{ih}, \tilde \eta_{ih}) + W_2\left(\tilde \eta_{ih}, \hat\eta_{ih}\right)\\
  & \stackrel{(i)}{\leq} W_2(\eta_{ih}, \eta_*) + 2W_2(\mu_{ih}, \tilde \mu_{ih}) + 2W_2\left(\tilde \mu_{ih}, \hat \mu_{ih}\right) \\
  & \stackrel{(ii)}{\leq} e^{-{h}/2\kappa}W_2(\eta_{(n-1){h}}, \eta_*) + 4h^2\sqrt{\frac{26}{5}\frac{d}{m}} + 2u{h} L\sqrt{\frac{15 nd}{k\nu}},
\end{split}\end{equation}
where $(i)$ follows the same idea as \eqref{eq_discrete_sandwich}, $(ii)$ proceeds by Lemma \ref{lemma_langevin_stochastic} and \eqref{eq_single_discrete}. Finally, a repeated application of \eqref{eq_single_stochastic} $i$ times yields the following results:
\begin{equation*}
\begin{split}
W_2(\hat\eta_{{ih}}, \eta_*) & \leq e^{-{ih}/ 2\kappa}W_2(\eta_{0}, \eta_*)+ \left(1+ e^{-{h}/ 2\kappa} + \cdots
+ e^{-{(i-1)h}/ 2\kappa} \right)\left(4h^2\sqrt{\frac{26}{5}\frac{d}{m}} + 2u{h} {L}\sqrt{\frac{15 nd}{k\nu}}\right)\\ 
& \stackrel{}{<} 2e^{-{ih}/ 2\kappa}W_2(\mu_{0}, \mu_*) + \frac{2}{1-e^{-{h}/ 2\kappa}}\left(5h^2\sqrt{\frac{d}{m}} + 4uhL\sqrt{\frac{ nd}{k\nu}}\right).
\end{split}
\end{equation*}
By leveraging Lemma \ref{lemma_sandwich} again we derive the results as depicted in \eqref{eq_langein_stochasticfinal}.

\end{proof}

\subsubsection*{Concentration on the Posterior Distribution}
\begin{lemma}[Lemma 10 in \citet{mazumdar2020approximate}]\label{lemma_posterior_concentrate}
When the marginal posterior $\mu'_*$ ($x\sim\mu'_*$) is characterized by ${m}$-strong log-concavity, the following result holds for $x_U=\text{argmax}_x \mu'_*(x)$:
\begin{equation*}
	W_2\left(\mu'_*, \delta\left(x_U\right)\right) \leq 6\sqrt{\frac{d}{{m}}}.
\end{equation*}

\end{lemma}
\begin{proof}
We start by segmenting $W_2\left(\mu'_*, \delta\left(x_U\right)\right)$ into two distinct terms and subsequently establish an upper bound for both.
\begin{equation*}
\begin{split}
	W_2\left(\mu'_*, \delta\left(x_U\right)\right) &\leq W_2\left(\mu'_*, \delta\left(\mathbb{E}_{x \sim \mu'_*}[x]\right)\right)+\left\|x_U-\mathbb{E}_{x \sim \mu'_*}[x]\right\|_2\\
	& \stackrel{(i)}{\leq} W_2\left(\mu'_*, \delta\left(\mathbb{E}_{x \sim \mu'_*}[x]\right)\right) + \sqrt{\frac{3}{{m}}} \\
	& \stackrel{(ii)}{\leq} \left[\int\left\|x-\mathbb{E}_{x \sim \mu'_*}[x]\right\|^2_2 \mathrm{~d} \mu'_*(x)\right]^{1/2} + \sqrt{\frac{3}{{m}}} \\
	& \stackrel{(iii)}{\leq} e^{1 / e} \sqrt{\frac{8d}{{m}}} + \sqrt{\frac{3}{{m}}} \\ 
	& \stackrel{(iv)}{<} 6\sqrt{\frac{d}{{m}}},
\end{split}
\end{equation*}
where $(i)$ builds upon an application of Theorem 7 in \citet{basu1997mean}, whereby the selection of $\alpha=1$, the relationship between the expectation and mode in unimodal distributions can be easily derived. In $(ii)$, the coupling relationship between $\mu'_*$ and $\delta\left(\mathbb{E}_{x \sim \mu'_*}[x]\right)$ is from the adoption of the product measure.  In $(iii)$, we use the Herbst argument as detailed in \citep{ledoux2006concentration}, yielding bounds on the second moment. Notably, an ${m}$-strongly log-concave distribution reveals a log Sobolev constant equivalent to ${m}$. Thus, by application of the Herbst argument, $x \sim \mu'_*$ can be described as a sub-Gaussian random vector characterized by the parameter $\sigma^2 = \frac{1}{2 {m}}$ with $\|\mathtt u\|_2=1$:
\begin{equation*}
	\int e^{\lambda \langle \mathtt u, x-\mathrm{E}_{x \sim \mu'_*}[x]\rangle} d \mu'_* \leq e^{\lambda^2/4 {m}}.
\end{equation*}
Therefore, $x$ is $2 \sqrt{\frac{d}{{m}}}$ norm-sub-Gaussian and the following inequality holds:
\begin{equation*}
	\mathbb{E}_{x \sim \mu'_*}\left[\left\|x-\mathbb{E}_{x \sim \mu'_*}[x]\right\|_2^2\right] \leq \frac{8d}{{m}}e^{2/e}.
\end{equation*}
Given that $d\geq 1$ and $e^{1 / e} \sqrt{8} + \sqrt{3} < 6$, the conclusion specified in $(iv)$ naturally follows and thus finalize the proof.

\end{proof}

\begin{remark}
	It is important to highlight that Lemma \ref{lemma_posterior_concentrate} provides an upper bound for the 2-Wasserstein distance between the marginal distribution $\mu'_*$ w.r.t. $x$ and $\delta (x_U)$. After this, based on the relationship $\mu_* \propto \exp\left(-f(x)-\|v\|_2^2/2u\right)$, we can readily establish that the 2-Wasserstein distance between the marginal posterior distribution relative to $v$ and $\delta (v_U)$ has an upper bound of $6\sqrt{d{u}}$. By adopting ${u} = \frac{1}{{L}}$ and, without loss of generality, knowing that $m \leq {L}$, it follows that the 2-Wasserstein distance between the joint posterior distribution of $(x, v)$ and $\delta(x_U, v_U)$ is upper-bounded by $6\sqrt{\frac{d}{{m}}}$.
\end{remark}

\subsubsection*{Controlling the Kinetic Energy}\label{subsubset_kinetic_energy}
This subsection presents a definitive bound for the kinetic energy, which helps in managing the discretization error at every step.
\begin{lemma}[Kinetic Energy Bound]\label{lemma_kinetic_bound}
	For $(x_t,v_t)\sim\mu_t$ governed by the dynamics in \eqref{eq_diffusion2}, and given the delta measure $\delta\left(x_0,0\right)$ when $t=0$. With the established distance criterion $\left\| x_{0}-x_* \right\|_2^2 \leq \mathcal{D}^2$ and defining $u=1/L$, then for $i=1,\ldots I$ and $t\in\left[ih,(i+1)h\right)$, we have the bound:
	\begin{equation*}
		\mathbb E_{v_{t}\sim\mu_{t}}{\left[\|v\|_2^2\right]}\leq 26(d/m +  \mathcal{D}^2).
	\end{equation*}
\end{lemma}
\begin{proof} 
We first denote $\mu_*$ the posterior described by $(x, v)$. We further define $(x, x+v)\sim\eta_*$, $(x_t, x_t+v_t)\sim\eta_t$, and define $\zeta_t\in \Gamma(\eta_t,\eta_*)$ be the optimal coupling between $\eta_t$ and $\eta_*$. Then we can upper bound the expected kinetic energy as:
\begin{equation}\label{eq_kinetic_sep}
	\begin{split}
		\mathbb E_{\mu_t}{\|v_t\|_2^2} & = \mathbb{E}_{\zeta_t}\left[\lVert v_t - v + v\rVert^2_2\right]\\
		& \stackrel{(i)}{\leq} 2\mathbb E_{\mu_*}{\|v \|_2^2} + 2  \mathbb{E}_{\zeta_t}\left[ \left\| v_t- v \right\|^2_2\right]\\
		& \stackrel{(ii)}{\leq} 2\mathbb E_{\mu_*}{\|v \|_2^2} + 4 \mathbb{E}_{\zeta_t}\left[ \left\| (x_t+v_t) - (x+v) \right\|^2_2 + \left\| x_t - x \right\|_2^2 \right]\\
		& \stackrel{(iii)}{=} 2 \underbrace{\mathbb E_{\mu_*}{\|v \|_2^2}}_{T1} + 4 \underbrace{W_2^2 (\eta_t, \eta_*)}_{T2},
	\end{split}
\end{equation}
where $(i)$ and $(ii)$ are applications of Young's inequality, and $(iii)$ holds by optimality of $\zeta_t$. Given that $\mu_*\propto \exp\left(-f(x) - \frac{L}{2}\|v\|_2^2\right)$, it follows that $T1$ conforms to the subsequent result:
\begin{equation*}
		T1 = \mathbb E_{\mu_*}{\|v\|_2^2} {=} \frac{d}{L}.\\
\end{equation*}
Next, we proceed to bound term $T2$:
\begin{equation*}
	\begin{split}
		T2 & = W_2^2 (\eta_t,\eta_*)\\
		& \stackrel{(i)}{\leq} W_2^2 (\eta_0,\eta_*) \\
		& \stackrel{(ii)}{\leq} 2\mathbb{E}_{\eta_*} \left[\left\|v \right\|_2^2\right] + 2\mathbb{E}_{x_0 \sim \mu_0, x \sim \eta_*} \left[\left\| x_0-x \right\|_2^2 \right] \\
		& \stackrel{(iii)}{\leq} \frac{2d}{L} + 2\mathbb{E}_{x_0 \sim \mu_0, x \sim \eta_*} \left[\left\| x_0-x \right\|_2^2 \right] \\
		& \stackrel{}{=} \frac{2d}{L} + 2 \mathbb{E}_{x\sim\eta_*} \left[ \left\| x - x_0\right\|_2^2\right] \\
		& \stackrel{(iv)}{\leq} \frac{2d}{L} + 4\left(\mathbb{E}_{x\sim\eta_*} \left[ \lVert x - x_* \rVert_2^2\right] + \left\| x_{0} - x_* \right\|_2^2 \right)\\
		& \stackrel{(v)}{\leq} \frac{2d}{L} + 4\left(\frac{d}{m} + {\mathcal D}^2 \right),
	\end{split}
\end{equation*}
where the derivation of $(i)$ is based on the repeated application of Theorem \ref{theorem_diffusion}. For a more detailed explanation, readers are referred to Step 4 in Lemma 12 \citep{cheng2018underdamped}. Both $(ii)$ and $(iv)$ derive from Young's inequality. The combination of $T1$ and Theorem \ref{theorem_diffusion} informs $(iii)$. Finally, $(v)$ aligns with our predefined assumption $\left\| x_{0}-x_* \right\|_2^2 \leq \mathcal{D}^2$ and is supported by Lemma \ref{lemma_bound_position}.

By incorporating the established upper bounds for $T1$ and $T2$ into \eqref{eq_kinetic_sep}, we successfully derive the upper bound as mentioned in Lemma \ref{lemma_kinetic_bound}:
\begin{equation}\label{eq_kinetic_results}
	\begin{split}
		\mathbb E_{v\sim \mu_t}{\|v\|_2^2} & \leq 2 \mathbb E_{\mu_*}{\|v \|_2^2} + 4 W_2^2 (\eta_t, \eta_*)\\
		& \leq \frac{2d}{L} + \frac{8d}{L} + \frac{16d}{m} + 16{\mathcal D}^2 \\
		& \leq 26\left(\frac{d}{m} + {\mathcal D}^2 \right).
	\end{split}
\end{equation}
\end{proof}
Subsequently, we establish a bound for the distance between $\mu_{0}$ and $\mu_*$.

\begin{lemma}[Initial Kinetic Energy Bound]\label{lemma_initial_energy} 
Assume $(x_t,v_t) \sim \mu_t$ which adheres to the dynamic defined in \eqref{eq_diffusion2}, the posterior distribution $(x, v)\sim\mu_*$ that satisfies $\exp\left(-f(x) - \frac{L}{2}\|v\|_2^2\right)$, a delta distribution $\delta(x_0,0)$ at $t=0$, and the starting distance from the optimal being $\lVert x_0 - x_* \rVert_2^2 \leq \mathcal{D}^2$ with the choice of $u = 1/L$. Then with $t=0$, the following result holds:
\begin{equation}
	\begin{split}
	W_2^2(\mu_0,\mu_*) \leq 3\left( \frac{d}{m} + \mathcal{D}^2 \right).
	\end{split}
\end{equation}
\end{lemma}

\begin{proof}
As $\mu_t$ at time $t=0$ can be represented by a delta distribution $\delta(x_0,v_0)$, the distance between measures $\mu_0$ and $\mu_*$ can be upper bounded by:
\begin{equation}
	\begin{split}
		W_2^2(\mu_0,\mu_*) & = W_2^2(\delta(x_0,v_0),\mu_*) \\
		& = \mathbb{E}_{(x,v) \sim \mu_*}\left[\lVert x-x_{0}\rVert_2^2 + \lVert v \rVert_2^2 \right] \\ 
		& =\mathbb{E}_{(x,v) \sim \mu_*}\left[\lVert x-x_* + x_* - x_{0}\rVert_2^2 + \lVert v \rVert_2^2 \right] \\
		& \stackrel{(i)}{\leq} 2\mathbb{E}_{x \sim \mu_*}\left[\lVert x-x_*\rVert_2^2 \right] + 2\mathcal{D}^2 + \mathbb{E}_{v \sim \mu_*}\left[\lVert v \rVert_2^2 \right]\\
		& \stackrel{(ii)}{\leq} 2\mathbb{E}_{x \sim \mu_*}\left[\lVert x-x_*\rVert_2^2 \right] + \frac{d}{L} + 2\mathcal{D}^2 \\
		& \stackrel{(iii)}{\leq} 3\left(\frac{d}{m} + \mathcal{D}^2 \right),
	\end{split}
\end{equation}
where we note that $(i)$ is from Young's inequality and the definition of $\mathcal{D}^2$. $(ii)$ is built on the observation that the marginal distribution of $\mu_*$ w.r.t. $v$ correlates with $\exp\left(-\frac{L}{2}\left\| v\right\|_2^2\right)$, thereby deducing  $\mathbb{E}_{v \sim \mu_*}\left\| v \right\|_2^2 = d/L$, and the conclusion in $(iii)$ is grounded in Lemma \ref{lemma_bound_position}, which completes the proof.

\end{proof}

\begin{lemma}[Bounding the Position, Theorem 1 in \citet{durmus:hal-01304430}]\label{lemma_bound_position}
For all $t>0$ and $x\sim\mu_t$ follows \eqref{eq_diffusion2} we have:
\begin{equation*}
	\mathbb{E}_{\mu_t}\left[\lVert x-x_* \rVert_2^2\right] \leq \mathcal D^2 =  \frac{d}{m}.
\end{equation*}
\end{lemma}

\subsubsection*{Bounding the Stochastic Gradient Estimates}

\begin{lemma}[Lemma 5 in \citet{mazumdar2020approximate}]\label{lemma_stochastic_grad_bound}
Consider $f: \mathbb{R}^d \rightarrow \mathbb{R}$ as the function representing the mean of $\log \mathbb{P}(\mathcal{R}_i|x)$ for $i = 1, 2, \ldots, n$, and we also consider $\hat{f}: \mathbb{R}^d \rightarrow \mathbb{R}$ as the stochastic estimator of $f$, which is derived by randomly sampling $\log \mathbb{P}(\mathcal{R}_j|x)$ $k$ times without replacement from $\llbracket 1, \mathcal L(n) \rrbracket$. We further denote $\mathcal S$ as the set of $k$ samples $\log \mathbb{P}(\mathcal{R}_j|x)$ we have,  and it follows trivially that $k=|\mathcal S|$. Provided that Assumptions \ref{assumption_reward} and \ref{assumption_sgld_lipschitz_main} hold for $\log \mathbb{P}(\mathcal{R}_i|x)$ for all $i$, an upper bound can be derived for the expected difference between the full gradient and its stochastic estimate computed over $k$ samples:
\begin{equation*}\begin{split}
\mathbb E{\left[{\left\| \nabla \hat f(x) - \nabla f(x) \right\|}^2_2 \big| x\right]} \leq \frac{4 nd{L}^2}{k\nu} .
\end{split}\end{equation*}
\end{lemma}

\begin{proof}

The proof outline is as follows: our first task involves the analysis of the Lipschitz smoothness of the gradient estimates corresponding to the likelihood function. Given this Lipschitz smoothness and the convexity assumption, a sub-Gaussian bound is achieved via the application of logarithmic Sobolev inequalities. We finally derive a more restrictive sub-Gaussian bound on the sum of the gradient of the likelihood, which serves to upper-bound the expected value of the gradient estimate. We start by formulating the expression as the gradient of $\log \mathtt P(\cdot|\cdot)$, denoted as:
\begin{equation}\label{eq_stochastic_estimate}
\begin{split}
\mathbb E{\left[{\left\|\nabla \hat f\left(x\right) - \nabla f\left(x\right)\right\|}^2_2\right]}
&= n^2  \mathbb E{\left[{\left\| \frac{1}{k} \sum_{j=1}^k \nabla \log \mathtt P\left({\mathcal R}_j|x\right) - \frac{1}{n} \sum_{i=1}^n \nabla \log \mathtt P\left({\mathcal R}_i|x\right) \right\|}^2_2\right]} \\
&= \frac{n^2  }{k^2} \mathbb E{\left[{\left\| \sum_{j=1}^k \left(\nabla \log \mathtt P\left({\mathcal R}_j|x\right) - \frac{1}{n} \sum_{i=1}^n \nabla \log \mathtt P\left({\mathcal R}_i|x\right)\right) \right\|}^2_2\right]}.
\end{split}\end{equation}
It should be noted that one term in $\nabla \log \mathtt P ({\mathcal R}_j|x) - \frac{1}{n} \sum_{i=1}^n \nabla \log \mathtt P ({\mathcal R}_i|x) $ has been canceled out:
\begin{equation*}
\nabla \log \mathtt P ({\mathcal R}_j|x) - \frac{1}{n} \sum_{i=1}^n \nabla \log \mathtt P ({\mathcal R}_i|x) 
= \nabla \log \mathtt P ({\mathcal R}_j|x) - \frac{1}{n} \sum_{i\neq j} {\nabla \log \mathtt P ({\mathcal R}_i|x)},
\end{equation*}
which indicates that the sum of $\nabla \log \mathtt P ({\mathcal R}_i|x)$ should have $n-1$ terms remaining. We now turn our attention to analyzing the term $\nabla \log \mathtt P ({\mathcal R}_j|x) - \nabla \log \mathtt P ({\mathcal R}_i|x) $. Given the joint Lipschitz smoothness and the strong convexity outlined in Assumption \ref{assumption_sgld_lipschitz_main}, the application of Theorem 3.16 from \citet{wainwright2019high} is employed to establish that $\nabla \log \mathtt P ({\mathcal R}_j|x) - \nabla \log \mathtt P ({\mathcal R}_i|x)$ is $\frac{2 {L} }{\sqrt{\nu }}$-sub-Gaussian, where ${L} $ is the Lipschitz smooth constant and $\nu $ strongly convex constant. 

As $\nabla \log \mathtt P ({\mathcal R}_j|x) - \nabla \log \mathtt P ({\mathcal R}_i|x)$ is an i.i.d random variable, we can use the Azuma-Hoeffding inequality for martingale difference sequences \citep{wainwright2019high} to characterize the summation of such variables across $n-1$ as $\frac{2{L} }{n} \sqrt{\frac{n-1}{\nu }}$-sub-Gaussian. Building upon this, another application of the Azuma-Hoeffding inequality allows us to have that $\sum_{j=1}^k \left(\nabla \log \mathtt P \left({\mathcal R}_j|x\right) - \frac{1}{n} \sum_{i=1}^n \nabla \log \mathtt P \left({\mathcal R}_i|x\right)\right)$ is $\frac{2{L} }{n}\sqrt{\frac{k (n-1)}{\nu }}$-sub-Gaussian. Therefore, we can derive the following upper bound through the sub-Gaussian property:
\begin{equation*}\begin{split}
\mathbb E{\left[{\left\| \sum_{j=1}^k \left(\nabla \log \mathtt P \left({\mathcal R}_j|x\right) - \frac{1}{n} \sum_{i=1}^n \nabla \log \mathtt P \left({\mathcal R}_i|x\right)\right) \right\|}^2_2\right]}
\leq 4{ \frac{{d k (n-1)} {L}^2 }{n^2 \nu } }\leq 4{ \frac{dk{L}^2}{n\nu } }.
\end{split}\end{equation*}
The proof is finalized by integrating the above results into \eqref{eq_stochastic_estimate}:
\begin{equation*}\begin{split}
\mathbb E{\left[{\left\| \nabla f(x)  - \nabla \hat f(x) \right\|}^2_2\right]}
&= \frac{n^2}{k^2} \mathbb E{\left[{\left\| \sum_{j=1}^k { \frac{1}{n} \sum_{i=1}^n \nabla \log \mathtt P({\mathcal R}_i|x ) - \nabla \log \mathtt P({\mathcal R}_j|x ) } \right\|}^2_2\right]} \\
& \leq 4 \frac{nd {L}^2}{k\nu }.
\end{split}\end{equation*}

\end{proof}

\subsubsection*{Controlling the Moment Generating Function}

\begin{lemma}\label{eq_mgf_bound}
Suppose for samples $x \in \mathbb R^{d}$ given by Algorithm \ref{algorithm_langevin_mcmc2} and fixed point $x_* \in \mathbb R^{d}$ satisfy the following bound:
\begin{equation}\label{eq_sample_bound}
  \mathbb{E}\left[\|x-x_*\|_2^2 \right]\leq {\frac{18}{ m n} \left(D + \Omega p \right)},
\end{equation}
where $D=8d+2\log B$, $\Omega = \frac{16L^2d}{m\nu}+256$. Then for $\bar\rho\leq \frac{m}{8 L\Omega}$ we will have the following bound for its moment generating function:
\begin{equation*}\begin{split}
  \mathbb{E}\left[e^{nL\bar\rho\|x-x_*\|_2^2}\right] {\leq} \frac{3}{2}\left(e^{\frac{4L D}{m}\bar\rho} + 1\right).
\end{split}\end{equation*}
\end{lemma}

\begin{proof}
From Theorem \ref{theorem_posterior_concentration2}, Lemma \ref{theorem_empirical_concentrate1}, and Lemma \ref{theorem_empirical_concentrate2} we can get the following upper bound from \eqref{eq_sample_bound}:
\begin{equation*}
  \mathbb{E}\left[\|x-x_*\|_2^{2p} \right]\leq 3 \left[\frac{2}{mn}\left(D+\Omega p\right)\right]^{p}.
\end{equation*}

Given the term $\|x-x_*\|_2^2$, we identify it as a sub-exponential random variable. This characterization is supported when its moment generating function is unfolded via Taylor expansion $\mathbb{E}[e^{nL\bar\rho\|x-x_*\|_2^2}]= 1+\sum_{p=1}^\infty  \mathbb{E}\left[\frac{(nL\bar\rho)^p\|x-x_*\|_2^{2p}}{p!}\right]$, which yields:
\begin{equation*}
	\begin{split}
		\mathbb{E}\left[e^{nL\bar\rho\|x-x_*\|_2^2}\right] & = \sum_{p=0}^\infty  \mathbb{E}\left[\frac{(nL\bar\rho)^p}{p!}\left\|x-x_*\right\|_2^{2p}\right] \\
		& \stackrel{}{\leq} 1+3\sum_{p=1}^\infty  \frac{(nL\bar\rho)^p}{p!}\left(\frac{2}{mn}\right)^{p}\left(D+\Omega p\right)^p\\
		& \stackrel{(i)}{\leq}  1+\frac{3}{2}\sum_{p=1}^\infty\frac{1}{{p!}}\left(\frac{2L D}{m}\bar\rho \right)^p +\frac{3}{2}\sum_{p=1}^\infty \frac{1}{p!}\left(\frac{2L\Omega}{m}\bar\rho p\right)^p \\
		& \stackrel{(ii)}{\leq} \frac{3}{2} e^{\frac{4L D}{m}\bar\rho}+\frac{3}{2}\sum_{p=1}^\infty \left( \frac{4L\Omega }{m}\bar\rho\right)^p \\
		& \stackrel{(iii)}{\leq} \frac{3}{2}\left(e^{\frac{4L D}{m}\bar\rho} + 1\right)
	\end{split}
\end{equation*}
where for $(i)$, we use a specific case of Young's inequality for products: $(x+y)^p\leq 2^{p-1}(x^p+y^p)$, which holds for all $p\geq 1$ and $x,y\geq 0$, $(ii)$ proceeds by Taylor expansion and $\frac{1}{p!} \leq \left(\frac{2}{p}\right)^p$, and $(iii)$ holds by the selection of $\bar\rho\leq \frac{m}{8 L\Omega}$.

\end{proof}

\clearpage
\section{EXPERIMENTAL DETAILS}\label{section_experiments}

In our experimental setup designed to evaluate Thompson sampling's efficacy, we adopted a systematic hyper-parameter configuration as delineated in Table \ref{table_hyper_parameter}. The table encapsulates the consideration for hyper-parameters among simulations of Langevin Monte Carlo and general Thompson sampling algorithms. A crucial distinction between our approach and the baseline is the inclusion of the friction coefficient and noise amplitude, specifically tailored for the underdamped Langevin Monte Carlo. It is pertinent to highlight that when $\mathcal L_a(n)\leq k$, we denote the batch size directly as $\mathcal L_a(n)$. 

\begin{table}[htbp]\centering
\caption{Hyper-parameters used in Thompson sampling algorithms.}\label{table_hyper_parameter}
	\begin{tabular}{cc}
	\toprule
		Names & Ranges \\ \midrule
		Parameter dimension ($d$) 			& [$10$, $30$, $100$, $300$, $1000$] \\ 
		Number of arms ($|\mathcal A|$) 	& $10 $ \\ 
		Time horizon ($N$) 		& $1000$ \\ 
		Step size ($h$)			& [$10^{-2}$, $10^{-3}$, $10^{-4}$, $10^{-5}$, $10^{-6}$] \\ 
		Number of steps ($I$) 	& [$10^0$, $10^1$, $10^2$, $10^3$, $10^4$] \\
		Batch size ($k$)		& [$2$, $5$, $10$, $20$] \\ \midrule
		Friction coefficient ($\gamma$) & [$0.1$, $1.0$, $2.0$, $5.0$, $10.0$] \\
		Noise amplitude ($u$)	& $1.0$ \\
		\bottomrule
	\end{tabular}
\end{table}

\subsection{Additional experiments}

Here we further provide a practical exploration task aimed at investigating the highest-rated restaurants, with the use of real Google Maps ratings as our foundational dataset. The average ratings of the selected restaurants span from 3.8 to 4.8, which serves as a real-world benchmark for our statistical investigation. To simulate a realistic decision-making process, we consider an agent to visit restaurants sequentially. At each visit, the agent assigns a rating on an integer scale from 1 to 5, where these ratings are sampled from true customer ratings. This methodology allows us to dynamically evaluate each restaurant and identify the restaurant with the highest average ratings.

We apply approximate Thompson Sampling with ULMC and LMC for the restaurant recommendations. We record the expected regrets associated with these restaurants and the regrets are illustrated in Figure \ref{fig_example}, where the red solid line represents regrets obtained from the underdamped algorithm, and the black dashed line is from the overdamped algorithm. The comparative analysis of these regrets further indicates the superiority of our proposed algorithm for real-world applications.

\begin{figure}[htbp]
\begin{center}
\includegraphics[width=0.5\textwidth]{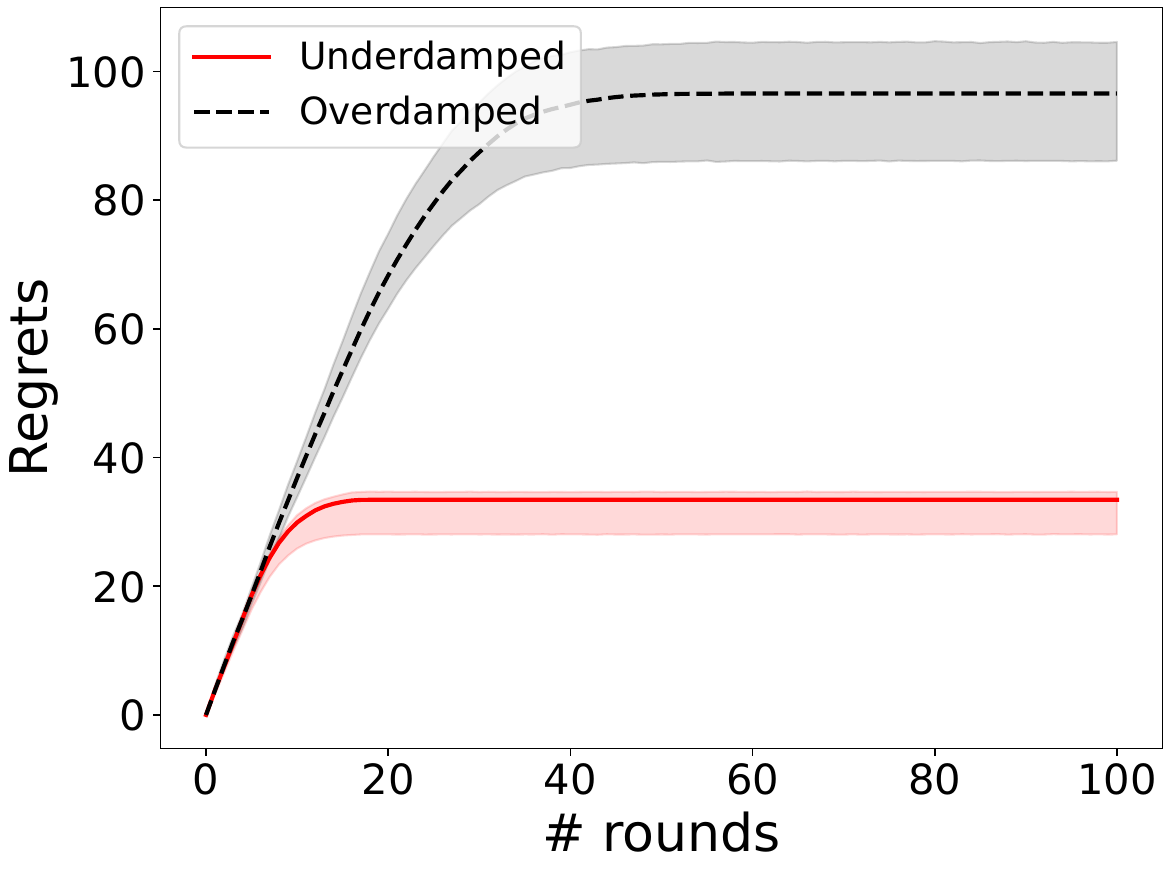}
\end{center}
\caption{Restaurant example}
\label{fig_example}
\end{figure}